\documentclass[aos]{imsart}

\RequirePackage{amsthm,amsmath,amsfonts,amssymb}
\RequirePackage[numbers,sort&compress]{natbib}
\RequirePackage[colorlinks,citecolor=blue,urlcolor=blue]{hyperref}
\RequirePackage{graphicx}

\usepackage{dsfont}
\usepackage{multirow}
\usepackage{enumitem}
\usepackage{ulem}
\usepackage{color}

\startlocaldefs
\theoremstyle{plain}

\newtheorem{theorem}{Theorem}[section]
\newtheorem{lemma}[theorem]{Lemma}
\theoremstyle{remark}
\newtheorem{definition}[theorem]{Definition}

\newtheorem{remark}{Remark}
\theoremstyle{plain}
\newtheorem{corollary}[theorem]{Corollary}
\newtheorem{proposition}[theorem]{Proposition}

\endlocaldefs

\begin{document}

\begin{frontmatter}
\title{Estimating Unbounded Density Ratios: Applications in Error Control under Covariate Shift}
\runtitle{Density Ratio Estimation and Error Control}

\begin{aug}
\author[A]{\fnms{Shuntuo}~\snm{Xu}\ead[label=e1]{oaksword@163.com}},
\author[A]{\fnms{Zhou}~\snm{Yu}\ead[label=e2]{zyu@stat.ecnu.edu.cn}}
\and
\author[B]{\fnms{Jian}~\snm{Huang}$^*$\ead[label=e3]{j.huang@polyu.edu.hk}}
\address[A]{School of Statistics, East China Normal University\printead[presep={ ,\ }]{e1,e2}}

\address[B]{Departments of Data Science and Artificial Intelligence, and Applied Mathematics \\ The Hong Kong Polytechnic University\printead[presep={,\ }]{e3}}
\end{aug}

\begin{abstract}
The density ratio is an important metric for evaluating the relative likelihood of two probability distributions, with extensive applications in statistics and machine learning. However, existing estimation theories for density ratios often depend on stringent regularity conditions, mainly focusing on density ratio functions with bounded domains and ranges. In this paper, we study density ratio estimators using loss functions based on least squares and logistic regression. We establish upper bounds on estimation errors with standard minimax optimal rates, up to logarithmic factors. Our results accommodate density ratio functions with unbounded domains and ranges. We apply our results to nonparametric regression and conditional flow models under covariate shift and identify the tail properties of the density ratio as crucial for error control across domains affected by covariate shift. We provide sufficient conditions under which loss correction is unnecessary and demonstrate effective generalization capabilities of a source estimator to any suitable target domain. Our simulation experiments support these theoretical findings, indicating that the source estimator can outperform those derived from loss correction methods, even when the true density ratio is known.
\end{abstract}



\end{frontmatter}

\section{Introduction}

The density ratio is a crucial metric for assessing the relative likelihood of two probability distributions. By comparing the densities of these distributions, the density ratio quantifies how one distribution differs from another. It has extensive applications across various areas, including  nonparametric regression \citep{sugiyama2007direct, tsuboi2009direct}, generative learning \citep{grover2018boosted, gao2022deep, heng2024generative}, change-point detection \citep{liu2013change, wang2023change}, and reinforcement learning \citep{liu2018breaking, chen2022offline}. In this paper, we study the theoretical properties of density ratio estimation using Bregman divergence. Our results accommodate density ratio functions with unbounded domains and ranges. We apply our results to covariate shift problems in the context of nonparametric regression and conditional distribution estimation using continuous flow models.

Let $X^s$ and $X^t$ represent two $d$-dimensional random vectors corresponding to the source and target domains, respectively. We assume that their probability measures are absolutely continuous with respect to the Lebesgue measure. As a consequence, $X^s$ and $X^t$ admit well-defined probability density functions, denoted by $p(x)$ and $q(x)$, respectively. Furthermore, assume that $X^t$ is absolutely continuous with respect to $X^s$. To be more specific, by defining $\mathcal{X}^s=\{x: p(x)>0\}$ and $\mathcal{X}^t=\{x: q(x)>0\}$, we suppose that $\mathcal{X}^t\subset \mathcal{X}^s$. The density ratio is then defined as $r_0(x)=q(x)/p(x)\in [0, \infty)$, where we adopt the convention that $0/0=0$. Covariate shift occurs when $p(x)\neq q(x)$ but the conditional distribution of the corresponding response variable given the covariate remains constant across both domains.

In practice, we only observe samples $\{X^s_1, \dots, X^s_n\}$ and $\{X^t_1, \dots, X^t_n\}$ from the source and target domains. Therefore, the true density ratio remains unknown and must be estimated.  Various density ratio estimation methods have been proposed in the literature \citep{suzuki2009mutual, sugiyama2012density, kremer2015nearest, rhodes2020telescoping}. However, the theoretical analyses accompanying these estimators are notably limited. Most theoretical advancements were derived under stringent conditions \citep{nguyen2010estimating, yamada2013relative, lin2023estimation}, often assuming that the density ratio was bounded from above or below, which may not be satisfied in practice.

{\color{black} We observe that unbounded density ratios are quite common. For instance, consider a scenario where the source distribution is $\text{Gamma}(1, 1)$ and the target distribution is $\text{Gamma}(2, 1)$. It is clear that the density ratio $r_0(x)=x\mathds{1}(x>0)$ diverges as $x\to \infty.$ In this example, both the domain and range of $r_0(x)$ are unbounded, and existing results on estimation error that assume a bounded density ratio do not apply to this simple case.}

We study the estimation error of density ratio estimators when both the domain and range of $r_0(x)$ are allowed to be unbounded. The estimators we consider are established based on the Bregman divergence induced by certain differentiable and strictly convex  functions 
\citep{bregman1967relaxation, kato2021non}. Particularly, we focus on two specific cases,
including the least squares loss and the logistic regression-based loss. It is important to note that these two types of loss functions present distinct continuity patterns, therefore, different regularity conditions are required for their respective analyses. Our results show that even when the density ratio is not bounded from both above and below, the estimation can still achieve nearly minimax optimal results, up to factors of logarithms.

Recently, \cite{feng2024deep} established a theoretical guarantee in handling unbounded density ratio functions under mild moment conditions. They required the truncated density ratio function to belong to a H{\"o}lder class defined on $[0, 1]^d$. This limitation inadvertently restricts flexibility when dealing with unbounded covariate domains. Furthermore, the rationale behind this constraint appears to be mainly for technical convenience rather than based on practical considerations. In contrast, the local H{\"o}lder class we propose in Subsection \ref{subsec: local_holder_class} effectively addresses challenges associated with unbounded covariate domains while enhancing interpretability.

Furthermore, another significant deficiency of assuming a bounded density ratio relates to downstream tasks. When the density ratio is bounded by a universal constant, the estimation process may be redundant in a supervised learning task subject to covariate shift. An estimator derived solely based on source data can potentially generalize to the target domain without loss of efficacy, provided that $\sup_x r_0(x)\le B$ for some constant $B>0$, in the sense that the expected excess risk in the target domain exhibits the same convergence rate to that in the source domain. Specifically, \cite{ma2023optimally} demonstrated such property in RKHS-based nonparametric regression. This raises a natural question: Is this still true for unbounded density ratios, and if so, under what conditions?

 To address this question, the tail behavior of the density ratio $r_0(X^s)$ is crucial. In more detail, the assumption that $r_0(X^s)$ is sub-exponentially distributed, articulated in our analysis of density ratio estimation,  enables us to concentrate on the region where $r_0(X^s)\le c\cdot\log n$ with $c$ being a constant. Contributions from the tail beyond this range have a negligible impact on the upper bound of the expected excess risk. This observation motivates us to rethink the necessity of loss correction (also termed as importance reweighting) through density ratios \citep{liu2015classification, fang2020rethinking, ma2023optimally}. Surprisingly and interestingly, we discover that, in the absence of such correction, the expected excess risk of a source estimator in the target domain can still be effectively controlled by its counterpart in the source domain, provided that the tail of $r_0(X^s)$ is not excessively heavy. To substantiate this phenomenon, we develop a series of general results with progressively relaxed assumptions. Furthermore, we illustrate these results through two specific cases related to nonparametric regression and conditional distribution estimation using continuous flow models.

To summarize, this paper makes two significant contributions. Firstly, we extend the estimation theory for density ratios to include cases with unbounded domains and ranges. By carefully analyzing the divergence pattern of $r_0(x)$ and the tail behavior of
$r_0(x^s)$ over the source domain, we obtain nearly optimal results. Our estimators are implemented using deep neural networks, which are particularly well-suited for handling unbounded multivariate functions using the truncation technique. Secondly, in nonparametric regression and conditional flow models under covariate shift, we identify specific regularity conditions under which classical loss correction is unnecessary, while still maintaining control over the expected excess risk in the target domain. This finding suggests that a plug-in strategy can be effectively employed in downstream tasks within the target domain by using an estimator derived from the source domain, especially when the source data is significantly more accessible.

The remainder of this paper is organized as follows: In Section \ref{sec: preliminaries}, we provide a brief introduction to density ratio estimation using Bregman divergence. Section \ref{sec: dre_efficiency} presents the theoretical results concerning density ratio estimators derived from two types of loss functions, which are special cases of Bregman divergence. In Section \ref{subsec: drc_generic_properties},
we provide the conditions under which loss correction is unnecessary under covariate shift.
In Section \ref{nrcfm}, we apply these results to nonparametric regression and conditional flow models under covariate shift.
Section \ref{sec: numerical_results} includes several numerical experiments that support our theoretical findings. Finally, Section \ref{sec: conclusion} offers a brief conclusion along with further discussions. Technical details are included in the Supplementary Materials.

\section{Preliminaries}\label{sec: preliminaries}
In this section, we present the preliminaries that will be important in the subsequent sections.

\subsection{Bregman divergence}\label{subsec: bregman_divergence}

The Bregman divergence \citep{bregman1967relaxation} quantifies the difference between two points based on a differentiable and strictly convex function $\varphi$. We present a formal definition of the Bregman divergence below.

\begin{definition}[Bregman divergence]\label{def: bregman_divergence}
Let $\varphi: \mathcal{X}\to\mathbb{R}$ be a differentiable and strictly convex function where $\mathcal{X}\subset\mathbb{R}$ is a convex set. The Bregman divergence associated with $\varphi$ for two scalars $x$ and $y$, denoted as $D_{\varphi}(x\| y)$, is defined by
$$
D_{\varphi}(x\| y)=\varphi(x)-\varphi(y)-\varphi'(y)(x-y),
$$
where $\varphi'$ represents the derivative of $\varphi$.
\end{definition}

Due to the convexity of $\varphi$, it follows that $D_{\varphi}(x\| y)=0$ implies $x=y$ (see the Supplementary Materials for more details). When $\varphi(x)=\varphi_{\mathrm{LS}}(x)=(x-1)^2$, the Bregman divergence $D_{\varphi}(x\| y)$ coincides with the least squares loss, specifically expressed as  $D_{\mathrm{LS}}(x\| y)=(x-y)^2$. On the other hand, when $\varphi(x)=\varphi_{\mathrm{LR}}(x)=x\log x-(x+1)\log(x+1)$, the Bregman divergence takes the form $D_{\varphi}(x\| y)=x\log x-(x+1)\log(x+1)+\log(y+1)-x\log y+x\log(y+1)$, which is associated with the logistic regression-based loss. This divergence will be referred to as $D_{\mathrm{LR}}(x\| y)$ hereafter.

\subsection{Local H{\"o}lder class} \label{subsec: local_holder_class}

The H{\"o}lder continuous condition is basic for analyzing the risk convergence rate of estimators in density estimation \citep{yang1999information} and nonparametric regression \citep{gyorfi2002distribution, jiao2023deep}. Definition \ref{def: holder_class} provides a characterization of the H{\"o}lder class.

\begin{definition}[H{\"o}lder class]\label{def: holder_class}
Let $\beta=s+r$ where $s=\lfloor \beta\rfloor\in \mathbb{N}$ and $r\in (0, 1]$. Here, $\lfloor \beta\rfloor$ denotes the integer strictly smaller than $\beta$ and $\mathbb{N}$ is the set of nonnegative integers. For a finite constant $B>0$, the H{\" o}lder class of functions defined on the $d$-dimensional unit hypercube and yielding a scalar as output, denoted by $\mathcal{H}^{\beta}([0, 1]^d, B)$, is defined as
$$
\mathcal{H}^{\beta}([0, 1]^d, B)=\left\{f: [0, 1]^d\to\mathbb{R}, \max_{\|\alpha\|_1\le s}\|\partial^{\alpha}f\|_{\infty}\le B, \max_{\|\alpha\|_1=s}\sup_{x\ne y}\frac{|\partial^{\alpha}f(x)-\partial^{\alpha}f(y)|}{\|x-y\|_2^r}\le B\right\}.
$$
Here, $\alpha=(\alpha_1, \dots, \alpha_d)^{\top}\in\mathbb{N}^d$, $\|\alpha\|_1=\sum_{i=1}^d\alpha_i$, and $\partial^{\alpha}=\partial^{\alpha_1}\cdots \partial^{\alpha_d}$. In addition, we call $\beta$ the smoothness index.
\end{definition}

Furthermore, for $u\in\mathbb{R}_+$ where $\mathbb{R}_+$ denotes the set of positive real numbers, let $B_u: \mathbb{R}_+\to \mathbb{R}_+$ be a function of $u$. For arbitrary function $f$, let $f_{|\mathcal{X}}$ represent the function $f$ constrained on a domain $\mathcal{X}$, such that
$$
f_{|\mathcal{X}}(x)=
\begin{cases}
    x, & x\in\mathcal{X}, \\
    0, & \text{otherwise}.
\end{cases}
$$
Then, the local H{\" o}lder class is defined in Definition \ref{def: local_holder_class}, which is a natural extension for the original H{\"o}lder class.

\begin{definition}[Local H{\"o}lder class]\label{def: local_holder_class}
 A local H{\" o}lder class with smoothness index $\beta$ and divergence regime $B_u$, denoted by $\mathcal{H}^{\beta}_{\mathrm{Loc}}(\mathbb{R}^d, B_u)$, is defined as
$$
\mathcal{H}^{\beta}_{\mathrm{Loc}}(\mathbb{R}^d, B_u)=\left\{f: \mathbb{R}^d\to \mathbb{R}, g(x)=f_{|[-u, u]^d}(2ux-u\mathrm{1}_{d})\in \mathcal{H}^{\beta}([0, 1]^d, B_u)\ \text{for any } u>0\right\}.
$$
Here, $\mathrm{1}_d$ represents the $d$-dimensional vector with all entries equal to 1.
\end{definition}

\begin{remark}
Consider the function $f(x)=x^m$ for $x\in\mathbb{R}$, where $m$ is an integer no less than 2. Let $s=m-1$ and $r=1$. For $u>0$ and $x\in [-u, u]$, let $g(x)=f(2ux-u)=u^m(2x-1)^m$. A straightforward calculation suggests that $f\in \mathcal{H}^{m}_{\mathrm{Loc}}(\mathbb{R}^d, m!2^mu^m)$. Furthermore, it can be verified that $f\in \mathcal{H}^{\ell}_{\mathrm{Loc}}(\mathbb{R}^d, m!2^mu^m)$ for any $\ell\ge m$.
\end{remark}

\subsection{Neural networks}

A neural network implemented in the multi-layer perceptron (MLP) architecture comprises a series of linear transformations and nonlinear activations. While diverse innovative network architectures have been proposed, demonstrating impressive performance, such as convolutional neural networks \citep{lecun1998gradient}, residual networks \citep{he2016deep}, and transformers \citep{vaswani2017attention}, we focus on the most fundamental representation, MLP, in this paper. Mathematically, a function $f$ implemented by an MLP with depth $L$ can be expressed as
$$
f(x)=\phi_{L+1}\circ \sigma_{L}\circ \phi_{L}\circ \cdots \circ \sigma_1\circ \phi_1(x),
$$
where $\phi_i(x)=W_ix+b_i$ with $W_i$ a matrix of $d_i$ rows and $d_{i-1}$ columns, $b_i$ a $d_i$-dimensional vector, and we let all the activations be the rectified linear unit (relu), i.e., $\sigma_i(x)=\max(x, 0)$ functioning by element, for $i=1, \dots, L+1$. The width of a network is defined as $\max_{i=1, \dots, L}d_i$.

A truncated neural network function can be represented by a deeper neural network. To illustrate this, let us define a truncation operator, denoted by $T_{a, b}: \mathbb{R}\to\mathbb{R}$ for $a<b$, expressed as
$$
T_{a, b}(x)=
\begin{cases}
    a, & x<a, \\
    x, & a\le x\le b, \\
    b, & x>b.
\end{cases}
$$
For the cases where $a<0<b$, it follows that
$$
T_{a, b}(x)=\mathrm{relu}(-\mathrm{relu}(-x+b)+b)-\mathrm{relu}(-\mathrm{relu}(x-a)-a).
$$
When $0\le a<b$, we have
$$
T_{a, b}(x)=\mathrm{relu}(-\mathrm{relu}(-x+b)+b-a)+a.
$$
This property facilitates the enforcement of boundedness within neural network functions, which is crucial in theoretical analysis where it is presumed that an estimator does not grow unrestrainedly.

For simplicity, we denote the function space consisting of elements implemented by MLPs with output dimension $d_{\mathrm{out}}$, depth $L$, width $M$, number of parameters $S$, uniformly upper bounded by a scalar $\bar{\delta}$ and lower bounded by another scalar $\underline{\delta}$, as $\mathcal{F}_{\mathrm{NN}}^{d_{\mathrm{out}}}$. The values of $\bar{\delta}$ and $\underline{\delta}$, typically relied on the sample size, should be carefully determined to avoid sub-optimal or trivial results. When $d_{\mathrm{out}}=1$, we abbreviate $\mathcal{F}_{\mathrm{NN}}^1$ to $\mathcal{F}_{\mathrm{NN}}$.

\subsection{Covariate shift}\label{subsec: covariate_shift}

Covariate shift is a prevalent challenge in supervised learning \citep{nair2019covariate}, signifying that the training and testing data are collected from different domains, namely the source domain and the target domain. Mathematically, the source domain is represented by a random pair $(X^s, Y^s)$, where $X^s$ denotes the covariate vector and $Y^s$ is the response. Concurrently, the target domain is characterized by $(X^t, Y^t)$ with $Y^t$ remaining unobservable. Under covariate shift, the conditional distributions are assumed identical, specifically that $Y^s|X^s=x$ and $Y^t|X^t=x$ share the same distribution given $x$. However, the marginal distributions of $X^s$ and $X^t$ may differ, which is a key aspect.

Let us denote the observations from source domain as $\{(X^s_1, Y^s_1), \dots, (X^s_N, Y^s_N)\}$ and those from target domain as $\{X^t_1, \dots, X^t_n\}$. For a nonparametric quantity $\theta_0(X)$ of interest, an empirical estimate $\hat{\theta}_N(X)$ is presumed to be constructed, by solely utilizing the source data. It is well-established for numerous classical methods that $\hat{\theta}_N(X)$ yields a sound performance within the source domain, in the sense that $\mathbb{E}\|\hat{\theta}_N(X^s)-\theta_0(X^s)\|_2^2$ converges to zero and even achieves an optimal rate \citep{van1996weak}. However, evaluating the performance of $\hat{\theta}_N(X)$ in the target domain, i.e., how $\mathbb{E}\|\hat{\theta}_N(X^t)-\theta_0(X^t)\|_2^2$ behaves, is not trivial.

Given that $X^s$ and $X^t$ possesses probability density functions, when the density ratio $r_0(x)$ is uniformly bounded, it becomes evident that $\hat{\theta}_N(X)$ generalizes effectively to the target domain, by noting that
$$
\begin{aligned}
\mathbb{E}\left\|\hat{\theta}_N(X^t)-\theta_0(X^t)\right\|_2^2 &= \mathbb{E}\left[\left\|\hat{\theta}_N(X^s)-\theta_0(X^s)\right\|_2^2\cdot r_0(X^s)\right] \\
&\le \|r_0\|_{\infty}\mathbb{E}\left\|\hat{\theta}_N(X^s)-\theta_0(X^s)\right\|_2^2.
\end{aligned}
$$
Here, $\|\cdot\|_{\infty}$ denotes the supremum norm of a function. Conversely, when $r_0(x)$ is unbounded, the situation becomes significantly more complicated, which necessitates a deeper investigation towards various types of discrepancies.

\subsection{Flow-based generative learning}\label{subsec: preliminary_flow_model}

Beyond regression tasks, modern machine leaning is rapidly advancing to explore entire data distributions through generative models \citep{goodfellow2014generative, song2019generative, ho2020denoising}. Flow-based generative models, among the notable recent developments, recover the data distribution from a base distribution (typically Gaussian noise) using bijective transformations \citep{rezende2015variational, papamakarios2021normalizing} or continuous-time dynamics \citep{liu2022flow, lipman2023flow}. In this paper, we adopt the conditional stochastic interpolation framework \citep{albergo2023building, huang2023conditional} due to its generality.

Specifically, let $\mathcal{I}(y_0, y_1, \tau)$ be a continuous interpolant connecting an initial point $y_0$ and a terminal point $y_1$ as the time parameter $\tau$ spans from 0 to 1, such that $\mathcal{I}(y_0, y_1, 0)=y_0$ and $\mathcal{I}(y_0, y_1, 1)=y_1$. Suppose that $(X, Y)$ is a random pair of interest, where $X\in\mathbb{R}^{d_x}$ and $Y\in\mathbb{R}^{d_y}$.
We consider the interpolation,
$$
Y_{\tau}=\mathcal{I}(\eta, Y, \tau),
$$
where $\eta$ represents the $d_y$-dimensional standard Gaussian distribution and is independent of $(X, Y)$. Let $v_0(x, y, \tau)$ be the time-dependent velocity field defined by
$$
v_0(x, y, \tau)=\mathbb{E}\left[\partial_{\tau}\mathcal{I}(\eta, Y, \tau)|Y_{\tau}=y, X=x\right].
$$
We denote the conditional probability density function of $Y_{\tau}|X=x$ by $\omega(x, \cdot, \tau)$ for $\tau\in [0, 1]$. A key property is that the family of densities $\{\omega(x, \cdot, \tau): \tau\in [0, 1]\}$ solves the following transport equation with respect to $\{\rho(\cdot, \tau): \mathbb{R}^{d_y}\to\mathbb{R}, \tau\in [0, 1]\}$ when $x$ is fixed \citep{huang2023conditional},
\begin{equation}\label{eqn: trainsport_equation}
\partial_{\tau}\rho(y, \tau)+\nabla_y\cdot \left[v_0(x, y, \tau)\rho(y, \tau)\right]=0.
\end{equation}
Here, $\nabla_y\cdot$ denotes the divergence operator over the variable $y$. Note that the transport equation $\eqref{eqn: trainsport_equation}$ corresponds tightly to the following ordinary differential equation (ODE),
\begin{equation}\label{eqn: true_ode}
\mathrm{d}Z_{\tau}=v_0(x, Z_{\tau}, \tau)\mathrm{d}\tau,
\end{equation}
for any fixed value of $x$. Under some regularity conditions, the terminal distribution deduced by ODE \eqref{eqn: true_ode} at $\tau=1$, with initial condition $Z_0\sim N(0, I_{d_y})$, is essentially the distribution of $Y|X=x$ \citep{gao2024convergence}.

\section{Density ratio estimation: error analysis}
\label{sec: dre_efficiency}

In this section, we conduct an in-depth analysis of density ratio estimators by examining two specific instances of the Bregman divergence: the least squares loss and the logistic regression loss. We establish error bounds for density ratio estimators based on these two loss functions.

\subsection{Estimation based on least squares loss}

In this subsection, we focus on a specific type of density ratio estimators deduced by the least squares loss function. As illustrated in Subsection \ref{subsec: bregman_divergence}, letting $\varphi(x)=\varphi_{\mathrm{LS}}(x)=(x-1)^2$ leads to that $D_{\mathrm{LS}}(x\| y)=(x-y)^2$. Hence, by assuming that $r_0\in \mathcal{L}^2(X^s)$, we have
$$
\begin{aligned}
r_0 &= \mathop{\mathrm{argmin}}_{f\in \mathcal{L}^2(X^s)}\mathbb{E}D_{\mathrm{LS}}(r_0(X^s)\| f(X^s)) \\
&= \mathop{\mathrm{argmin}}_{f\in \mathcal{L}^2(X^s)}\mathbb{E}\left[r_0(X^s)^2\right]+\mathbb{E}\left[f(X^s)^2\right]-2\mathbb{E}\left[f(X^s)r_0(X^s)\right] \\
&= \mathop{\mathrm{argmin}}_{f\in \mathcal{L}^2(X^s)}\mathbb{E}\left[f(X^s)^2\right]-2\mathbb{E}\left[f(X^t)\right].
\end{aligned}
$$
Here, $\mathcal{L}^2(X^s)$ represents the function space comprising all square-integrable functions with respect to the distribution of $X^s$. We note that the minimizer for $\mathbb{E}D_{\mathrm{LS}}(r_0(X^s)\| f(X^s))$ with respect to $f\in \mathcal{L}^2(X^s)$ is not unique. Nonetheless, every two minimizers are equivalent $X^s$-almost surely.

At the empirical level, given the observations $\{X^s_1, \dots, X^s_n\}$ and $\{X^t_1, \dots, X^t_n\}$, the least squares estimator, denoted as $\hat{r}_{\mathrm{LS}}$, is constructed by
$$
\hat{r}_{\mathrm{LS}}=\mathop{\mathrm{argmin}}_{f\in \mathcal{F}_{\mathrm{NN}}}\frac 1n\sum_{i=1}^nf(X^s_i)^2-\frac 2n\sum_{i=1}^nf(X^t_i).
$$
To evaluate the estimation error of  $\hat{r}_{\mathrm{LS}}$, we concentrate on the expected excess risk, defined as
$$
\mathcal{R}^{\jmath}(\hat{r}_{\mathrm{LS}})=\mathbb{E}\left[r_0(X^{\jmath})-\hat{r}_{\mathrm{LS}}(X^{\jmath})\right]^2,
$$
for $\jmath=s, t$.

Generally, $\mathcal{R}^{\jmath}(\hat{r}_{\mathrm{LS}})$ can be decomposed into two components, namely the stochastic error term and the approximation error term \citep{gyorfi2002distribution}. Both of these components are intricately related to the richness of the hypothesis space $\mathcal{F}_{\mathrm{NN}}$. Recall that every function in $\mathcal{F}_{\mathrm{NN}}$ is bounded from above by $\bar{\delta}$ and from below by $\underline{\delta}$. In this scenario, it is appropriate to set $\underline{\delta}$ to 0. However, the selection of $\bar{\delta}$ presents a more complex challenge. Since the true function $r_0$ may not be bounded, it is essential for $\bar{\delta}$ to increase with the sample size $n$. Consequently, the divergence pattern of $\bar{\delta}=\bar{\delta}_n$ becomes critical in ensuring that the hypothesis space adequately captures the complexity of $r_0$ as more data are available. On one hand, $\bar{\delta}_n$ should not be excessively large, as this may compromise control over stochastic error. On the other hand, an overly small value of $\bar{\delta}_n$ could hinder the ability of $\mathcal{F}_{\mathrm{NN}}$ to densely approximate the true function $r_0$. Here, we present Theorem \ref{thm: dre_convergence_rate_ls}, which achieves a nuanced balance.

\begin{theorem}\label{thm: dre_convergence_rate_ls}
Assume that
\begin{longlist}
    \item $r_0(x)\in \mathcal{H}^{\beta_r}_{\mathrm{Loc}}(\mathbb{R}^d, B_u)$ with $\beta_r>0$ and $B_u\le c(u^m+1)$ for some universal constants $c>0$, $m\ge 0$;
    \item $r_0(X^s)$ and $\|X^s\|_{\infty}$ are sub-exponentially distributed random variables.
\end{longlist}
Then, given the hyper-parameters $L$ set to $\mathcal{O}(n^{d/(2d+4\beta_r)}\log n)$, $M$ set to $\mathcal{O}(1)$, $\underline{\delta}$ set to 0 and $\bar{\delta}=\bar{\delta}_n$ set to $(\log n)^{1+\kappa}$ with an arbitrarily fixed $\kappa\in (0, 1]$, for $n\ge 2$, we have
$$
\begin{aligned}
\mathcal{R}^s(\hat{r}_{\mathrm{LS}}) &\le c^*n^{-\frac{2\beta_r}{d+2\beta_r}}(\log n)^{(8+4\kappa)\vee (2m)}, \\
\mathcal{R}^t(\hat{r}_{\mathrm{LS}}) &\le c^{**}n^{-\frac{2\beta_r}{d+2\beta_r}}(\log n)^{(8+4\kappa)\vee (2m)+1},
\end{aligned}
$$
where $c^*$ and $c^{**}$ are constants not depending on $n$.
\end{theorem}

In Theorem \ref{thm: dre_convergence_rate_ls}, we make two critical assumptions on $r_0$. Firstly, we require that the function $r_0(x)$ can be bounded by a universal polynomial function within any compact region. This assumption encompasses not only the trivial case where $r_0(x)$ is uniformly bounded, but also permits the density ratio to diverge across $\mathbb{R}^d$ at a polynomial rate.
The divergence rate $m$ plays a secondary role in determining the final upper bound of $\mathcal{R}^{\jmath}(\hat{r}_{\mathrm{LS}})$ for $\jmath=s, t$. While a more rapid rate yields a larger upper bound, the rate $m$ influences only the logarithmic term, thereby ensuring that the overall upper bound remains controllable and (nearly) optimal \citep{stone1982optimal, tsybakov2009introduction}.

Secondly, we suppose that $r_0(X^s)$ follows a sub-exponential distribution. This condition serves two purposes. On one hand, it ensures that $r_0(X^s)$ is square-integrable and is less stringent than conditions imposed in other works that require boundedness on $r_0(X^s)$. Additionally, our finding enhances the result presented in \cite{feng2024deep}, where a (nearly) minimax optimal rate was attained under the condition that the square of the density ratio was sub-exponential. On the other hand, this condition allows us to set $\bar{\delta}_n$ to $(\log n)^{1+\kappa}$ for $\kappa\in (0, 1]$, as we only need to consider values no greater than $c\cdot \log n$ for a sub-exponentially distributed random variable, where $c$ is a constant. Meanwhile, we assume that $\|X^s\|_{\infty}$ is also sub-exponentially distributed for similar reasons, facilitating the application of approximation theorems derived on compact sets \citep{schmidt2020nonparametric, shen2020deep}. Note that the condition of sub-exponential distribution for $r_0(X^s)$ could potentially be relaxed if the divergence rate $m$ is known. The proof of Theorem \ref{thm: dre_convergence_rate_ls} is presented in the Supplementary Materials.

\subsection{Estimation based on logistic regression loss}

In this subsection, we consider the
logistic regression-based loss function with $\varphi_{\mathrm{LR}}(x)=x\log x-(x+1)\log(x+1).$
The behavior of the range of $r_0(x)$ is crucial in this scenario. Recall that $\mathcal{X}^s=\{x: p(x)>0\}$ denotes the domain of $X^s$ and $\mathcal{X}^t=\{x: q(x)>0\}$ represents the domain of $X^t$. It is clear that $r_0$ equals to 0 over $\mathcal{X}^s\backslash\mathcal{X}^t$. While it is reasonable to assign $\varphi_{\mathrm{LR}}(0)=0$ so that the divergence $D_{\mathrm{LR}}(r_0(x)\| f(x))$ is well-defined on $\mathcal{X}^s\backslash\mathcal{X}^t$, we note that the estimation efficiency may be compromised when $\mathrm{P}(X^s\in \mathcal{X}^s\backslash\mathcal{X}^t)>0$. This inefficiency arises from an overwhelming approximation error within the region $\mathcal{X}^s\backslash\mathcal{X}^t$. To achieve a fast rate, we therefore assume that $\mathrm{P}(X^s\in \mathcal{X}^s\backslash\mathcal{X}^t)=0$. It is important to note that this condition was not imposed in the previous analysis, where we set $\underline{\delta}$ to 0. Contrastively, $\underline{\delta}=\underline{\delta}_n$ will vanish slowly in the sequel to align with the regularity of function $\varphi_{\mathrm{LR}}(x)$.

Let $\mathcal{L}_{\mathrm{LR}}(X^s)=\{h:\mathbb{R}^d\to\mathbb{R}_+\cup \{0\}, \mathbb{E}[h(X^s)^2]<\infty, \mathbb{E}[h(X^s)^{-2}\mathds{1}(X^s\in\mathcal{X}^t)]<\infty \text{ and } h(x)>0 \text{ for any } x\in\mathcal{X}^t\}$. At the population level, we have Lemma \ref{lem: dre_validity_lr_loss} which justifies the validity of the logistic regression-based loss; see the Supplementary Materials for its proof. We note that in Lemma \ref{lem: dre_validity_lr_loss}, the function space $\mathcal{L}_{\mathrm{LR}}(X^s)$ can be relaxed to $\check{\mathcal{L}}=\{h:\mathbb{R}^d\to\mathbb{R}_+\cup \{0\}, h \text{ is measurable} \text{ and } h(x)>0 \text{ for any } x\in\mathcal{X}^t\}$.

\begin{lemma}\label{lem: dre_validity_lr_loss}
Suppose that $r_0\in\mathcal{L}_{\mathrm{LR}}(X^s)$. Then, it holds that $r_0(X^s)=f^*(X^s)$ a.s.
$X^s,$
where $f^*$ is a minimizer of  $\mathbb{E}D_{\mathrm{LR}}(r_0(X^s)\| f(X^s))$ with respect to $f\in \mathcal{L}_{\mathrm{LR}}(X^s).$
\end{lemma}

As a consequence, by assuming $r_0\in\mathcal{L}_{\mathrm{LR}}(X^s)$, we have
$$
\begin{aligned}
r_0 &= \mathop{\mathrm{argmin}}_{f\in\mathcal{L}_{\mathrm{LR}}(X^s)}\mathbb{E}D_{\mathrm{LR}}(r_0(X^s)\| f(X^s)) \\
&= \mathop{\mathrm{argmin}}_{f\in\mathcal{L}_{\mathrm{LR}}(X^s)}\mathbb{E}\left[\log(f(X^s)+1)-r_0(X^s)\log f(X^s)+r_0(X^s)\log(f(X^s)+1)\right] \\
&= \mathop{\mathrm{argmin}}_{f\in\mathcal{L}_{\mathrm{LR}}(X^s)}\mathbb{E}\left[\log(f(X^s)+1)-\log f(X^t)+\log(f(X^t)+1)\right].
\end{aligned}
$$
Thus, at the sample level, the estimator is given by
$$
\hat{r}_{\mathrm{LR}}=\mathop{\mathrm{argmin}}_{f\in\mathcal{F}_{\mathrm{NN}}}\frac 1n\sum_{i=1}^n\log(f(X^s_i)+1)+\frac 1n\sum_{i=1}^n\left[-\log f(X^t_i)+\log(f(X^t_i)+1)\right].
$$
The loss function above is similar to the logistic loss. Here, the lower bounded $\underline{\delta}$ should not be set to 0, which otherwise may result in the loss being undefined. The determination of both $\bar{\delta}$ and $\underline{\delta}$ relies on the continuity of $\varphi_{\mathrm{LR}}$. Specifically, let $x$ and $y$ be two positive scalars lying in the interval $[a, b]$ with $a>0$. Then, it can be verified that
\begin{equation}\label{eqn: second_order_expansion_of_lr_divergence}
\frac{1}{2b(b+1)}(x-y)^2\le D_{\mathrm{LR}}(x\| y)\le \frac{1}{2a(a+1)}(x-y)^2.
\end{equation}
The first inequality in Eqn.~\eqref{eqn: second_order_expansion_of_lr_divergence} illustrates the relationship between the estimation error and Bregman divergence, indicating that $b$ (associated with $\bar{\delta}$) is supposed to be small. Meanwhile, the second inequality in Eqn.~\eqref{eqn: second_order_expansion_of_lr_divergence} connects Bregman divergence with the approximation error, suggesting that $a^{-1}$ (corresponding to $\underline{\delta}^{-1}$) should also remain small. Through a careful balancing of these parameters, we obtain the erorr bounds of $\hat{r}_{\mathrm{LR}}$ in Theorem \ref{thm: dre_convergence_rate_lr}.

\begin{theorem}\label{thm: dre_convergence_rate_lr}
Assume that
\begin{longlist}
    \item $r_0(x)\in \mathcal{H}^{\beta_r}_{\mathrm{Loc}}(\mathbb{R}^d, B_u)$ with $\beta_r>0$ and $B_u\le c(u^m+1)$ for some universal constants $c>0$, $m\ge 0$;
    \item $r_0(X^s)$, $r_0(X^s)^{-1}\mathds{1}(X^s\in\mathcal{X}^t)$ and $\|X^s\|_{\infty}$ are sub-exponentially distributed random variables.
\end{longlist}
Then, given the hyper-parameters $L$ set to $\mathcal{O}(n^{d/(2d+4\beta_r)}\log n)$, $M$ set to $\mathcal{O}(1)$, $\underline{\delta}=\underline{\delta}_n$ set to $(\log n)^{-1-\kappa}$ and $\bar{\delta}=\bar{\delta}_n$ set to $(\log n)^{1+\kappa}$ with an arbitrarily fixed $\kappa\in (0, 1]$, for $n\ge 3$, we have
$$
\begin{aligned}
\mathcal{R}^s(\hat{r}_{\mathrm{LR}}) &\le c^*n^{-\frac{2\beta_r}{d+2\beta_r}}(\log n)^{(11+7\kappa)\vee (2m+3+3\kappa)}, \\
\mathcal{R}^t(\hat{r}_{\mathrm{LR}}) &\le c^{**}n^{-\frac{2\beta_r}{d+2\beta_r}}(\log n)^{(11+7\kappa)\vee (2m+3+3\kappa)+1},
\end{aligned}
$$
where $c^*$ and $c^{**}$ are constants not depending on $n$.
\end{theorem}

The assumptions in Theorem \ref{thm: dre_convergence_rate_lr} closely parallel those outlined in Theorem \ref{thm: dre_convergence_rate_ls}, while we introduce an additional requirement concerning the sub-exponential restriction of $r_0(X^s)^{-1}\mathds{1}(X^s\in\mathcal{X}^t)$, which is well-defined under the convention $0/0=0$. This condition constrains the rate at which $r_0(X^s)$  approaches 0, which ensures that the region $r_0(X^s)<c(\log n)^{-1}$ remains sufficiently small for some constant $c$. In addition, we note that these assumptions may need modification depending on the formulation of the estimator. For instance, in cases where the estimator is expressed as $\exp(f)$ for $f\in\mathcal{F}_{\mathrm{NN}}$, attention should be given to the properties of $\log r_0$.

The upper bounds established in Theorem \ref{thm: dre_convergence_rate_lr} have a slower convergence rate compared to those derived in Theorem \ref{thm: dre_convergence_rate_ls}. This discrepancy can be attributed to two factors. Firstly, in this subsection, we work on the surrogate loss deduced by $\varphi_{\mathrm{LR}}$, which incurs a certain cost when converting the Bregman divergence into squared loss. Secondly, the logarithm terms enhances sharpness of $\varphi_{\mathrm{LR}}$ with respective to its derivative, leading to an increased stochastic error. The proof of Theorem \ref{thm: dre_convergence_rate_lr} is given in the Supplementary Materials.

\section{Error control under covariate shift} 
\label{subsec: drc_generic_properties}

In order to tackle covariate shift, a typical strategy is the density ratio correction \citep{sugiyama2008direct, stojanov2019low, zhang2024adapting}. Specifically, let us consider the population-level loss function associated with the parameter of interest, $\theta$, in the target domain, denoted as $\mathbb{E}\ell(X^t, Y^t, \theta)$. It is direct to observe that $\mathbb{E}\ell(X^t, Y^t, \theta)=\mathbb{E}[r_0(X^s)\ell(X^s, Y^s, \theta)]$. This relationship indicates that the original loss $\ell(x, y, \theta)$ can be adjusted to a corrected loss $r_0(x)\ell(x, y, \theta)$, which, when integrated with source data, essentially serves as the desired loss function for the target domain.

In situations where \textit{the density ratio is known} and has a finite second moment, \cite{ma2023optimally} showed that an estimator derived from a corrected loss function can achieve a minimax optimal rate. However, in practice, the true density ratio is typically unknown, necessitating the use of a density ratio estimator for correction. Consequently, the estimation error of
$\theta$  is also influenced by the estimation error of the density ratio, especially in transfer learning problems where source data is often much more accessible than target data \citep{weiss2016survey, li2020robust, wu2022power}. For example, \cite{rezaei2021robust} showed that density ratio correction can be fragile, exhibiting high variance and sensitivity to the methods used for density ratio estimation, even in cases of seemingly minor shifts \citep{cortes2010learning}. Therefore, it is crucial to explore the conditions under which density ratio correction can be avoided.

The generalization capacity of a source estimator in the target domain was examined by \cite{schmidt2024local} in their Lemma 12. Although they provided a clear upper bound, deriving explicit conditions to ensure that the convergence rate remains optimal is still a nontrivial task. Moreover, their hypothesis space was limited to the class of all 1-Lipschitz functions, and their analysis was focused on the nonparametric regression setting. Based on the analyses in Section \ref{sec: dre_efficiency}, we show that the sub-exponential property ensures effective generalization between source and target domains.

In this subsection, we outline some regularity conditions that facilitate the control for the performance of $\hat{\theta}_N(X)$ in the target domain based on its performance in the source domain, allowing for the presence of slowly divergent factors. We first present the following generic lemma; see the Supplementary Materials for its proof.

\begin{lemma}\label{lem: sub_exp_convergence}
Suppose that $U_1, \dots, U_n, U$ are $d$-dimensional random vectors, with $\|U_i\|_{\infty}\le \xi_i$ almost surely for $i=1, \dots, n$. Assume that $\|U\|_{\infty}$ has a finite fourth moment, and $V$ is a random variable such that $\mathbb{E}\exp(\varsigma|V|)<\infty$ for some positive constant $\varsigma$. Let $\gamma_n=\mathbb{E}\|U_n-U\|_2^2$. Then, for $n\ge 2$, we have
$$
\mathbb{E}\left(\|U_n-U\|_2^2|V|\right)\le c_1\gamma_n\log n+\frac{c_2d(\xi_n^2+1)}{n},
$$
where $c_1$ and $c_2$ are constants not depending on $n$.
\end{lemma}

\begin{remark}
The existence of a positive scalar $\varsigma$ such that $\mathbb{E}\exp(\varsigma|V|)<\infty$ is satisfied when $V$ is sub-exponentially distributed; see, e.g., \cite{wainwright2019high}.
\end{remark}

Lemma \ref{lem: sub_exp_convergence} illustrates that when the tail of random variable $V$ is not excessively heavy, the impact of multiplying by $|V|$ is minimal, resulting in only a sacrifice of $\log n$. By leveraging Lemma \ref{lem: sub_exp_convergence} within the framework of covariate shift, we derive Corollary \ref{cor: sub_exp_dr}. We note that the condition requiring $r_0(X^s)$ to be sub-exponentially distributed can be satisfied when $X^s\sim Ga(\alpha_1, \lambda)$ and $X^t\sim Ga(\alpha_2, \lambda)$ with $0<\alpha_2-\alpha_1\le 1$.

\begin{corollary}\label{cor: sub_exp_dr}
Suppose that
\begin{longlist}
    \item $\|\hat{\theta}_N(X^s)\|_{\infty}\le \xi_N$ almost surely for every $N\ge 1$;
    \item $\|\theta_0(X^s)\|_{\infty}$ has a finite fourth moment;
    \item $r_0(X^s)$ is sub-exponentially distributed.
\end{longlist}
Then, for $N\ge 2$, we have
$$
\begin{aligned}
\mathbb{E}\left\|\hat{\theta}_N(X^t)-\theta_0(X^t)\right\|_2^2 &= \mathbb{E}\left[\left\|\hat{\theta}_N(X^s)-\theta_0(X^s)\right\|_2^2\cdot r_0(X^s)\right] \\
&\le c_1\mathbb{E}\left\|\hat{\theta}_N(X^s)-\theta_0(X^s)\right\|_2^2\log N+\frac{c_2d_{\theta}(\xi_N^2+1)}{N},
\end{aligned}
$$
where $c_1, c_2$ are constants not depending on $N$, and $d_{\theta}$ represents the dimensionality of $\theta_0(X)$.
\end{corollary}

While the condition that $r_0(X^s)$ follows a sub-exponential distribution is appealing, it is susceptible to violation under certain circumstances. For instance, consider the case where $X^s$ and $X^t$ are drawn from normal distributions, specifically $N(\mu_1, 1)$ and $N(\mu_2, 1)$ with $\mu_1\ne \mu_2$. In this scenario, the density ratio can be expressed as $r_0(X^s)=\exp((\mu_2-\mu_1)X^s-(\mu_2^2-\mu_1^2)/2)$, which clearly indicates a significant departure from the characteristics of a sub-exponential distribution. To effectively tackle this challenge, it is crucial to analyze the divergence pattern of $r_0(x)$ alongside the tail property of $X^s$, as shown in the following two propositions. Their proofs can be found in Supplementary Materials.

\begin{proposition}\label{prop: poly_dr}
Suppose that
\begin{longlist}
    \item $\|\hat{\theta}_N(X^s)\|_{\infty}\le \xi_N$ almost surely for every $N\ge 1$;
    \item $\|\theta_0(X^s)\|_{\infty}$ has a finite eighth moment;
    \item there exists a dominant function $G(u)=c(u^m+1)$ with constants $c>0$ and $m\ge 0$ such that $r_0(x)\le G(\|x\|_{\infty})$;
    \item $r_0(X^s)$ has a finite second moment;
    \item there exists a positive constant $\varsigma$ such that $\mathbb{E}\exp(\varsigma\|X^s\|_{\infty})<\infty$.
\end{longlist}
Then, for $N\ge 2$, we have
$$
\begin{aligned}
& \mathbb{E}\left\|\hat{\theta}_N(X^t)-\theta_0(X^t)\right\|_2^2 \\
=& \mathbb{E}\left[\left\|\hat{\theta}_N(X^s)-\theta_0(X^s)\right\|_2^2\cdot r_0(X^s)\right] \\
\le & c_1(\log N)^m\mathbb{E}\left\|\hat{\theta}_N(X^s)-\theta_0(X^s)\right\|_2^2+\frac{c_2d_{\theta}(\xi_N^2+1)}{N},
\end{aligned}
$$
where $c_1, c_2$ are constants not depending on $N$, and $d_{\theta}$ represents the dimensionality of $\theta_0(X)$.
\end{proposition}

\begin{proposition}\label{prop: exp_dr}
Suppose that
\begin{longlist}
    \item $\|\hat{\theta}_N(X^s)\|_{\infty}\le \xi_N$ almost surely for every $N\ge 1$;
    \item $\|\theta_0(X^s)\|_{\infty}$ has a finite eighth moment;
    \item there exists a dominant function $G(u)=c\exp(mu)$ with constants $c>0$, $m\ge 0$ such that $r_0(x)\le G(\|x\|_{\infty})$;
    \item $r_0(X^s)$ has a finite second moment;
    \item there exists a positive constant $\varsigma$ such that $\mathbb{E}\exp(\varsigma\|X^s\|_{\infty}^2)<\infty$.
\end{longlist}
Then, for $N\ge 2$, we have
$$
\begin{aligned}
& \mathbb{E}\left\|\hat{\theta}_N(X^t)-\theta_0(X^t)\right\|_2^2 \\
=& \mathbb{E}\left[\left\|\hat{\theta}_N(X^s)-\theta_0(X^s)\right\|_2^2\cdot r_0(X^s)\right] \\
\le & c_1\exp\left\{c_2(\log N)^{1/2}\right\}\mathbb{E}\left\|\hat{\theta}_N(X^s)-\theta_0(X^s)\right\|_2^2+\frac{c_3d_{\theta}(\xi_N^2+1)}{N},
\end{aligned}
$$
where $c_1, c_2, c_3$ are constants not depending on $N$, and $d_{\theta}$ represents the dimensionality of $\theta_0(X)$.
\end{proposition}

\begin{remark}
For any positive scalar $\zeta$, we have $\exp\{(\log N)^{1/2}\}=o(N^{\zeta})$. The existence of a positive scalar $\varsigma$ such that $\mathbb{E}\exp(\varsigma\|X^s\|_{\infty}^2)<\infty$ can be satisfied when $\|X^s\|_{\infty}$ is sub-Gaussian distributed; see, e.g., \cite{wainwright2019high}.
\end{remark}

Propositions \ref{prop: poly_dr} -- \ref{prop: exp_dr} extend the result presented in Corollary \ref{cor: sub_exp_dr} to the scenarios where $r_0(X^s)$ may not be sub-exponentially distributed and instead possesses only a second finite moment. By assuming various divergence patterns of $r_0(x)$ as well as the tail properties of $X^s$, we derive distinct upper bounds for the expected excess risk in the target domain. Specifically, when $r_0(x)$ diverges according to a polynomial rate, the excess risk in target domain is shown to differ from that in source domain by a factor that is a polynomial function of logarithm order. Furthermore, if $r_0(x)$ diverges at a more rapid rate, we require that $X^s$ exhibits greater concentration. In this case, the difference in excess risks between the source and target domains becomes more pronounced, which is larger than any polynomial function of logarithm order but smaller than any positive power of $N$.

When the density ratio is unbounded, \cite{ma2023optimally} considered a reweighted RKHS least squares estimator using the truncated density ratio as weights. They demonstrated that this reweighted RKHS estimator is nearly optimal in the target domain under appropriate conditions. However, their estimator assumes that the density ratio is known, which is often not the case in practice, as the density ratio typically needs to be estimated. In contrast to \cite{ma2023optimally}, our analysis indicates that it is possible to construct a nearly optimal estimator without relying on the density ratio. In such scenarios, an estimator based solely on source domain data can still generalize effectively to the target domain.
For a class of parametric models, \cite{ma2023iclr} showed that the classical maximum likelihood estimator, using only source data without any modifications, achieves minimax optimality for covariate shift if the parametric model is correctly specified. Their results also hold without requiring any boundedness condition on the density ratio.


\section{
Nonparametric regression and conditional flow models}
\label{nrcfm}

In this section, we apply the  results from Section \ref{subsec: drc_generic_properties} to two important scenarios: nonparametric regression and conditional distribution estimation using flow models under covariate shift.

\subsection{Nonparametric regression}
Over the past few decades, nonparametric regression has emerged as an active area of research in statistical learning, with extensive studies established based on methods such as splines \citep{van1990estimating}, reproducing kernels \citep{caponnetto2007optimal} and neural networks \citep{schmidt2020nonparametric}. In this subsection, we aim to elucidate the generalization capacity of the regression estimator derived from source domain when applied to target domain. Specifically, the regression task is framed in a general case, where both the covariate domain and true conditional mean function may be unbounded, and the response can be multi-dimensional.

Consider the following models
$$
Y^{\jmath}=f_0(X^{\jmath})+\varepsilon^{\jmath}, \quad \text{for } \jmath=s, t.
$$
Here, $f_0$ represents the unknown regression function of interest, and $\varepsilon^{\jmath}$ denotes the noise term with $\mathbb{E}(\varepsilon^{\jmath}|X^{\jmath})=0$ and $\mathrm{Var}(\varepsilon^{\jmath}|X^{\jmath})=\Xi_0$ for some positive semi-definite matrix $\Xi_0$ and for $\jmath=s, t$; additionally, $\varepsilon^s$ and $\varepsilon^t$ are identically distributed. Suppose that the covariate vector is $d_x$-dimensional and the response vector is $d_y$-dimensional. Our estimation paradigm concentrate on the source domain. Given source data $\{(X^s_1, Y^s_1), \dots, (X^s_N, Y^s_N)\}$, the sample estimator is given by
$$
\hat{f}^s_{N}=\mathop{\mathrm{argmin}}_{f\in\mathcal{F}_{\mathrm{NN}}^{d_y}}\frac 1N\sum_{i=1}^N\left\|Y^s_i-f(X^s_i)\right\|_2^2.
$$
We note that $\mathcal{F}_{\mathrm{NN}}^{d_y}$ is a neural network function class such that $f: \mathbb{R}^{d_x}\to [\underline{\delta}, \bar{\delta}]^{d_y}$ for any $f\in\mathcal{F}_{\mathrm{NN}}^{d_y}$.

Theorem \ref{thm: reg_convergence_rate} delineates the estimation error, founded on the expected excess risk, associated with the source estimator $\hat{f}^s_N$ in both source and target domains. It is noticeable that, $\hat{f}^s_N$ attains, within the source domain, a (nearly) standard minimax optimal convergence rate \citep{stone1982optimal}, while offering an remarkably similar rate in the target domain, with only a logarithmic factor as the compromise. The proof of Theorem \ref{thm: reg_convergence_rate} is present in Supplementary Materials.

\begin{theorem}\label{thm: reg_convergence_rate}
Assume that
\begin{longlist}
    \item $e_j^{\top}f_0(x)\in \mathcal{H}^{\beta_f}_{\mathrm{Loc}}(\mathbb{R}^d, B_u)$ with $\beta_f>0$ and $B_u\le c(u^m+1)$ for some universal constants $c>0$, $m\ge 0$, and for any $j\in\{1, \dots, d_y\}$, where $e_j$ denotes a $d_y$-dimensional one-hot vector with the $j$-th component equal to 1 and all other components equal to 0;
    \item $\|Y^s\|_{\infty}$, $r_0(X^s)$ and $\|X^s\|_{\infty}$ are sub-exponentially distributed random variables.
\end{longlist}
Then, given the hyper-parameters $L$ set to $\mathcal{O}(N^{d_x/(2d_x+4\beta_f)}\log N)$, $M$ set to $\mathcal{O}(1)$, $\bar{\delta}=\bar{\delta}_N$ set to $(\log N)^{1+\kappa}$ with an arbitrarily fixed $\kappa\in (0, 1]$, and $\underline{\delta}=\underline{\delta}_N$ set to $-(\log N)^{1+\kappa}$, for $N\ge 2$, we have
$$
\begin{aligned}
\mathbb{E}\left\|\hat{f}^s_N(X^s)-f_0(X^s)\right\|_2^2 &\le c^*N^{-\frac{2\beta_f}{d_x+2\beta_f}}(\log N)^{(8+4\kappa)\vee (2m)}, \\
\mathbb{E}\left\|\hat{f}^s_N(X^t)-f_0(X^t)\right\|_2^2 &\le c^{**}N^{-\frac{2\beta_f}{d_x+2\beta_f}}(\log N)^{(8+4\kappa)\vee (2m)+1},
\end{aligned}
$$
where $c^*$ and $c^{**}$ are constants not depending on $N$.
\end{theorem}

We note that Theorem \ref{thm: reg_convergence_rate} is a direct application of Corollary \ref{cor: sub_exp_dr}, under the assumption that the density ratio is distributed sub-exponentially. Furthermore, when $r_0(X^s)$ exhibits different patterns, as demonstrated in Propositions \ref{prop: poly_dr} and \ref{prop: exp_dr}, analogous results can be obtained with appropriate modifications.

In recent years, there has been considerable effort dedicated to the error analysis of nonparametric regression using deep neural network models \citep{bauer2019deep, nakada2020adaptive, schmidt2020nonparametric, farrell2021deep, jiao2023deep}. These studies typically rely on the crucial assumption that the regression function belongs to a uniformly bounded Hölder class defined on a bounded domain. This assumption simplifies the analysis by ensuring that the function's behavior is well-controlled across its entire domain.
In contrast, our results, as presented in Theorem \ref{thm: reg_convergence_rate}, relax this assumption by only requiring that the regression function belongs to a local Hölder class. This  allows for unbounded domains and ranges, which is more realistic for scenarios where data may not be neatly confined within bounded limits. Handling unbounded functions requires a careful analysis of the tail behavior of the relevant distributions and functions, making the analysis more technically challenging. Our results significantly enhance the understanding  of deep neural networks' performance in nonparametric regression tasks.

\subsection{Conditional flow models}

In this subsection, we focus on the task of learning a  conditional distribution using
generative flow models, as described in Section \ref{subsec: preliminary_flow_model}.
We consider a specific stochastic interpolant,
\begin{equation}\label{eqn: additive_interpolant}
Y^s_{\tau}=a_{\tau}\eta+b_{\tau}Y^s,
\end{equation}
where $a_{\tau}$ and $b_{\tau}$ are continuously differentiable with respect to $\tau\in [0, 1]$, satisfying the boundary conditions $a_0=b_1=1$ and $a_1=b_0=0$, and  $\eta$ denotes the $d_y$-dimensional standard Gaussian random vector.

We define the deduced velocity field as
$$
v_0(x, y, \tau)=\mathbb{E}\left(\dot{a}_{\tau}\eta+\dot{b}_{\tau}Y^s\big|Y^s_{\tau}=y, X^s=x\right).
$$
Let $\mathcal{L}^2(X^s, Y^s)=\{f: \mathbb{R}^{d_x}\times \mathbb{R}^{d_y}\times [0, 1], \mathbb{E}\|f(X^s, Y^s_{\tau}, \tau)\|_2^2<\infty \text{ for any } \tau\in [0, 1]\}$. Clearly, at the population level, it holds that
$$
v_0=\mathop{\mathrm{argmin}}_{f\in \mathcal{L}^2(X^s, Y^s)}\int_0^1\mathbb{E}\left\|\dot{a}_{\tau}\eta+\dot{b}_{\tau}Y^s-f(X^s, Y^s_{\tau}, \tau)\right\|_2^2\mathrm{d}\tau,
$$
provided that $\|Y^s\|_2$ has a finite second moment. At the empirical level, given source observations $\{(X^s_1, Y^s_1), \dots, (X^s_N, Y^s_N)\}$, we independently sample $N$ random vectors $\{\eta_1, \dots, \eta_N\}$ from the Gaussian distribution $N(0, I_{d_y})$, and $N$ random values $\{\tau_1, \dots, \tau_N\}$ from the uniform distribution $U(0, 1)$. Then, the empirical estimator of $v_0$ is constructed by
$$
\hat{v}^s_N=\mathop{\mathrm{argmin}}_{f\in \mathcal{F}_{\mathrm{NN}}^{d_y}}\frac{1}{N}\sum_{i=1}^N\left\|\dot{a}_{\tau_i}\eta_i+\dot{b}_{\tau_i}Y^s_i-f(X^s_i, Y^s_{i, \tau_i}, \tau_i)\right\|_2^2,
$$
where $Y^s_{i, \tau_i}=a_{\tau_i}\eta_i+b_{\tau_i}Y^s_i$.

With the estimate $\hat{v}^s_N$ and for any fixed $x\in\mathcal{X}^s$, an ODE, with respect to $\tau\in [0, 1]$, is established as follows,
$$
\mathrm{d}\hat{Z}_{\tau}=\hat{v}^s_N(x, \hat{Z}_{\tau}, \tau)\mathrm{d}\tau, \quad \hat{Z}_0\sim N(0, I_{d_y}).
$$
Intuitively, the distribution of $\hat{Z}_1$ is an approximation of the conditional distributions of $Y^s|X^s=x$ and $Y^t|X^t=x$. To quantify this approximation more concretely, we employ the 2-Wasserstein distance as a criterion for measuring the discrepancy of two distributions (see, e.g., \cite{kantorovich1960mathematical, shen2018wasserstein} for its definition). We use $W_2^2(\rho_1\| \rho_2)$ to denote the squared 2-Wasserstein distance for two probability density functions $\rho_1$ and $\rho_2$. Supposing that $Y^s|X^s=x$ admits a conditional density function denoted as $\rho_{0, x}$, we denote the density function of $\hat{Z}_1$, when given $x$, as $\hat{\rho}^s_x$. Then, the estimation errors for source and target domains are defined respectively as
$$
\begin{aligned}
\mathcal{E}^s &= \mathbb{E}\left[W_2^2(\rho_{0, X^s}\|\hat{\rho}^s_{X^s})\right]=\int \mathbb{E}\left[W_2^2(\rho_{0, x}\|\hat{\rho}^s_x)\right]p(x)\mathrm{d}x, \\
\mathcal{E}^t &= \mathbb{E}\left[W_2^2(\rho_{0, X^t}\|\hat{\rho}^s_{X^t})\right]=\int \mathbb{E}\left[W_2^2(\rho_{0, x}\|\hat{\rho}^s_x)\right]q(x)\mathrm{d}x.
\end{aligned}
$$
While the relationship between $\mathcal{E}^s$ and $\mathcal{E}^t$ is not immediately evident, Lemma \ref{lem: wasserstein_convergence} reveals that they exhibits similar property compared to the squared loss demonstrated in Corollary \ref{cor: sub_exp_dr}. The proof of Lemma \ref{lem: wasserstein_convergence} is present in the Supplementary Materials.

\begin{lemma}\label{lem: wasserstein_convergence}
Assume that
\begin{longlist}
    \item the solution of ODE \eqref{eqn: true_ode}, with standard Gaussian initialization, is unique such that $Z_1$ given $x$ follows the distribution of $Y^s|X^s=x$ for all $x\in\mathcal{X}^s$;
    \item $\|Y^s\|_2$ has a finite fourth moment;
    \item $r_0(X^s)$ is sub-exponentially distributed.
\end{longlist}
Then, for $N\ge 2$, we have
$$
\mathcal{E}^t\le c_1\mathcal{E}^s\log N+\frac{c_2d_y[\max(\bar{\delta}^2, \underline{\delta}^2)+1]}{N},
$$
where $c_1$ and $c_2$ are constants not depending on $N$.
\end{lemma}

\begin{remark}
Assumption \textit{(i)} in Lemma \ref{lem: wasserstein_convergence} can be satisfied by some regularity conditions on the probability structure of $(X^s, Y^s)$ and on the continuity of $v_0$; see, e.g., \cite{bris2008existence, huang2023conditional, gao2024gaussian}.
\end{remark}

\begin{remark}
The result present in Lemma \ref{lem: wasserstein_convergence} is not restricted to the specific interpolant~\eqref{eqn: additive_interpolant}.
\end{remark}

We now proceed to present explicit upper bounds for both $\mathcal{E}^s$ and $\mathcal{E}^t$. Essentially, the estimation error of conditional density function heavily relies on the estimation error of the velocity field. As in previous analyses, it is pivotal to specify the continuity pattern of the underlying function, which is $v_0$ in this case. Here, we adopt a Sobolev-type function class, which facilitates the application of Gr{\"o}nwall's inequality. The classical Sobolev space is defined as follows.
\begin{definition}[Sobolev space]
Let $\beta\in\mathbb{N}$. The Sobolev space $\mathcal{W}^{\beta, \infty}(\Omega)$ is defined by
$$
\mathcal{W}^{\beta, \infty}(\Omega)=\left\{f: \Omega\to\mathbb{R}, \|f\|_{\infty}<\infty, \|D^{\alpha}f\|_{\infty}<\infty \text{ for all } \alpha\in\mathbb{N}^d \text{ with } \|\alpha\|_1\le \beta\right\}.
$$
Furthermore, for any $f\in\mathcal{W}^{\beta, \infty}(\Omega)$, we define the Sobolev norm $\|\cdot\|_{\mathcal{W}^{\beta, \infty}(\Omega)}$ as
$$
\|f\|_{\mathcal{W}^{\beta, \infty}(\Omega)}=\max_{0\le \|\alpha\|_1\le \beta}\|D^{\alpha}f\|_{\infty}.
$$
\end{definition}
Then, we introduce the local and time-space version of $\mathcal{W}^{\beta, \infty}(\Omega)$, denoted as $\mathcal{W}_{\mathrm{Gen}}^{\beta, \infty}(\mathbb{R}^d, B_u)$, which is defined by
$$
\begin{aligned}
\mathcal{W}_{\mathrm{Gen}}^{\beta, \infty}(\mathbb{R}^d, B_u)=\Big\{& f: \mathbb{R}^d\times [0, 1]\to \mathbb{R}, \\
& g(x, \tau)=f_{|[-u, u]^d\times [0, 1]}(2ux-u\mathrm{1}_d, \tau)\in \mathcal{W}^{\beta, \infty}([0, 1]^{d+1})\\
& \text{with } \|g\|_{\mathcal{W}^{\beta, \infty}([0, 1]^{d+1})}\le B_u \text{ for any } u>0\Big\}.
\end{aligned}
$$
It is worthy noting that this function class is chosen for simplicity. Recent studies indicate that a general velocity field may exhibit singular behavior at $\tau=1$ \citep{gao2024convergence, jiao2024convergence}; however, this aspect falls outside the scope of this paper and deserves a more thorough investigation.

To be compatible with the Lipschitz continuity inherent in Sobolev space, the modified neural network class $\mathcal{F}^{d}_{\mathrm{NN}, \Lambda}=\mathcal{F}^{d}_{\mathrm{NN}}\cap \mathcal{F}_{\mathrm{Lip}, \Lambda}^d$ is employed, where
$$
\mathcal{F}_{\mathrm{Lip}, \Lambda}^d=\{f: \Omega\to \mathbb{R}^d, \|f(z_1)-f(z_2)\|_2\le \Lambda\|z_1-z_2\|_2, \text{ for any } z_1, z_2\in\Omega\}.
$$
Furthermore, the depth and width of $\mathcal{F}^{d}_{\mathrm{NN}, \Lambda}$ correspond to the depth and width of $\mathcal{F}^{d}_{\mathrm{NN}}$. Theorem \ref{thm: gen_convergence_rate} establishes the sample convergence of the conditional density estimation error; see the Supplementary Materials for its proof.

\begin{theorem}\label{thm: gen_convergence_rate}
Assume that
\begin{longlist}
    \item the solution of ODE \eqref{eqn: true_ode}, with standard Gaussian initialization, is unique such that $Z_{\tau}$ given $x$ follows the distribution of $Y^s_{\tau}|X^s=x$ for all $x\in\mathcal{X}^s$ and $\tau\in [0, 1]$;
    \item $e_j^{\top}v_0\in \mathcal{W}^{1, \infty}_{\mathrm{Gen}}(\mathbb{R}^{d_x+d_y}, B_u)$ with $B_u\le c(u^m+1)$ for some universal constants $c>0$, $m\in [0, 1]$, and for any $j\in\{1, \dots, d_y\}$, where $e_j$ denotes a $d_y$-dimensional one-hot vector with the $j$-th component equal to 1 and all other components equal to 0;
    \item $\|Y^s\|_{\infty}$ and $\|X^s\|_{\infty}$ are sub-Gaussian random variables;
    \item $r_0(X^s)$ is sub-exponentially distributed.
\end{longlist}
Then, given the hyper-parameters $L$ set to $\mathcal{O}(N^{(d_x+d_y+1)/[2(d_x+d_y+1)+4]}\log N)$, $M$ set to $\mathcal{O}(1)$, $\bar{\delta}=\bar{\delta}_N$ set to $(\log N)^{(1+\kappa)/2}$ with an arbitrarily fixed $\kappa\in (0, 1)$, $\underline{\delta}=\underline{\delta}_N$ set to $-(\log N)^{(1+\kappa)/2}$ and $\Lambda=\Lambda_N$ set to $(\log N)^{(1+\kappa)/2}$, for $N\ge 2$, we have
$$
\begin{aligned}
\mathcal{E}^s &\le c^*N^{-\frac{2}{d_x+d_y+3}}(\log N)^{6+2\kappa}\exp\left(1+2(\log N)^{(1+\kappa)/2}\right), \\
\mathcal{E}^t &\le c^{**}N^{-\frac{2}{d_x+d_y+3}}(\log N)^{7+2\kappa}\exp\left(1+2(\log N)^{(1+\kappa)/2}\right)
\end{aligned}
$$
where $c^*$ and $c^{**}$ are constants not depending on $N$.
\end{theorem}

The unique-solution assumption in Theorem \ref{thm: gen_convergence_rate} is an enhancement of assumption (i) in Lemma \ref{lem: wasserstein_convergence}. This refinement stems from the examination of the whole dynamics concerning the induced ODEs, whose validity is also contingent on the probability structure of $(X^s, Y^s)$. Furthermore, we assume the divergent rate of each component in $v_0$ does not exceed a linear rate. Therefore, the Lipschitz constant $(\log N)^{(1+\kappa)/2}$ with $\kappa\in (0, 1)$ is sufficient for performing the approximation, while simultaneously ensuring the convergence of $\mathcal{E}^s$. Notably, $\|Y^s\|_{\infty}$ and $\|X^s\|_{\infty}$ are assumed to exhibit sub-Gaussian behavior, as tighter concentration is essential for the approximation process. We emphasize that Theorem \ref{thm: gen_convergence_rate} aims to offer a concrete instance of the controllability for $\mathcal{E}^t$ given the convergence of $\mathcal{E}^s$. The convergence rates of both $\mathcal{E}^s$ and $\mathcal{E}^t$ can potentially be improved through a more nuanced investigation into the continuity properties of $v_0$.

\section{Simulation studies}\label{sec: numerical_results}

To practically justify the theoretical findings, we here present some empirical results from simulation experiments, demonstrating the consistency of our density ratio estimators and the risk controllability under covariate shift. In particular, we concentrate on the scenarios where the source and target covariates follow gamma distributions, thereby fulfilling or surpassing the sub-exponential assumptions. For clarity of notation, given a $d$-dimensional vector $x$, its $j$-th entry is denoted as $x_{(j)}$ for $j=1, \dots, d$.

\subsection{Performance of density ratio estimators} \label{subsec: experiments_dre}

Regarding the source covariate $X^s$ and the target covariate $X^t$, let $X^s_{(j)}$ independently follows $Ga(j, 2)$ and let $X^t_{(j)}$ independently follows $Ga(j+1, 2)$, for $j=1, \dots, d$. It is straightforward to verify that the true density ratio function can be expressed as $r_0(x)=2^d(d!)^{-1}\prod_{j=1}^dx_{(j)}$. Hence, $r_0(x)\in \mathcal{H}^{\beta_r}_{\mathrm{Loc}}(\mathbb{R}^d, 2^{2d}(d!)^{-1}u^d)$ for any $\beta_r\ge d$, and $\|X^s\|_{\infty}$ is sub-exponentially distributed. Notably, $r_0(X^s)$ is a sub-exponential random variable when $d=1$ while exhibiting a heavier tail for larger $d\ge 2$.

We adopted the least squares loss for estimation. The neural network was designed with $\lfloor (\log n)/2\rfloor$ hidden layers, where $n$ represented the sample size. Each hidden layer contained 64 neurons. We simply set $\kappa=0.5$. The training algorithm was implemented using Pytorch framework \citep{paszke2019pytorch} along with the Adam optimizer \citep{kingma2014adam}. We specified a learning rate at 1e-4 and a batch size of 100. The sample size was varied among $\{200, 500, 1000, 1500, 2000, 3000\}$; the dimension $d$ was set to 1, 2 and 5, with the number of training iterations being 1000, 2000 and 5000, respectively. After obtaining an estimator, we evaluated its performance by calculating the mean squared loss within both source and target domains based on 1000 testing samples per domain. To enhance robustness and reliability, we conducted 100 replications for each $(n, d)$ combination.

\begin{table*}[t]
\centering
\caption{Averages and standard deviations (shown in the brackets) of mean squared errors between the true density ratio values and predicted values, based on results from 100 replications.}
\begin{tabular}{ccrrr}
\hline
Domain & Sample size & $d=1$ & $d=2$ & $d=5$ \\ \hline
\multirow{6}{*}{Source} & 200 & 0.099 (0.206) & 0.508 (0.447) & 4.032 (1.753) \\
 & 500 & 0.049 (0.065) & 0.302 (0.226) & 3.979 (1.706) \\
 & 1000 & 0.022 (0.023) & 0.211 (0.178) & 2.434 (1.044) \\
 & 1500 & 0.021 (0.051) & 0.149 (0.091) & 1.801 (1.033) \\
 & 2000 & 0.013 (0.023) & 0.130 (0.063) & 1.263 (0.560) \\
 & 3000 & 0.013 (0.019) & 0.117 (0.107) & 1.351 (0.852) \\ \hline
\multirow{6}{*}{Target} & 200 & 0.390 (0.789) & 2.337 (1.343) & 64.130 (26.729) \\
 & 500 & 0.226 (0.316) & 1.574 (0.906) & 57.939 (25.431) \\
 & 1000 & 0.095 (0.112) & 1.240 (0.896) & 49.817 (25.358) \\
 & 1500 & 0.089 (0.261) & 1.116 (0.699) & 43.566 (19.931) \\
 & 2000 & 0.053 (0.098) & 1.017 (0.478) & 39.076 (18.381) \\
 & 3000 & 0.051 (0.082) & 0.926 (0.739) & 37.078 (19.932) \\
\hline
\end{tabular}
\label{tab: results_dre}
\end{table*}

\begin{figure}[t]
    \centering
    \includegraphics[width=0.3\linewidth]{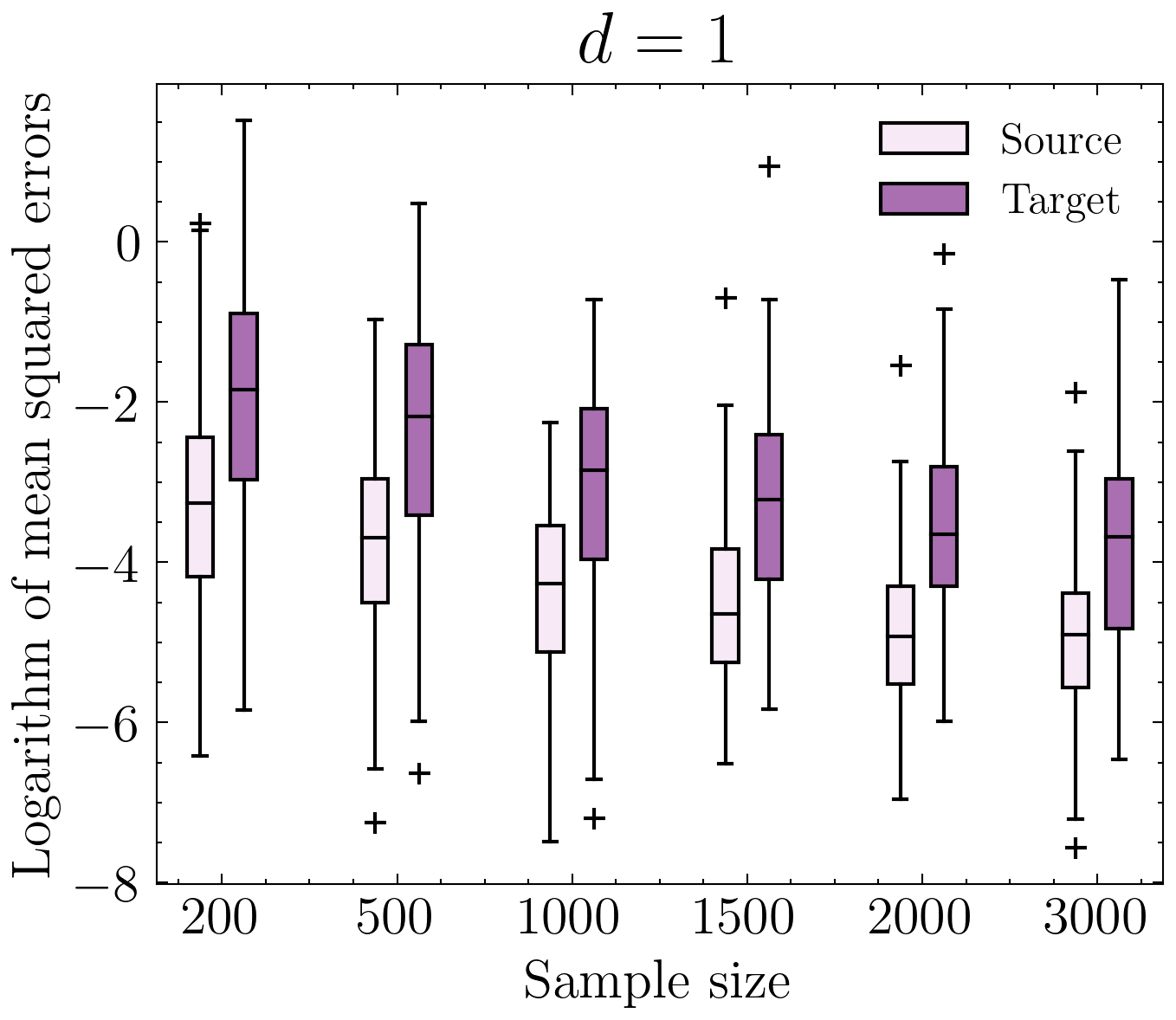}
    \includegraphics[width=0.3\linewidth]{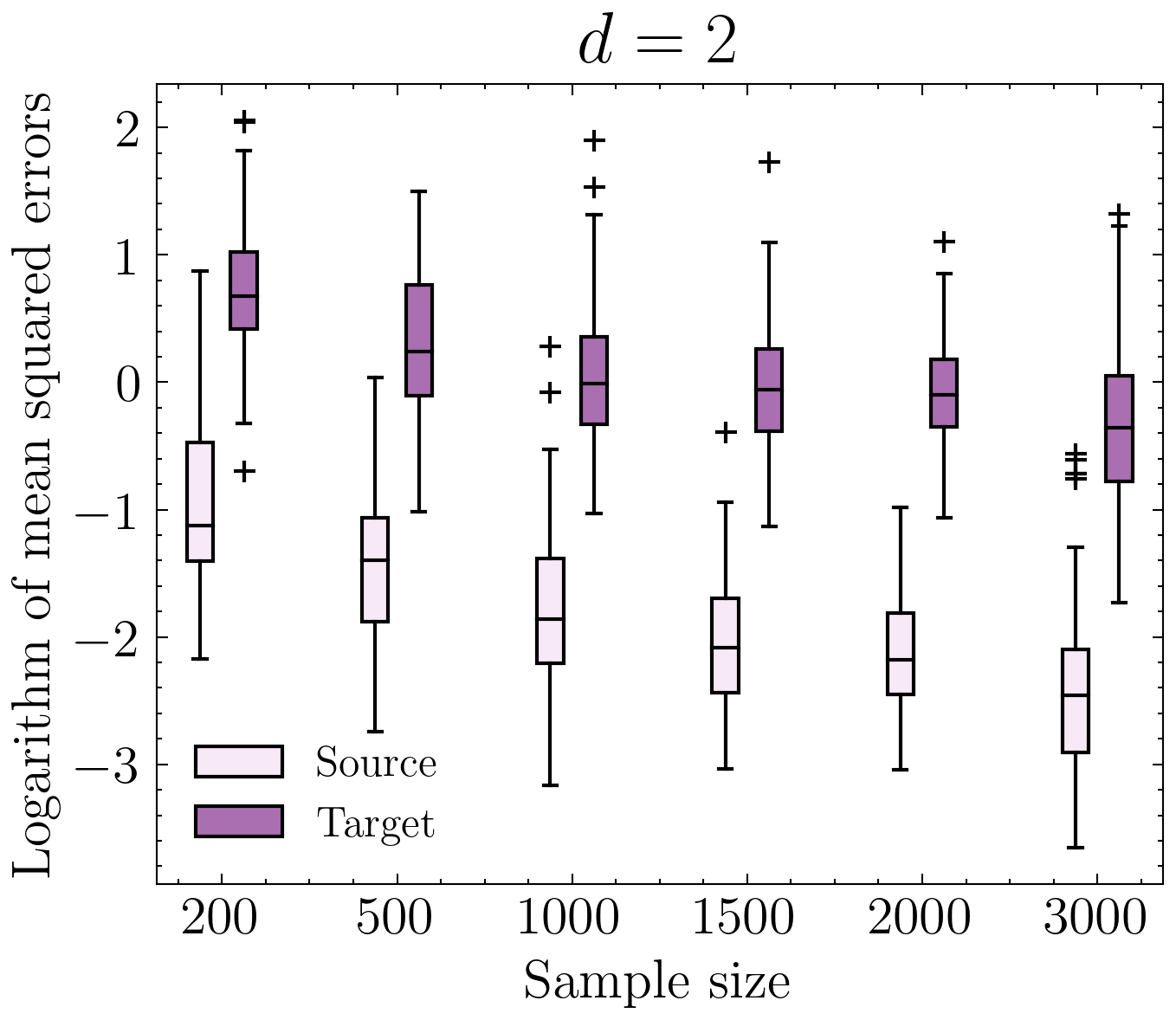}
    \includegraphics[width=0.3\linewidth]{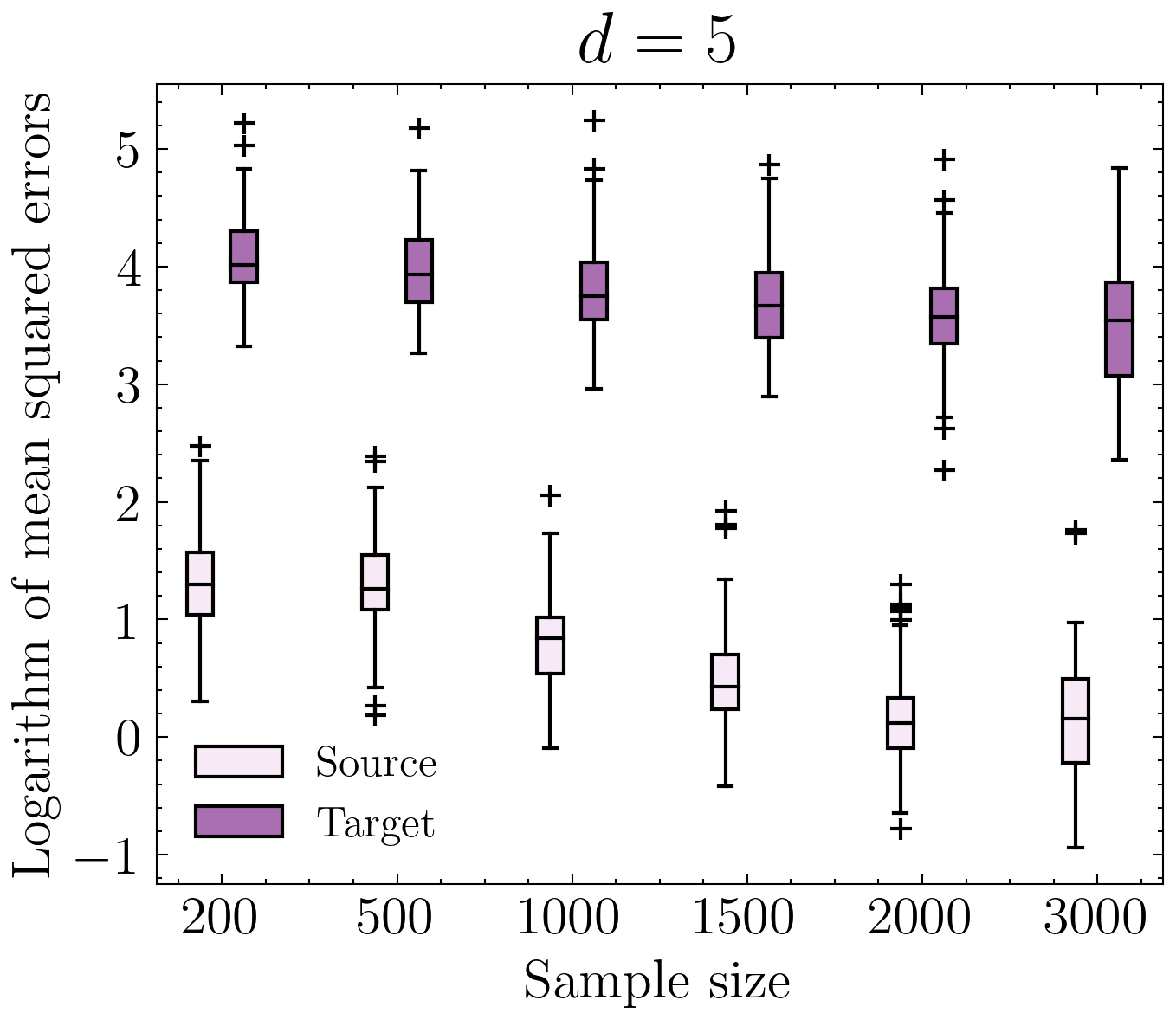}
    \caption{Boxplots of logarithm of mean square errors between the true density ratio values and predicted values, based on results from 100 replications.}
    \label{fig: results_dre}
\end{figure}

Table \ref{tab: results_dre} and Figure \ref{fig: results_dre} illustrate the performance of density ratio estimators across source and target domains. A consistent inverse relationship between sample size and mean squared errors, with the latter asymptotically approaching zero as the former increased, was noticeable. In the univariate cases ($d=1$), the estimation errors in the target domain exhibited a magnitude approximately four times those of the source domain. As the dimensionality increased, the disparity in estimation errors between the two domains became more pronounced.

\subsection{Sufficiency of source estimators for covariate shift}

In this subsection, we consider the regression model $Y=f_0(X)+\nu\varepsilon$ where $\nu>0$ and
$$
\begin{aligned}
f_0(X) &=
\begin{pmatrix}
    f_{01}(X) \\
    f_{02}(X)
\end{pmatrix}=
\begin{pmatrix}
    \sin(\pi(X_{(1)}-X_{(2)}))\log(1+X_{(3)}^2) \\
    \exp(-X_{(2)})\mathds{1}(X_{(4)}>2)
\end{pmatrix}, \\
\varepsilon &= \begin{pmatrix}
    W \\
    -W
\end{pmatrix}, \quad W\sim N(0, 1).
\end{aligned}
$$
The covariate $X$ was drawn from either the source domain ($X^s$) or the target domain ($X^t$), as defined in Subsection \ref{subsec: experiments_dre} with the dimensionality $d=5$. The parameter $\nu$ was assigned values of 0.1, 0.2 and 1, corresponding to low, moderate, and high noise levels, respectively. Such levels were calibrated relative to the variances of $f_{01}(X^s)$ and $f_{02}(X^s)$, with the moderate noise level (that is, $0.2^2$) approximating the variance of $f_{02}(X^s)$ and the high noise level approximating that of $f_{01}(X^s)$. This setup generated both source and target datasets. Specifically, for training, we sampled $n_{11}$ observations of covariates and responses from the source domain, and $n_{12}=500$ observations containing only covariates from the target domain. Here, $n_{11}$ varied among $\{500, 1000, 1500, 2000, 2500, 3000\}$. Additionally, we generated $n_2=1000$ testing data in the form $(X, f_0(X))$ for each domain.

\begin{table*}[t]
\centering
\caption{The averages and standard deviations (indicated in brackets) of the mean squared errors between the true conditional mean values and the predicted values, based on results from 100 replications. SERS denotes the source estimator risk in the source domain, while SERT represents the source estimator risk in the target domain. EDRC stands for the estimated density ratio correction, and ODRC refers to the oracle density ratio correction.}

\begin{tabular}{ccrrrr}
\hline
Noise level & Sample size ($n_{11}$) & SERS & SERT & EDRC & ODRC \\ \hline
\multirow{6}{*}{$\nu=0.1$} & 500 & 0.201 (0.061) & 0.623 (0.162) & 2.080 (0.375) & 0.819 (0.184) \\
 & 1000 & 0.110 (0.034) & 0.391 (0.092) & 1.157 (0.338) & 0.530 (0.119) \\
 & 1500 & 0.071 (0.024) & 0.265 (0.068) & 0.829 (0.280) & 0.394 (0.094) \\
 & 2000 & 0.059 (0.021) & 0.220 (0.053) & 0.572 (0.151) & 0.322 (0.084) \\
 & 2500 & 0.054 (0.022) & 0.193 (0.040) & 0.485 (0.121) & 0.283 (0.063) \\
 & 3000 & 0.032 (0.013) & 0.119 (0.042) & 0.332 (0.106) & 0.193 (0.056) \\ \hline
\multirow{6}{*}{$\nu=0.2$} & 500 & 0.246 (0.060) & 0.698 (0.164) & 2.142 (0.384) & 0.926 (0.186) \\
 & 1000 & 0.133 (0.035) & 0.439 (0.099) & 1.250 (0.355) & 0.651 (0.130) \\
 & 1500 & 0.092 (0.025) & 0.308 (0.069) & 0.859 (0.280) & 0.432 (0.089) \\
 & 2000 & 0.074 (0.021) & 0.256 (0.056) & 0.611 (0.150) & 0.355 (0.083) \\
 & 2500 & 0.063 (0.021) & 0.218 (0.045) & 0.539 (0.132) & 0.323 (0.067) \\
 & 3000 & 0.045 (0.014) & 0.150 (0.037) & 0.371 (0.104) & 0.235 (0.061) \\ \hline
\multirow{6}{*}{$\nu=1.0$} & 500 & 0.959 (0.155) & 2.000 (0.363) & 3.691 (0.616) & 2.435 (0.402) \\
 & 1000 & 0.454 (0.079) & 1.036 (0.184) & 2.586 (0.440) & 1.710 (0.309) \\
 & 1500 & 0.321 (0.048) & 0.773 (0.124) & 2.039 (0.410) & 1.394 (0.288) \\
 & 2000 & 0.264 (0.037) & 0.658 (0.093) & 1.725 (0.306) & 1.100 (0.199) \\
 & 2500 & 0.243 (0.039) & 0.606 (0.085) & 1.551 (0.354) & 0.979 (0.183) \\
 & 3000 & 0.210 (0.042) & 0.480 (0.092) & 1.344 (0.400) & 0.851 (0.157) \\
\hline
\end{tabular}
\label{tab: results_reg}
\end{table*}

By minimizing the least squares loss and utilizing merely source data, we obtained the source estimator. The neural network architecture remained identical to that described in Subsection \ref{subsec: experiments_dre}, while maintaining $\kappa$ to 0.5. For the regression task, we employed a learning rate of 1e-3, and carefully selected the number of iterations through cross validation among the candidate list $\{1000, 2000, 3000, 4000, 5000\}$. Subsequently, as benchmarks, we performed two types of loss correction methods, namely the estimated density ratio correction (EDRC) and the oracle density ratio correction (ODRC). For EDRC, we first conducted the density ratio estimation using $n_{11}$ source covariates and $n_{12}$ target covariates. Then, a corrected least squares loss based on this estimated ratio was applied to construct an estimator for $f_0$ (see Section \ref{subsec: drc_generic_properties}). For ODRC, we corrected the least squares loss using the oracle density ratio. For both correction methods, we maintained the same neural network architecture and determined the optimal number of training iterations for estimating $f_0$ through cross validation.

We assessed the performance through the mean squared error. To be more specific, we recorded the mean squared errors between true and predicted conditional mean values for the source estimator in both source and target domains using testing data. In addition, we calculated mean squared errors in the target domain with respect to estimators derived from EDRC and ODRC methods. To summarize, we obtained four risk measures, namely the source estimator risk in source domain, the source estimator risk in target domain, the EDRC estimator risk in target domain and the ODRC estimator risk in target domain. Furthermore, for each combination of $(n_{11}, \nu)$, we ran 100 replications.

\begin{figure}[!t]
    \centering
    \includegraphics[width=0.7\linewidth]{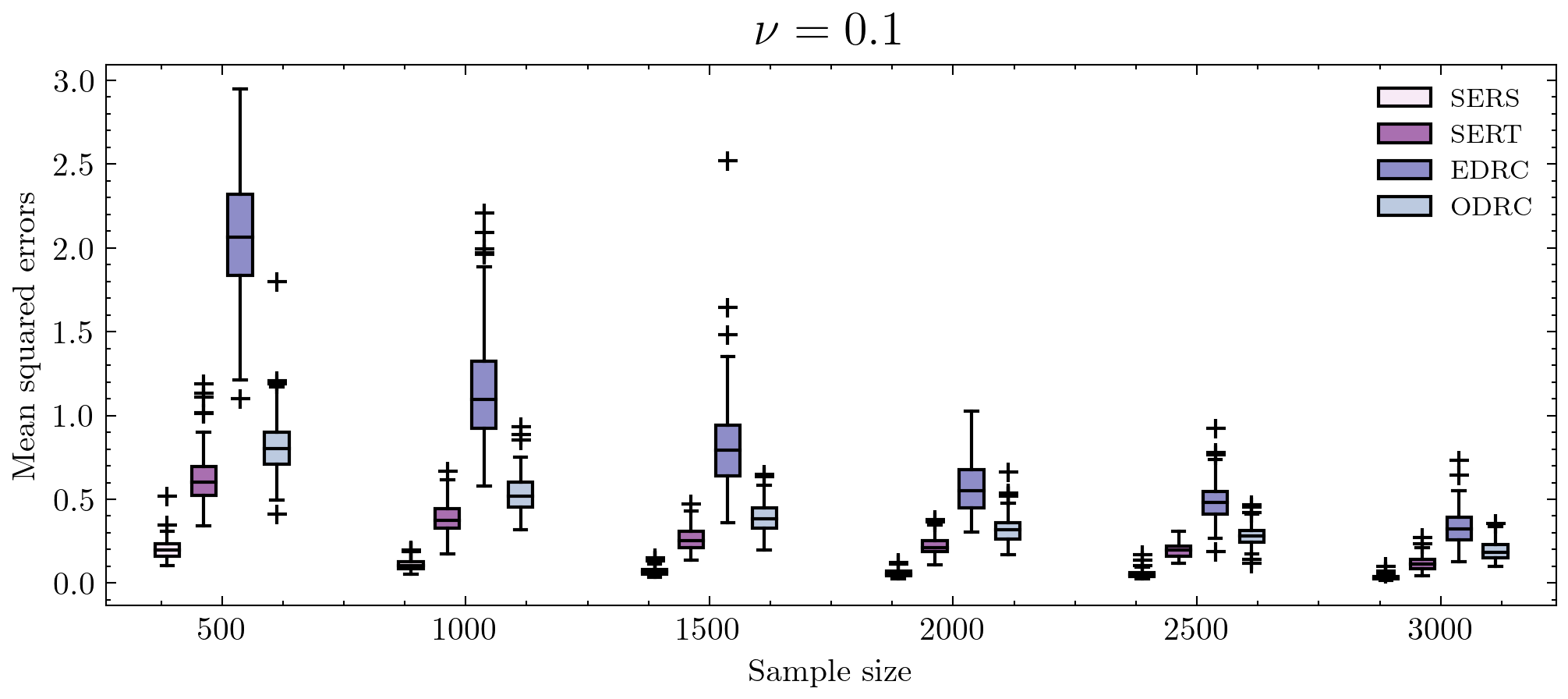}
    \includegraphics[width=0.7\linewidth]{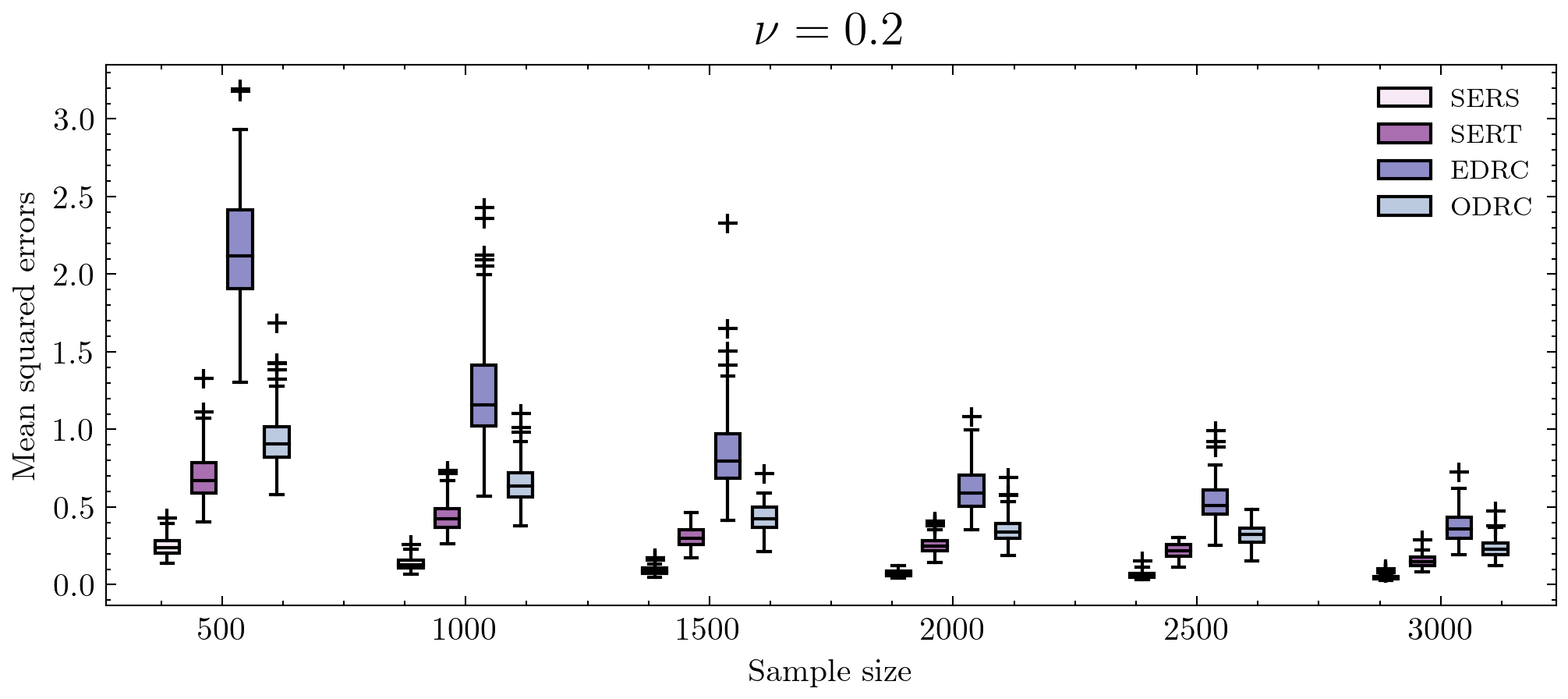}
    \includegraphics[width=0.7\linewidth]{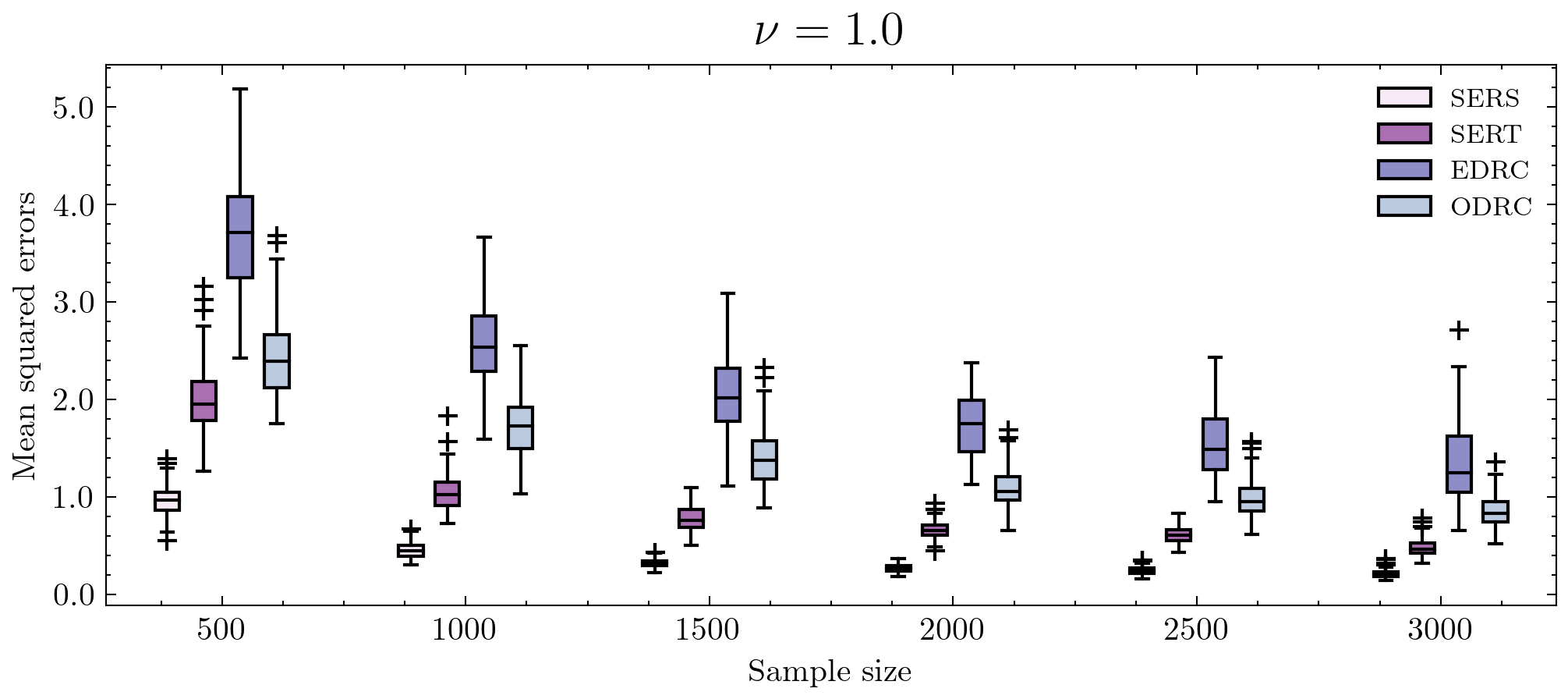}
    \caption{Boxplots of mean square errors between the true conditional mean values and predicted values, based on results from 100 replications. SERS means the source estimator risk in source domain and SERT represents the source estimator risk in target domain.}
    \label{fig: results_reg}
\end{figure}

It is evident from Table \ref{tab: results_reg} and Figure \ref{fig: results_reg} that the source estimator's risk in the target domain decreased commensurately with its risk in the source domain, demonstrating the risk controllability of the source estimator in the target domain. Notably, the source estimator significantly outperformed the estimator based on EDRC method. It is particularly surprising and interesting that even with access to the true density ratio, the ODRC estimator showed its weakness compared to the source estimator, which lacked this additional information. This observation substantiates the fragility of density ratio correction \citep{li2020robust} from the empirical perspective.

\section{Conclusion}\label{sec: conclusion}
In this paper, we address the problem of density ratio estimation, allowing density ratios with unbounded domains and ranges. We develop a rigorous theoretical framework for density ratio estimators based on Bregman divergences, including  least squares and logistic regression loss. Our findings contribute significantly to the existing literature on the estimation theory of density ratios.

To demonstrate the applications of our results in unbounded density ratio estimation, we study nonparametric regression and conditional flow models under covariate shift. We discover that the tail properties of the density ratio are crucial for ensuring risk transferability across different domains. Based on suitable tail conditions and divergent patterns of the density ratio function, we show that the source estimator is nearly optimal in the target domain.
Our numerical results support these theoretical insights, particularly highlighting that the source estimator can outperform estimators derived from loss correction methods, even when the true density ratio is known.

We caution that, in general, the near-optimality of the source estimator cannot be guaranteed without explicitly accounting for covariate shift. For instance, in constrained RKHS-based nonparametric regression, \cite{ma2023optimally} showed that
 there exists a specific pair of random elements $((X^s, Y^s), (X^t, Y^t))$ characterized by a particular probability structure, such that a kernel regression estimator for the conditional mean has a slower convergence rate in the target domain compared to the source domain.
However, if the density ratio is unknown, constructing an optimal estimator in this constrained kernel regression setting requires further investigation.

Several other directions merit exploration. Beyond the covariate shift problem addressed in this work, our density estimation results have potential applications in areas where density ratios are crucial, such as transfer learning, optimal transport methods for generative learning \citep{gao2022deep}, mutual information estimation, and propensity score estimation \citep{lei2021conformal}. Moreover, the techniques developed in this work for handling density ratios with unbounded domains and ranges could be adapted and extended to other settings where unbounded functions arise, such as score-based generative models.

\bigskip\bigskip

\newpage
\begin{appendix}

\bigskip\noindent
\textbf{\LARGE Appendix}

\medskip
In the Appendix, we provide proofs of the results presented in the paper, along with additional technical details.

\section{Auxiliary lemmas}

\subsection{Regularity of Bregman divergence}

\begin{lemma}\label{lem: bregman_divergence_solution}
Let $\varphi: \mathcal{X}\to\mathbb{R}$ be a differentiable and strictly convex function where $\mathcal{X}\subset\mathbb{R}$ is a convex set. Then, the deduced Bregman divergence $D_{\varphi}(x\| y)=0$ implies $x=y$.
\end{lemma}

\begin{proof}[Proof of Lemma \ref{lem: bregman_divergence_solution}]
Suppose that there exist distinct $x, y\in \mathcal{X}$ such that $D_{\varphi}(x\| y)=0$. By the definition of $D_{\varphi}(x\| y)$, we have
\begin{equation}\label{eqn: bregman_divergence_solution-proof_1}
\varphi(x)=\varphi(y)+\varphi'(y)(x-y).
\end{equation}
For arbitrary $t\in (0, 1)$, the strictly convexity of $\varphi$ indicates that
\begin{equation}\label{eqn: bregman_divergence_solution-proof_2}
\varphi(y+t(x-y))=\varphi(tx+(1-t)y)<t\varphi(x)+(1-t)\varphi(y)
\end{equation}
Combining Eqns.~\eqref{eqn: bregman_divergence_solution-proof_1} and \eqref{eqn: bregman_divergence_solution-proof_2}, we obtain
$$
\varphi(y+t(x-y))<\varphi(y)+t\varphi'(y)(x-y).
$$
However, this contradicts to the fact that $\varphi$ is a convex function. Therefore, distinct $x, y$ can not yield that $D_{\varphi}(x\| y)=0$. When $x$ equals to $y$, it is straightforward to verify that $D_{\varphi}(x\| y)=0$. This completes the proof.
\end{proof}

\subsection{Approximation properties of neural networks}

\begin{lemma}[Theorem 3.3 in \cite{jiao2023deep}]\label{lem: thm_3.3_jiao}
Assume $g\in \mathcal{H}^{\beta}([0, 1]^d, B)$. For any $S_1, S_2\in\mathbb{N}_+$, there exists a function $f$ implemented by ReLU feedforward neural network with depth $L=21(\lfloor\beta\rfloor+1)^2S_1\lceil\log_2(8S_1)\rceil+2d$, width $M=38(\lfloor\beta\rfloor+1)^2d^{\lfloor\beta\rfloor+1}S_2\lceil\log_2(8S_2)\rceil$, such that
$$
|f(x)-g(x)|\le 18B(\lfloor\beta\rfloor+1)^2d^{\lfloor\beta\rfloor+(\beta\vee 1)/2}(S_1S_2)^{-2\beta/d},
$$
for all $x\in [0, 1]^d\backslash \Omega([0, 1]^d, K, \delta)$. Here, $\mathbb{N}_+$ denotes the set of positive integers, $\lceil a\rceil$ means the smallest integer no less than $a$, $a\vee b=\max(a, b)$, and
$$
\Omega([0, 1]^d, K, \delta)=\bigcup_{i=1}^d\left\{x=(x_1, \dots, x_d)^{\top}: x_i\in\bigcup_{k=1}^{K-1}(k/K-\delta, k/K)\right\},
$$
where $K=\lceil(S_1S_2)^{2/d}\rceil$ and $\delta$ is an arbitrary scalar in $(0, 1/(3K)]$.
\end{lemma}

\begin{lemma}[Corollary B.2 in \cite{gao2024convergence}]\label{lem: cor_b.2_gao}
Given any $g\in\mathcal{W}^{1, \infty}((0, 1)^d)$ with $\|g\|_{\mathcal{W}^{1, \infty}((0, 1)^d)}<\infty$, for any $S_1, S_2\in\mathbb{N}_+$, there exists a function $f$ implemented by a deep ReLU network with depth $\mathcal{O}(d^2S_1\log S_1)$ and width $\mathcal{O}(2^ddS_2\log S_2)$ such that $\|f\|_{\mathcal{W}^{1, \infty}((0, 1)^d)}\le C_1\|g\|_{\mathcal{W}^{1, \infty}((0, 1)^d)}$ and
$$
|f(x)-g(x)|\le C_2\|g\|_{\mathcal{W}^{1, \infty}((0, 1)^d)}(S_1S_2)^{-2/d},
$$
for all $x\in (0, 1)^d$. Here, $C_1$ and $C_2$ are constants depending only on $d$.
\end{lemma}

\subsection{Complexity of neural network function space}

Let $\mathcal{D}$ be a subset of $\mathbb{R}^d$. Given a positive real number $\epsilon$, a set $\mathcal{C}$ is called an $\epsilon$-covering set of $\mathcal{D}$ with respect to the supremum norm if $\mathcal{C}\subset \mathcal{D}$ and for any $x\in\mathcal{D}$, there exists $z\in\mathcal{C}$ such that $\|x-z\|_{\infty}<\epsilon$. Then, the minimal cardinality of all possible $\mathcal{C}$ is termed the covering number of $\mathcal{D}$, denoted as $\mathcal{N}(\epsilon, \|\cdot\|_{\infty}, \mathcal{D})$.

Furthermore, consider a function space $\mathcal{F}$ whose elements are defined on $\mathcal{X}$. Given an integer $n$ and $\mathcal{D}_n=\{x_1, \dots, x_n\}\subset \mathcal{X}^n$, the covering number of $\mathcal{F}$ constrained on $\mathcal{D}_n$ is defined as the covering number of $\mathcal{F}_{|\mathcal{D}_n}$, where
$$
\mathcal{F}_{|\mathcal{D}_n}=\left\{(f(x_1), \dots, f(x_n))^{\top}: f\in\mathcal{F}\right\}.
$$
Then, the covering number of $\mathcal{F}$ with respect to $n$, denoted as $\mathcal{N}_n(\epsilon, \|\cdot\|_{\infty}, \mathcal{F})$, is defined as $\sup_{\mathcal{D}_n}\mathcal{N}(\epsilon, \|\cdot\|_{\infty}, \mathcal{F}_{|\mathcal{D}_n})$.

\begin{lemma}[Theorem 12.2 in \cite{anthony1999neural}]\label{lem: thm_12.2_anthony}
Let $\mathcal{F}$ be a set of real functions that map from a domain $\mathcal{X}$ to a bounded interval $[0, B]$. Denote the pseudo-dimension of $\mathcal{F}$ as $\mathrm{Pdim}(\mathcal{F})$. Then, for $n\ge \mathrm{Pdim}(\mathcal{F})$ and $B\ge \epsilon$, we have
$$
\mathcal{N}_n(\epsilon, \|\cdot\|_{\infty}, \mathcal{F})\le \left(\frac{eBn}{\epsilon \mathrm{Pdim}(\mathcal{F})}\right)^{\mathrm{Pdim}(\mathcal{F})}.
$$
\end{lemma}

\begin{lemma}[Theorem 7 in \cite{bartlett2019nearly}]\label{lem: thm_7_bartlett}
Let $\mathcal{F}_{\mathrm{NN}}$ be a neural network function class with depth $L$ and number of parameters $S$. Then, there exists a universal constant $C$ such that
$$
\mathrm{Pdim}(\mathcal{F}_{\mathrm{NN}})\le CSL\log S.
$$
\end{lemma}

\subsection{Stochastic error analysis}

\begin{theorem}\label{thm: gen_thm_11.4_gyorfi}
Let $Z$ be a random vector supported on $\mathcal{Z}\subset \mathbb{R}^d$, and let $\mathcal{D}_n=\{Z_1, \dots, Z_n\}$ be a random sample of $Z$, whose elements are independent and identically distributed. Let $\mathcal{F}_n$ be a nonrandom function space with elements real-valued. For a functional $g: \mathcal{F}_n\times \mathcal{Z}\to \mathbb{R}$, assume there exist some universal sequences $\xi_n, \zeta_n>0$, such that (i) $\sup_{f\in \mathcal{F}_n, Z\in\mathcal{Z}}|g(f, Z)|\le \xi_n$, (ii) $\mathbb{E}[g(f, Z)^2]\le \zeta_n\mathbb{E}[g(f, Z)]$ for all $f\in\mathcal{F}_n$, where the expectation are taken with respect to $Z$, and (iii) there exists an integer $N>0$ such that for any $n\ge N$, $\zeta_n\ge \eta$ for some constant $\eta\ge 1$, and $\xi_n\le \gamma\zeta_n$ for some constant $\gamma>0$. Then, for $n\ge N$, we have
$$
\begin{aligned}
& \mathbb{P}\left(\exists f\in\mathcal{F}_n: \mathbb{E}[g(f, Z)]-\frac 1n\sum_{i=1}^ng(f, Z_i)\ge \epsilon\left\{\alpha+\beta+\mathbb{E}[g(f, Z)]\right\}\right) \\
\le & 14\mathcal{N}_n\left(\varpi\epsilon\beta, \|\cdot \|_{\infty}, \{g(f, \cdot): \mathcal{Z}\to\mathbb{R}, f\in\mathcal{F}_n\}\right)\exp\left(-\frac{27\epsilon^2(1-\epsilon)\alpha\omega n}{40(\xi_n^2\vee \zeta_n)(1+\epsilon)}\right).
\end{aligned}
$$
where $\alpha, \beta>0$, $0<\epsilon \le1/2$, $\varpi=(6\eta-2)/(30\eta+3\gamma\eta)$, and $\omega=400/(\gamma+60)^2$.
\end{theorem}

\begin{remark}
Theorem \ref{thm: gen_thm_11.4_gyorfi} is a generalization of Theorem 11.4 in \cite{gyorfi2002distribution}. Particularly, the case that $Z=(X, Y)$ and $g(f, Z)=|f(X)-Y|^2-|\mathbb{E}(Y|X)-Y|^2$ represents classical nonparametric regression. The proof of Theorem \ref{thm: gen_thm_11.4_gyorfi} can be found in Appendix \ref{sec: proof_of_thm_gen_thm_11.4_gyorfi}.
\end{remark}

\subsection{Gr{\"o}wnwall's inequality}

\begin{lemma}[Lemma 37 in \cite{jiao2024convergence}]\label{lem: grownwall_inequality}
Suppose that $f(t)$ is a scalar-output function defined on the interval $[a, b]$, satisfying $\mathrm{d}f(t)/\mathrm{d}t\le \alpha f(t)+g(t)$ with some constant $\alpha\ge 0$. Then, we have
$$
f(b)\le e^{\alpha (b-a)}f(a)+\int_a^be^{\alpha(b-t)}g(t)\mathrm{d}t.
$$
\end{lemma}

\section{Proofs of main results}

For a sub-exponentially distributed random variable $X$, there exists a universal constant $\varsigma>0$ such that $\mathbb{E}\exp(\varsigma|X|)<\infty$, where $\varsigma$ is termed the sub-exponential index. Throughout the subsequent proofs, we let $\varsigma$ be a uniform sub-exponential index without loss of generality, due to the finite number of associated sub-exponentially distributed random variables we will handle.

\subsection{Proof of Theorem 3.1} \label{subsec: proof_of_thm_dre_convergence_rate_ls}

For any square-integrable function $f: \mathbb{R}^d\to\mathbb{R}$, define
$$
\begin{aligned}
L(f) &= \mathbb{E}_P[f(X)^2-2f(X)r_0(X)], \\
L_n(f) &= \frac 1n\sum_{i=1}^nf(X^s_i)^2-\frac 2n\sum_{i=1}^nf(X^t_i).
\end{aligned}
$$
Here, $\mathbb{E}_P[h(X)]\equiv \mathbb{E}[h(X^s)]$ for any $X^s$-integrable function $h$, where the expectation is taken with respect to $X^s$. Analogously, $\mathbb{E}_Q[h(X)]\equiv \mathbb{E}[h(X^t)]$ for any $X^t$-integrable function $h$.

\begin{lemma}\label{lem: dre_error_decomposition_ls}
Assume that $r_0(X^s)$ is a square-integrable. Then, $\mathcal{R}^s(\hat{r}_{\mathrm{LS}})\le \mathbb{E}[L(\hat{r}_{\mathrm{LS}})-2L_n(\hat{r}_{\mathrm{LS}})+L(r_0)]+2\inf_{f\in\mathcal{F}_{\mathrm{NN}}}\mathbb{E}_P[f(X)-r_0(X)]^2$.
\end{lemma}

\begin{proof}[Proof of Lemma \ref{lem: dre_error_decomposition_ls}]
For any $f\in\mathcal{F}_{\mathrm{NN}}\subset\mathcal{L}^2(X^s)$, we have
$$
\begin{aligned}
\mathcal{R}^s(\hat{r}_{\mathrm{LS}}) &= \mathbb{E}[L(\hat{r}_{\mathrm{LS}})-L(r_0)] \\
&\le \mathbb{E}[L(\hat{r}_{\mathrm{LS}})-L(r_0)]+2\mathbb{E}[L_n(f)-L_n(\hat{r}_{\mathrm{LS}})] \\
&= \mathbb{E}[L(\hat{r}_{\mathrm{LS}})-L(r_0)]+2\mathbb{E}[L_n(f)-L_n(r_0)+L_n(r_0)-L_n(\hat{r}_{\mathrm{LS}})] \\
&= \mathbb{E}[L(\hat{r}_{\mathrm{LS}})-2L_n(\hat{r}_{\mathrm{LS}})+L(r_0)]+2[L(f)-L(r_0)] \\
&= \mathbb{E}[L(\hat{r}_{\mathrm{LS}})-2L_n(\hat{r}_{\mathrm{LS}})+L(r_0)]+2\mathbb{E}_P[f(X)-r_0(X)]^2.
\end{aligned}
$$
This indicates that $\mathcal{R}^s(\hat{r}_{\mathrm{LS}})\le \mathbb{E}[L(\hat{r}_{\mathrm{LS}})-2L_n(\hat{r}_{\mathrm{LS}})+L(r_0)]+2\inf_{f\in\mathcal{F}_{\mathrm{NN}}}\mathbb{E}_P[f(X)-r_0(X)]^2$.
\end{proof}

\begin{lemma}\label{lem: dre_stochatic_error_bound_ls}
Assume that $r_0(X^s)$ is sub-exponentially distributed. Let $\bar{\delta}=\bar{\delta}_n=(\log n)^{1+\kappa}$, with an arbitrarily fixed $\kappa\in (0, 1]$, and let $\underline{\delta}=0$. Then, for sufficiently large $n$ and $n\ge \mathrm{Pdim}(\mathcal{F}_{\mathrm{NN}})$, it follows that
$$
\mathbb{E}[L(\hat{r}_{\mathrm{LS}})-2L_n(\hat{r}_{\mathrm{LS}})+L(r_0)]\le \frac{c^*SL\log S(\log n)^{5+4\kappa}}{n},
$$
where $c^*$ is a constant not depending on $S, L$ and $n$.
\end{lemma}

\begin{proof}[Proof of Lemma \ref{lem: dre_stochatic_error_bound_ls}]
Let $T_i=(X^s_i, X^t_i)$ for $i=1, \dots, n$, $\mathcal{D}_n=\{T_1, \dots, T_n\}$, and $T=(X^s, X^t)$ be an independent copy of $T_1$. Firstly, we have
$$
\begin{aligned}
\mathbb{E}[L(\hat{r}_{\mathrm{LS}})-2L_n(\hat{r}_{\mathrm{LS}})+L(r_0)] &= \mathbb{E}_{\mathcal{D}_n}[L(\hat{r}_{\mathrm{LS}})-2L_n(\hat{r}_{\mathrm{LS}})+L(r_0)] \\
&= \mathbb{E}_{\mathcal{D}_n}\{L(\hat{r}_{\mathrm{LS}})-L(r_0)-2[L_n(\hat{r}_{\mathrm{LS}})-L_n(r_0)]\} \\
&= \mathbb{E}_{\mathcal{D}_n}\left\{\mathbb{E}_{T}[g(\hat{r}_{\mathrm{LS}}, T)]-\frac 2n\sum_{i=1}^ng(\hat{r}_{\mathrm{LS}}, T_i)\right\},
\end{aligned}
$$
where $g(f, T)=[f(X^s)^2-r_0(X^s)^2]-2[f(X^t)-r_0(X^t)]$ for $f\in\mathcal{F}_{\mathrm{NN}}$. Next, we introduce a truncation step. For any $\iota_n>0$ and any $f\in\mathcal{F}_{\mathrm{NN}}$, define
$$
g_{\iota_n}(f, T)=[f(X^s)^2-r_0(X^s)^2]\mathds{1}(r_0(X^s)\le \iota_n)-2[f(X^t)-r_0(X^t)]\mathds{1}(r_0(X^t)\le \iota_n).
$$
Then, it follows that
$$
\begin{aligned}
|g(f, T)-g_{\iota_n}(f, T)| &\le |f(X^s)^2-r_0(X^s)^2|\mathds{1}(r_0(X^s)>\iota_n)+2|f(X^t)-r_0(X^t)|\mathds{1}(r_0(X^t)>\iota_n) \\
&\le [\bar{\delta}_n^2+r_0(X^s)^2]\mathds{1}(r_0(X^s)>\iota_n)+2[\bar{\delta}_n+r_0(X^t)]\mathds{1}(r_0(X^t)>\iota_n).
\end{aligned}
$$
Taking expectation, we have
$$
\begin{aligned}
& \mathbb{E}[L(\hat{r}_{\mathrm{LS}})-2L_n(\hat{r}_{\mathrm{LS}})+L(r_0)] \\
=& \mathbb{E}_{\mathcal{D}_n}\left\{\mathbb{E}_{T}[g(\hat{r}_{\mathrm{LS}}, T)]-\frac 2n\sum_{i=1}^ng(\hat{r}_{\mathrm{LS}}, T_i)\right\} \\
\le & \mathbb{E}_{\mathcal{D}_n}\left\{\mathbb{E}_{T}[g_{\iota_n}(\hat{r}_{\mathrm{LS}}, T)]-\frac 2n\sum_{i=1}^ng_{\iota_n}(\hat{r}_{\mathrm{LS}}, T_i)\right\}+3\mathbb{E}_P\left\{[\bar{\delta}_n^2+r_0(X)^2]\mathds{1}(r_0(X)>\iota_n)\right\} \\
& +6\mathbb{E}_Q\left\{[\bar{\delta}_n+r_0(X)]\mathds{1}(r_0(X)>\iota_n)\right\}.
\end{aligned}
$$
Specifically, observe that
$$
\begin{aligned}
& \mathbb{E}_P\left\{[\bar{\delta}_n^2+r_0(X)^2]\mathds{1}(r_0(X)>\iota_n)\right\} \\
\le & \bar{\delta}_n^2\mathbb{E}_P[\exp(\varsigma r_0(X)/2)]\exp(-\varsigma\iota_n/2)+\mathbb{E}_P\left[r_0(X)^2\mathds{1}(r_0(X)>\iota_n)\right] \\
\le & \bar{\delta}_n^2\mathbb{E}_P[\exp(\varsigma r_0(X)/2)]\exp(-\varsigma\iota_n/2)+\frac{16}{\varsigma^2}\mathbb{E}_P[\exp(\varsigma r_0(X))]\exp(-\varsigma\iota_n/2),
\end{aligned}
$$
and
$$
\begin{aligned}
& \mathbb{E}_Q\left\{[\bar{\delta}_n+r_0(X)]\mathds{1}(r_0(X)>\iota_n)\right\} \\
=& \mathbb{E}_P\left\{[\bar{\delta}_n+r_0(X)]r_0(X)\mathds{1}(r_0(X)>\iota_n)\right\} \\
\le & \frac{2}{\varsigma}\bar{\delta}_n\mathbb{E}_P[\exp(r_0(X))]\exp(-\varsigma\iota_n/2)+\frac{16}{\varsigma^2}\mathbb{E}_P[\exp(\varsigma r_0(X))]\exp(-\varsigma\iota_n/2).
\end{aligned}
$$
Here, we have applied the inequalities $a\le \exp(a)$ and $\mathds{1}(a>0)\le \exp(a)$ for $a\in\mathbb{R}$. As a consequence, we obtain
$$
\begin{aligned}
& \mathbb{E}[L(\hat{r}_{\mathrm{LS}})-2L_n(\hat{r}_{\mathrm{LS}})+L(r_0)] \\
\le & \mathbb{E}_{\mathcal{D}_n}\left\{\mathbb{E}_{T}[g_{\iota_n}(\hat{r}_{\mathrm{LS}}, T)]-\frac 2n\sum_{i=1}^ng_{\iota_n}(\hat{r}_{\mathrm{LS}}, T_i)\right\}+c_1(\bar{\delta}_n^2+\bar{\delta}_n+1)\exp(-\varsigma\iota_n/2),
\end{aligned}
$$
where $c_1$ is a constant which depends only on $\varsigma$ and $\mathbb{E}_P[\exp(\varsigma r_0(X))]$. Recall that $\bar{\delta}_n=(\log n)^{1+\kappa}$. Setting $\iota_n$ to $(2\varsigma^{-1}\log n)\vee 1$ yields
$$
\begin{aligned}
& \mathbb{E}[L(\hat{r}_{\mathrm{LS}})-2L_n(\hat{r}_{\mathrm{LS}})+L(r_0)] \\
\le & \mathbb{E}_{\mathcal{D}_n}\left\{\mathbb{E}_{T}[g_{\iota_n}(\hat{r}_{\mathrm{LS}}, T)]-\frac 2n\sum_{i=1}^ng_{\iota_n}(\hat{r}_{\mathrm{LS}}, T_i)\right\}+c_1[(\log n)^{1+\kappa}+1]^2n^{-1}.
\end{aligned}
$$
Then, we proceed to verify the conditions in Theorem \ref{thm: gen_thm_11.4_gyorfi}. Notice that
$$
\sup_{f\in\mathcal{F}_{\mathrm{NN}}, T\in\mathbb{R}^{2d}}|g_{\iota_n}(f, T)|\le \bar{\delta}_n^2+\iota_n^2+2\bar{\delta}_n+2\iota_n\le 6\bar{\delta}_n^2=6(\log n)^{2+2\kappa},
$$
whenever $\bar{\delta}_n\ge \iota_n$. Furthermore, for any $f\in\mathcal{F}_{\mathrm{NN}}$,
$$
\begin{aligned}
& \mathbb{E}[g_{\iota_n}(f, T)] \\
=& \mathbb{E}_P\left\{[f(X)^2-r_0(X)^2]\mathds{1}(r_0(X)\le\iota_n)\right\}-2\mathbb{E}_Q\left\{[f(X)-r_0(X)]\mathds{1}(r_0(X)\le\iota_n)\right\} \\
=& \mathbb{E}_P\left\{[f(X)^2-r_0(X)^2]\mathds{1}(r_0(X)\le\iota_n)\right\}-2\mathbb{E}_P\left\{[f(X)-r_0(X)]r_0(X)\mathds{1}(r_0(X)\le\iota_n)\right\} \\
=& \mathbb{E}_P\left\{[f(X)-r_0(X)]^2\mathds{1}(r_0(X)\le\iota_n)\right\},
\end{aligned}
$$
and
$$
\begin{aligned}
& \mathbb{E}[g_{\iota_n}(f, T)^2] \\
=& \mathbb{E}_P\left\{[f(X)^2-r_0(X)^2]^2\mathds{1}(r_0(X)\le\iota_n)\right\}+4\mathbb{E}_Q\left\{[f(X)-r_0(X)]^2\mathds{1}(r_0(X)\le\iota_n)\right\} \\
& -4\mathbb{E}_P\left\{[f(X)^2-r_0(X)^2]\mathds{1}(r_0(X)\le\iota_n)\right\}\mathbb{E}_Q\left\{[f(X)-r_0(X)]\mathds{1}(r_0(X)\le\iota_n)\right\} \\
\le & 2(\bar{\delta}_n^2+\iota_n^2)\mathbb{E}_P\left\{[f(X)-r_0(X)]^2\mathds{1}(r_0(X)\le\iota_n)\right\} \\
& +4\iota_n\mathbb{E}_P\left\{[f(X)-r_0(X)]^2\mathds{1}(r_0(X)\le\iota_n)\right\} \\
& +4\iota_n(\bar{\delta}_n+\iota_n)\mathbb{E}_P\left\{[f(X)-r_0(X)]^2\mathds{1}(r_0(X)\le\iota_n)\right\} \\
\le & 16\bar{\delta}_n^2\mathbb{E}[g_{\iota_n}(f, T)] \\
=& 16(\log n)^{2+2\kappa}\mathbb{E}[g_{\iota_n}(f, T)],
\end{aligned}
$$
provided that $n$ is sufficiently large such that $\bar{\delta}_n\ge \iota_n$. Hence, Theorem \ref{thm: gen_thm_11.4_gyorfi} suggests that, for sufficiently large $n$ such that $\bar{\delta}_n\ge \iota_n$, with $n\ge \mathrm{Pdim}(\mathcal{F}_{\mathrm{NN}})$, and for arbitrary $t>0$, we have
$$
\begin{aligned}
& \mathbb{P}_{\mathcal{D}_n}\left\{\mathbb{E}_{T}[g_{\iota_n}(\hat{r}_{\mathrm{LS}}, T)]-\frac 2n\sum_{i=1}^ng_{\iota_n}(\hat{r}_{\mathrm{LS}}, T_i)\ge t\right\} \\
\le & \mathbb{P}_{\mathcal{D}_n}\left\{\mathbb{E}_{T}[g_{\iota_n}(\hat{r}_{\mathrm{LS}}, T)]-\frac 1n\sum_{i=1}^ng_{\iota_n}(\hat{r}_{\mathrm{LS}}, T_i)\ge \frac 12\left\{\frac t2+\frac t2+\mathbb{E}_{T}[g_{\iota_n}(\hat{r}_{\mathrm{LS}}, T)]\right\}\right\} \\
\le & \mathbb{P}_{\mathcal{D}_n}\left(\exists f\in\mathcal{F}_{\mathrm{NN}}: \mathbb{E}[g_{\iota_n}(f, Z)]-\frac 1n\sum_{i=1}^ng_{\iota_n}(f, Z_i)\ge \frac 12\left\{\frac t2+\frac t2+\mathbb{E}_{T}[g_{\iota_n}(f, T)]\right\}\right) \\
\le & 14\mathcal{N}_n\left(c_2t, \|\cdot \|_{\infty}, \{g_{\iota_n}(f, \cdot): \mathcal{X}^t\times \mathcal{X}^s\to\mathbb{R}, f\in\mathcal{F}_{\mathrm{NN}}\}\right)\exp\left(-\frac{nt}{c_3(\log n)^{4+4\kappa}}\right),
\end{aligned}
$$
where $c_2, c_3$ are universal constants and $\mathcal{X}^s, \mathcal{X}^t$ represents the domain of $X^s, X^t$, respectively. Subsequently, we bound the covering number. Fix $\{x^s_1, \dots, x^s_n\}\subset (\mathcal{X}^s)^n$ and $\{x^t_1, \dots, x^t_n\}\subset (\mathcal{X}^t)^n$. Let $\mathcal{C}=\{x^s_1, \dots, x^s_n, x^t_1, \dots, x^t_n\}$, and let $h^{\sharp}=\{h_1, \dots, h_k\}$ be an $\epsilon$-covering set of $\mathcal{F}_{\mathrm{NN}|\mathcal{C}}$ where $h_i=f_{i|\mathcal{C}}$ for some $f_i\in\mathcal{F}_{\mathrm{NN}} (i=1, \dots, k)$, such that for any $f\in \mathcal{F}_{\mathrm{NN}}$, there exists $h^*=f^*_{|\mathcal{C}}\in h^{\sharp}$ satisfying $\|h^*-f_{|\mathcal{C}}\|_{\infty}<\epsilon$. This indicates
$$
\begin{aligned}
& |g_{\iota_n}(f, (x^s_i, x^t_i))-g_{\iota_n}(f^*, (x^s_i, x^t_i))| \\
\le & |f(x^s_i)^2-f^*(x^s_i)^2|+2|f(x^t_i)-f^*(x^t_i)| \\
\le & 2(\bar{\delta}_n+1)\epsilon.
\end{aligned}
$$
Therefore,
$$
\mathcal{N}_n(c_2t, \|\cdot \|_{\infty}, \{g_{\iota_n}(f, \cdot): \mathcal{X}^t\times \mathcal{X}^s\to\mathbb{R}, f\in\mathcal{F}_{\mathrm{NN}}\})\le \mathcal{N}_{2n}(c_2t/[2(\bar{\delta}_n+1)], \|\cdot \|_{\infty}, \mathcal{F}_{\mathrm{NN}}).
$$
Then, with Lemma \ref{lem: thm_12.2_anthony} and Lemma \ref{lem: thm_7_bartlett}, for sufficiently large $n$ with $n\ge \mathrm{Pdim}(\mathcal{F}_{\mathrm{NN}})$ and any $a_n\ge 1/n$, we have
$$
\begin{aligned}
& \mathbb{E}_{\mathcal{D}_n}\left\{\mathbb{E}_{T}[g_{\iota_n}(\hat{r}_{\mathrm{LS}}, T)]-\frac 2n\sum_{i=1}^ng_{\iota_n}(\hat{r}_{\mathrm{LS}}, T_i)\right\} \\
\le & a_n+14\int_{a_n}^{\infty}\mathcal{N}_{2n}(c_2t/[2(\bar{\delta}_n+1)], \|\cdot \|_{\infty}, \mathcal{F}_{\mathrm{NN}})\exp\left(-\frac{nt}{c_3(\log n)^{4+4\kappa}}\right)\mathrm{d}t \\
\le & a_n+14\mathcal{N}_{2n}(c_2a_n/[2(\bar{\delta}_n+1)], \|\cdot \|_{\infty}, \mathcal{F}_{\mathrm{NN}})\int_{a_n}^{\infty}\exp\left(-\frac{nt}{c_3(\log n)^{4+4\kappa}}\right)\mathrm{d}t \\
\le & a_n+14\left(c_4n^2\bar{\delta}_n^2\right)^{c_5SL\log S}\cdot \frac{c_3(\log n)^{4+4\kappa}}{n}\exp\left(-\frac{na_n}{c_3(\log n)^{4+4\kappa}}\right),
\end{aligned}
$$
where $c_4$ and $c_5$ are universal constants. Choose
$$
a_n=\frac{c_3c_5(\log n)^{4+4\kappa}}{n}SL\log S\log\left(c_4n^2\bar{\delta}_n^2\right).
$$
For sufficiently large $n$, we have
$$
\mathbb{E}_{\mathcal{D}_n}\left\{\mathbb{E}_{T}[g_{\iota_n}(\hat{r}_{\mathrm{LS}}, T)]-\frac 2n\sum_{i=1}^ng_{\iota_n}(\hat{r}_{\mathrm{LS}}, T_i)\right\}\le \frac{c_6SL\log S(\log n)^{5+4\kappa}}{n},
$$
where $c_6$ is a constant not depending on $S, L$ and $n$. This completes the proof.
\end{proof}

\begin{lemma}\label{lem: dre_approximation_error_bound_ls}
Assume that
\begin{enumerate}[label=(\roman*)]
    \item $r_0(x)\in \mathcal{H}^{\beta_r}_{\mathrm{Loc}}(\mathbb{R}^d, B_u)$ with $\beta_r>0$ and $B_u\le c(u^m+1)$ for some universal constants $c>0$, $m\ge 0$;
    \item $r_0(X^s)$ and $\|X^s\|_{\infty}$ are sub-exponentially distributed random variables.
\end{enumerate}
Suppose that the depth $L$ and width $M$ of $\mathcal{F}_{\mathrm{NN}}$ are expressed as
$$
\begin{aligned}
L &= 21(\lfloor\beta_r\rfloor+1)^2S_1\lceil\log_2(8S_1)\rceil+2d+3, \\
M &= 38(\lfloor\beta_r\rfloor+1)^2d^{\lfloor\beta_r\rfloor+1}S_2\lceil\log_2(8S_2)\rceil,
\end{aligned}
$$
for any $S_1, S_2\in\mathbb{N}_+$. Let $\bar{\delta}=\bar{\delta}_n=(\log n)^{1+\kappa}$, with an arbitrarily fixed $\kappa\in (0, 1]$, and $\underline{\delta}=0$. Then, for sufficiently large $n$, it follows that
$$
\begin{aligned}
& \inf_{f\in\mathcal{F}_{\mathrm{NN}}}\mathbb{E}_P[f(X)-r_0(X)]^2 \\
\le & c^*\left\{\left[(\lfloor\beta_r\rfloor+1)^2d^{\lfloor\beta_r\rfloor+(\beta_r\vee 1)/2}(S_1S_2)^{-2\beta_r/d}(\log n)^m\right]^2+\frac{(\log n)^{2+2\kappa}}{n}\right\},
\end{aligned}
$$
where $c^*$ is a constant not depending on $S_1, S_2$ and $n$.
\end{lemma}

\begin{proof}[Proof of Lemma \ref{lem: dre_approximation_error_bound_ls}]
For any $\iota_n>0$, observe that
$$
\begin{aligned}
& \mathbb{E}_P[f(X)-r_0(X)]^2\\
=& \mathbb{E}_P\left\{[f(X)-r_0(X)]^2\mathds{1}(\|X\|_{\infty}\le \iota_n)\right\}+\mathbb{E}_P\left\{[f(X)-r_0(X)]^2\mathds{1}(\|X\|_{\infty}>\iota_n)\right\}.
\end{aligned}
$$
On the one hand, it follows that
$$
\begin{aligned}
& \mathbb{E}_P\left\{[f(X)-r_0(X)]^2\mathds{1}(\|X\|_{\infty}>\iota_n)\right\} \\
\le & 2\mathbb{E}_P\left\{f(X)^2\mathds{1}(\|X\|_{\infty}>\iota_n)\right\}+2\mathbb{E}_P\left\{r_0(X)^2\mathds{1}(\|X\|_{\infty}>\iota_n)\right\} \\
\le & 2\bar{\delta}_n^2\mathbb{E}_P[\exp(\varsigma\|X\|_{\infty}/2)]\exp(-\varsigma\iota_n/2) \\
& +\frac{32}{\varsigma^2}\mathbb{E}_P[\exp(\varsigma r_0(X)/2)\exp(\varsigma\|X\|_{\infty}/2)]\exp(-\varsigma\iota_n/2) \\
\le & 2\bar{\delta}_n^2\mathbb{E}_P[\exp(\varsigma\|X\|_{\infty}/2)]\exp(-\varsigma\iota_n/2) \\
& +\frac{32}{\varsigma^2}\left\{\mathbb{E}_P[\exp(\varsigma r_0(X))]\mathbb{E}_P[\exp(\varsigma\|X\|_{\infty})]\right\}^{1/2}\exp(-\varsigma\iota_n/2) \\
\le & c_1(\bar{\delta}_n^2+1)\exp(-\varsigma\iota_n/2),
\end{aligned}
$$
where $c_1$ is a constant which depends only on $\varsigma, \mathbb{E}_P[\exp(\varsigma\|X\|_{\infty})]$ and $\mathbb{E}_P[\exp(\varsigma r_0(X))]$. On the other hand, we first notice that
$$
\begin{aligned}
& \mathbb{E}_P\left\{[f(X)-r_0(X)]^2\mathds{1}(\|X\|_{\infty}\le \iota_n)\right\} \\
=& \mathbb{E}_P\left\{[f(X)-r_0(X)]^2\mathds{1}(r_0(X)\le \iota_n)\mathds{1}(\|X\|_{\infty}\le \iota_n)\right\} \\
& +\mathbb{E}_P\left\{[f(X)-r_0(X)]^2\mathds{1}(r_0(X)>\iota_n)\mathds{1}(\|X\|_{\infty}\le \iota_n)\right\} \\
\le & \mathbb{E}_P\left\{[f(X)-r_0(X)]^2\mathds{1}(r_0(X)\le \iota_n)\mathds{1}(\|X\|_{\infty}\le \iota_n)\right\} \\
& +2\mathbb{E}_P\left[f(X)^2\mathds{1}(r_0(X)>\iota_n)\right]+2\mathbb{E}_P\left[r_0(X)^2\mathds{1}(r_0(X)>\iota_n)\right] \\
\le & \mathbb{E}_P\left\{[f(X)-r_0(X)]^2\mathds{1}(r_0(X)\le \iota_n)\mathds{1}(\|X\|_{\infty}\le \iota_n)\right\} \\
& +2\bar{\delta}_n^2\mathbb{E}_P[\exp(\varsigma r_0(X)/2)]\exp(-\varsigma\iota_n/2)+\frac{32}{\varsigma^2}\mathbb{E}_P[\exp(\varsigma r_0(X))]\exp(-\varsigma\iota_n/2) \\
\le & \mathbb{E}_P\left\{[f(X)-r_0(X)]^2\mathds{1}(r_0(X)\le \iota_n)\mathds{1}(\|X\|_{\infty}\le \iota_n)\right\}+c_2(\bar{\delta}_n^2+1)\exp(-\varsigma\iota_n/2),
\end{aligned}
$$
where $c_2$ is a constant which depends only on $\varsigma$ and $\mathbb{E}_P[\exp(\varsigma r_0(X))]$. Then, we focus on $\{x: \|x\|_{\infty}\le \iota_n\}=[-\iota_n, \iota_n]^d$. Let $r_0^*(x)=r_0(2\iota_nx-\iota_n\mathrm{1}_d)$ for $x\in [0, 1]^d$. Lemma \ref{lem: thm_3.3_jiao} demonstrates that for any $S_1, S_2\in\mathbb{N}_+$, there exists a function $f^*$ implemented by a ReLU network with depth $L^*=21(\lfloor\beta_r\rfloor+1)^2S_1\lceil\log_2(8S_1)\rceil+2d$, width $M^*=38(\lfloor\beta_r\rfloor+1)^2d^{\lfloor\beta_r\rfloor+1}S_2\lceil\log_2(8S_2)\rceil$, such that
$$
|f^*(x)-r_0^*(x)|\le 18c(\iota_n^m+1)(\lfloor\beta_r\rfloor+1)^2d^{\lfloor\beta_r\rfloor+(\beta_r\vee 1)/2}(S_1S_2)^{-2\beta_r/d},
$$
for all $x\in [0, 1]^d\backslash \Omega([0, 1]^d, K, \Delta)$. Here,
$$
\Omega([0, 1]^d, K, \Delta)=\bigcup_{i=1}^d\left\{x=(x_1, \dots, x_d)^{\top}: x_i\in\bigcup_{k=1}^{K-1}(k/K-\Delta, k/K)\right\},
$$
where $K=\lceil(S_1S_2)^{2/d}\rceil$ and $\Delta$ is an arbitrary scalar in $(0, 1/(3K)]$. Let $f^{\dagger}(x)=f^*((x+\iota_n\mathrm{1}_d)/(2\iota_n))$ for $x\in[-\iota_n, \iota_n]^d$. We obtain that
$$
|f^{\dagger}(x)-r_0(x)|\le 18c(\iota_n^m+1)(\lfloor\beta_r\rfloor+1)^2d^{\lfloor\beta_r\rfloor+(\beta_r\vee 1)/2}(S_1S_2)^{-2\beta_r/d},
$$
for all $x\in [-\iota_n, \iota_n]^d\backslash \Omega^{\dagger}$, where $\Omega^{\dagger}=\{x: (x+\iota_n\mathrm{1}_d)/(2\iota_n)\in \Omega([0, 1]^d, K, \Delta)\}$. Furthermore, note that
$$
f^{\dagger}(x)=f^*\left(\frac{x+\iota_n\mathrm{1}_d}{2\iota_n}\right)=f^*\left(\mathrm{relu}\left(\frac{x+\iota_n\mathrm{1}_d}{2\iota_n}\right)-\mathrm{relu}\left(-\frac{x+\iota_n\mathrm{1}_d}{2\iota_n}\right)\right),
$$
which is implemented by a neural network with ReLU activations, depth $L^{\dagger}=L^*+1$, and width $M^{\dagger}=M^*$. In addition, let
$$
f^{\ddagger}(x)=
\begin{cases}
    \bar{\delta}_n, & f^{\dagger}(x)>\bar{\delta}_n, \\
    f^{\dagger}(x), & 0\le f^{\dagger}(x)\le \bar{\delta}_n, \\
    0, & f^{\dagger}(x)<0.
\end{cases}
$$
A straightforward calculation shows that
$$
f^{\ddagger}(x)=\mathrm{relu}(-\mathrm{relu}(-f^{\dagger}(x)+\bar{\delta}_n)+\bar{\delta}_n),
$$
indicating that $f^{\ddagger}(x)$ can be implemented by a ReLU network with depth $L=L^*+3$ and width $M=M^*$. Due to the arbitrariness of $\Delta$, when $\bar{\delta}_n\ge \iota_n$, it follows that
$$
\begin{aligned}
& \inf_{f\in \mathcal{F}_{\mathrm{NN}}}\mathbb{E}_P\left\{[f(X)-r_0(X)]^2\mathds{1}(r_0(X)\le \iota_n)\mathds{1}(\|X\|_{\infty}\le \iota_n)\right\} \\
\le & \mathbb{E}_P\left\{[f^{\ddagger}(X)-r_0(X)]^2\mathds{1}(r_0(X)\le \iota_n)\mathds{1}(\|X\|_{\infty}\le \iota_n)\right\} \\
\le & \left[18c(\iota_n^m+1)(\lfloor\beta_r\rfloor+1)^2d^{\lfloor\beta_r\rfloor+(\beta_r\vee 1)/2}(S_1S_2)^{-2\beta_r/d}\right]^2.
\end{aligned}
$$
Recall that $\bar{\delta}_n=(\log n)^{1+\kappa}$, and set $\iota_n$ to $(2\varsigma^{-1}\log n)\vee 1$. We conclude that for sufficiently large $n$ satisfying $\bar{\delta}_n\ge \iota_n$ and $\log n\ge \varsigma/2$, it holds that
$$
\begin{aligned}
& \inf_{f\in\mathcal{F}_{\mathrm{NN}}}\mathbb{E}_P[f(X)-r_0(X)]^2 \\
\le & \inf_{f\in \mathcal{F}_{\mathrm{NN}}}\mathbb{E}_P\left\{[f(X)-r_0(X)]^2\mathds{1}(r_0(X)\le \iota_n)\mathds{1}(\|X\|_{\infty}\le \iota_n)\right\} \\
& +(c_1+c_2)(\bar{\delta}_n^2+1)\exp(-\varsigma\iota_n/2) \\
\le & c_3\left\{\left[(\lfloor\beta_r\rfloor+1)^2d^{\lfloor\beta_r\rfloor+(\beta_r\vee 1)/2}(S_1S_2)^{-2\beta_r/d}(\log n)^m\right]^2+\frac{(\log n)^{2+2\kappa}}{n}\right\},
\end{aligned}
$$
where $c_3$ is a constant not depending on $S_1, S_2$ and $n$.
\end{proof}

\begin{proof}[Proof of Theorem 3.1]
To commence, we notice that Lemma \ref{lem: dre_error_decomposition_ls}, Lemma \ref{lem: dre_stochatic_error_bound_ls} and Lemma \ref{lem: dre_approximation_error_bound_ls} indicate
$$
\begin{aligned}
\mathcal{R}^s(\hat{r}_{\mathrm{LS}}) \le & \frac{c_1SL\log S(\log n)^{5+4\kappa}}{n} \\
& +c_2\left\{\left[(\lfloor\beta_r\rfloor+1)^2d^{\lfloor\beta_r\rfloor+(\beta_r\vee 1)/2}(S_1S_2)^{-2\beta_r/d}(\log n)^m\right]^2+\frac{(\log n)^{2+2\kappa}}{n}\right\},
\end{aligned}
$$
where $c_1, c_2$ are constants not depending on $S, L, S_1, S_2$ and $n$, and $S_1, S_2$ satisfy the conditions that the network depth $L=21(\lfloor\beta_r\rfloor+1)^2S_1\lceil\log_2(8S_1)\rceil+2d+3$, network width $M=38(\lfloor\beta_r\rfloor+1)^2d^{\lfloor\beta_r\rfloor+1}S_2\lceil\log_2(8S_2)\rceil$, for sufficiently large $n$ and $n\ge \mathrm{Pdim}(\mathcal{F}_{\mathrm{NN}})$. Therefore, by letting $S_1=\mathcal{O}(n^{d/(2d+4\beta_r)})$ and $S_2=\mathcal{O}(1)$,  we obtain
$$
M=\mathcal{O}(1), \quad L=\mathcal{O}\left(n^{\frac{d}{2d+4\beta_r}}\log n\right), \quad S=\mathcal{O}(M^2L)=\mathcal{O}\left(n^{\frac{d}{2d+4\beta_r}}\log n\right),
$$
yielding
$$
\mathcal{R}^s(\hat{r}_{\mathrm{LS}})\le c_3n^{-\frac{2\beta_r}{d+2\beta_r}}(\log n)^{(8+4\kappa)\vee (2m)},
$$
where $c_3$ is a constant not depending on $n$, for $n\ge 2$. Furthermore, observe that
$$
\begin{aligned}
\mathcal{R}^t(\hat{r}_{\mathrm{LS}}) &= \mathbb{E}\left[\hat{r}_{\mathrm{LS}}(X^t)-r_0(X^t)\right]^2 \\
&= \mathbb{E}\left\{\left[\hat{r}_{\mathrm{LS}}(X^s)-r_0(X^s)\right]^2r_0(X^s)\right\}
\end{aligned}
$$
Similar truncation operation suggests that, for any $\iota_n>0$, we have
$$
\begin{aligned}
\mathcal{R}^t(\hat{r}_{\mathrm{LS}}) =& \mathbb{E}\left\{\left[\hat{r}_{\mathrm{LS}}(X^s)-r_0(X^s)\right]^2r_0(X^s)\right\} \\
=& \mathbb{E}\left\{\left[\hat{r}_{\mathrm{LS}}(X^s)-r_0(X^s)\right]^2r_0(X^s)\mathds{1}(r_0(X^s)\le \iota_n)\right\} \\
& +\mathbb{E}\left\{\left[\hat{r}_{\mathrm{LS}}(X^s)-r_0(X^s)\right]^2r_0(X^s)\mathds{1}(r_0(X^s)> \iota_n)\right\} \\
\le & \iota_n\mathbb{E}\left\{\left[\hat{r}_{\mathrm{LS}}(X^s)-r_0(X^s)\right]^2\right\} \\
& +\mathbb{E}\left\{\left[\hat{r}_{\mathrm{LS}}(X^s)-r_0(X^s)\right]^2r_0(X^s)\mathds{1}(r_0(X^s)> \iota_n)\right\} \\
=& \iota_n\mathcal{R}^s(\hat{r}_{\mathrm{LS}})+\mathbb{E}\left\{\left[\hat{r}_{\mathrm{LS}}(X^s)-r_0(X^s)\right]^2r_0(X^s)\mathds{1}(r_0(X^s)> \iota_n)\right\}.
\end{aligned}
$$
Subsequently, note that
$$
\begin{aligned}
& \mathbb{E}\left\{\left[\hat{r}_{\mathrm{LS}}(X^s)-r_0(X^s)\right]^2r_0(X^s)\mathds{1}(r_0(X^s)> \iota_n)\right\} \\
\le & 2\mathbb{E}\left[\hat{r}_{\mathrm{LS}}(X^s)^2r_0(X^s)\mathds{1}(r_0(X^s)> \iota_n)\right]+2\mathbb{E}\left[r_0(X^s)^3\mathds{1}(r_0(X^s)> \iota_n)\right] \\
\le & 2\bar{\delta}_n^2\mathbb{E}\left[r_0(X^s)\mathds{1}(r_0(X^s)> \iota_n)\right]+2\mathbb{E}\left[r_0(X^s)^3\mathds{1}(r_0(X^s)> \iota_n)\right] \\
\le & 4\varsigma^{-1}\bar{\delta}_n^2\mathbb{E}[\exp(\varsigma r_0(X^s))]\exp(-\varsigma\iota_n/2)+432\varsigma^{-3}\mathbb{E}[\exp(\varsigma r_0(X^s))]\exp(-\varsigma\iota_n/2) \\
\le & c_4(\bar{\delta}_n^2+1)\exp(-\varsigma\iota_n/2) \\
=& c_4[(\log n)^{2+2\kappa}+1]\exp(-\varsigma\iota_n/2),
\end{aligned}
$$
where $c_4$ is a constant depending only on $\varsigma$ and $\mathbb{E}[\exp(\varsigma r_0(X^s))]$. Hence, by taking $\iota_n=(2\varsigma^{-1}\log n)\vee 1$, we obtain that
$$
\mathcal{R}^t(\hat{r}_{\mathrm{LS}})\le [(2\varsigma^{-1}\log n)\vee 1]\mathcal{R}^s(\hat{r}_{\mathrm{LS}})+\frac{c_4[(\log n)^{2+2\kappa}+1]}{n}.
$$
This completes the proof.
\end{proof}

\subsection{Proof of Lemma 3.2}\label{subsec: proof_of_lem_dre_validity_lr_loss}

\begin{proof}[Proof of Lemma 3.2]

Given $r_0\in\mathcal{L}_{\mathrm{LR}}(X^s)$, we first claim that for any function $f\in \mathcal{L}_{\mathrm{LR}}(X^s)$, we have $\mathbb{E}D_{\mathrm{LR}}(r_0(X^s)\| f(X^s))<\infty$. In fact, it follows that
$$
\begin{aligned}
0 \le & \mathbb{E}D_{\mathrm{LR}}(r_0(X^s)\| f(X^s)) \\
=& \mathbb{E}_P\left[D_{\mathrm{LR}}(r_0(X)\| f(X))\mathds{1}(X\in\mathcal{X}^t)\right] \\
=& \mathbb{E}_P\Big(\big\{r_0(X)\log r_0(X)-[r_0(X)+1]\log(r_0(X)+1)+\log(f(X)+1) \\
& -r_0(X)\log f(X)+r_0(X)\log(f(X)+1)\big\}\mathds{1}(X\in\mathcal{X}^t)\Big) \\
\le & \mathbb{E}_P\Big(\big\{\left|r_0(X)\log r_0(X)\right|+[r_0(X)+1]\log(r_0(X)+1)+\log(f(X)+1) \\
& +\left|r_0(X)\log f(X)\right|+r_0(X)\log(f(X)+1)\big\}\mathds{1}(X\in\mathcal{X}^t)\Big).
\end{aligned}
$$
Recall that $\mathbb{E}_Ph(X)\equiv \mathbb{E}h(X^s)$ for any $X^s$-integrable function $h$. Then, observe that for any scalars $x, y\in\mathbb{R}_+$, we have $\max(\log x, \log(x+1))\le x$ and $|x\log y|\le xy^{-1}+xy$. Hence, it holds that
$$
\begin{aligned}
0 \le & \mathbb{E}D_{\mathrm{LR}}(r_0(X^s)\| f(X^s)) \\
\le & \mathbb{E}_P\Big(\big\{\left|r_0(X)\log r_0(X)\right|+[r_0(X)+1]\log(r_0(X)+1)+\log(f(X)+1) \\
& +\left|r_0(X)\log f(X)\right|+r_0(X)\log(f(X)+1)\big\}\mathds{1}(X\in\mathcal{X}^t)\Big) \\
\le & \mathbb{E}_P\Big(\big\{1+r_0(X)^2+[r_0(X)+1]r_0(X)+f(X) \\
& +r_0(X)f(X)^{-1}+r_0(X)f(X)\big\}\mathds{1}(X\in\mathcal{X}^t)\Big) \\
=& \mathbb{E}_P\big\{1+r_0(X)^2+[r_0(X)+1]r_0(X)+f(X) \\
& +r_0(X)f(X)^{-1}+r_0(X)f(X)\big\}<\infty.
\end{aligned}
$$
Next, note that $\mathbb{E}D_{\mathrm{LR}}(r_0(X^s)\| r_0(X^s))=0$. Therefore, for any minimizer $f^*$ of $\mathbb{E}D_{\mathrm{LR}}(r_0(X^s)\| f(X^s))$ with respect to $f\in\mathcal{L}_{\mathrm{LR}}(X^s)$, we have $\mathbb{E}D_{\mathrm{LR}}(r_0(X^s)\| f^*(X^s))=0$, which indicates that
$$
D_{\mathrm{LR}}(r_0(X^s)\| f^*(X^s))=0 \text{ a.s. } X^s.
$$
Let the set $\mathcal{A}=\{x\in\mathcal{X}^t: D_{\mathrm{LR}}(r_0(x)\| f^*(x))=0\}$. Notice that $\varphi_{\mathrm{LR}}''(x)=[x(x+1)]^{-1}>0$ for any $x\in \mathbb{R}_+$, thus $\varphi_{\mathrm{LR}}$ is strictly convex over $\mathbb{R}_+$. By Lemma \ref{lem: bregman_divergence_solution}, we obtain that
$$
1=\mathrm{P}(X^s\in\mathcal{A})\le \mathrm{P}(r_0(X^s)=f^*(X^s)).
$$
Consequently, it follows that $r_0(X^s)=f^*(X^s)$ a.s. $X^s$.
\end{proof}

\subsection{Proof of Theorem 3.3} \label{subsec: proof_of_thm_dre_convergence_rate_lr}

For any function $f\in\mathcal{L}_{\mathrm{LR}}(X^s)$, define
$$
\begin{aligned}
J(f) &= \mathbb{E}_P\left[\log(f(X)+1)-r_0(X)\log f(X)+r_0(X)\log(f(X)+1)\right], \\
J_n(f) &= \frac 1n\sum_{i=1}^n\log(f(X^s_i)+1)+\frac 1n\sum_{i=1}^n\left[-\log f(X^t_i)+\log(f(X^t_i)+1)\right].
\end{aligned}
$$
Here, $\mathbb{E}_P[h(X)]\equiv \mathbb{E}[h(X^s)]$ for any $X^s$-integrable function $h$, where the expectation is taken with respect to $X^s$. Analogously, $\mathbb{E}_Q[h(X)]\equiv \mathbb{E}[h(X^t)]$ for any $X^t$-integrable function $h$.

\begin{lemma}\label{lem: dre_error_decomposition_lr}
Assume that $r_0(X^s)$ and $r_0(X^s)^{-1}\mathds{1}(X^s\in\mathcal{X}^t)$ are square-integrable. Then, it follows that
$$
\begin{aligned}
& \mathbb{E}D_{\mathrm{LR}}(r_0(X^s)\| \hat{r}_{\mathrm{LR}}(X^s)) \\
\le & \mathbb{E}[J(\hat{r}_{\mathrm{LR}})-2J_n(\hat{r}_{\mathrm{LR}})+J(r_0)]+2\inf_{f\in\mathcal{F}_{\mathrm{NN}}}\mathbb{E}_PD_{\mathrm{LR}}(r_0(X)\| f(X)).
\end{aligned}
$$
\end{lemma}

\begin{proof}[Proof of Lemma \ref{lem: dre_error_decomposition_lr}]
Given that $r_0(X^s)$ and $r_0(X^s)^{-1}\mathds{1}(X^s\in\mathcal{X}^t)$ are square-integrable, it is evident that $r_0\in \mathcal{L}_{\mathrm{LR}}(X^s)$. In addition, setting $\underline{\delta}$ to arbitrary positive scalar implies $\mathcal{F}_{\mathrm{NN}}\subset\mathcal{L}_{\mathrm{LR}}(X^s)$. Subsequently, for any $f\in\mathcal{F}_{\mathrm{NN}}$, we have
$$
\begin{aligned}
\mathbb{E}D_{\mathrm{LR}}(r_0(X^s)\| \hat{r}_{\mathrm{LR}}(X^s)) &= \mathbb{E}[J(\hat{r}_{\mathrm{LR}})-J(r_0)] \\
&\le \mathbb{E}[J(\hat{r}_{\mathrm{LR}})-J(r_0)]+2\mathbb{E}[J_n(f)-J_n(\hat{r}_{\mathrm{LR}})] \\
&= \mathbb{E}[J(\hat{r}_{\mathrm{LR}})-J(r_0)]+2\mathbb{E}[J_n(f)-J_n(r_0)+J_n(r_0)-J_n(\hat{r}_{\mathrm{LR}})] \\
&= \mathbb{E}[J(\hat{r}_{\mathrm{LR}})-2J_n(\hat{r}_{\mathrm{LR}})+J(r_0)]+2[J(f)-J(r_0)] \\
&= \mathbb{E}[J(\hat{r}_{\mathrm{LR}})-2J_n(\hat{r}_{\mathrm{LR}})+J(r_0)]+2\mathbb{E}_PD_{\mathrm{LR}}(r_0(X)\| f(X)).
\end{aligned}
$$
Take the infimum on both sides with respect to $f\in\mathcal{F}_{\mathrm{NN}}$ and we complete the proof.
\end{proof}

\begin{lemma}\label{lem: dre_stochatic_error_bound_lr}
Assume that $r_0(X^s)$ and $r_0(X^s)^{-1}\mathds{1}(X^s\in\mathcal{X}^t)$ are sub-exponentially distributed random variables. Let $\bar{\delta}=\bar{\delta}_n=(\log n)^{1+\kappa}$ and $\underline{\delta}=\underline{\delta}_n=(\log n)^{-1-\kappa}$ for arbitrarily fixed constant $\kappa\in (0, 1]$. Then, for sufficiently large $n$ and $n\ge \mathrm{Pdim}(\mathcal{F}_{\mathrm{NN}})$, it follows that
$$
\mathbb{E}[J(\hat{r}_{\mathrm{LR}})-2J_n(\hat{r}_{\mathrm{LR}})+J(r_0)]\le \frac{c^*SL\log S(\log n)^{6+5\kappa}}{n},
$$
where $c^*$ is a constant not depending on $S, L$ and $n$.
\end{lemma}

\begin{proof}[Proof of Lemma \ref{lem: dre_stochatic_error_bound_lr}]
Let $T_i=(X^s_i, X^t_i)$ for $i=1, \dots, n$, $\mathcal{D}_n=\{T_1, \dots, T_n\}$, and $T=(X^s, X^t)$ be an independent copy of $T_1$. Firstly, we have
$$
\begin{aligned}
\mathbb{E}[J(\hat{r}_{\mathrm{LR}})-2J_n(\hat{r}_{\mathrm{LR}})+J(r_0)] &= \mathbb{E}_{\mathcal{D}_n}[J(\hat{r}_{\mathrm{LR}})-2J_n(\hat{r}_{\mathrm{LR}})+J(r_0)] \\
&= \mathbb{E}_{\mathcal{D}_n}\{J(\hat{r}_{\mathrm{LR}})-J(r_0)-2[J_n(\hat{r}_{\mathrm{LR}})-J_n(r_0)]\} \\
&= \mathbb{E}_{\mathcal{D}_n}\left\{\mathbb{E}_{T}[g(\hat{r}_{\mathrm{LR}}, T)]-\frac 2n\sum_{i=1}^ng(\hat{r}_{\mathrm{LR}}, T_i)\right\},
\end{aligned}
$$
where
$$
\begin{aligned}
g(f, T) =& \log(f(X^s)+1)-\log(r_0(X^s)+1) \\
& +\log r_0(X^t)-\log f(X^t) \\
& +\log(f(X^t)+1)-\log(r_0(X^t)+1),
\end{aligned}
$$
for $f\in\mathcal{F}_{\mathrm{NN}}$. Next, we introduce a truncation step. For any $\iota_n\ge 1$ and any $f\in\mathcal{F}_{\mathrm{NN}}$, define
$$
\begin{aligned}
g_{\iota_n}(f, T) =& [\log(f(X^s)+1)-\log(r_0(X^s)+1)]\mathds{1}(\iota_n^{-1}\le r_0(X^s)\le \iota_n) \\
& +[\log r_0(X^t)-\log f(X^t)]\mathds{1}(\iota_n^{-1}\le r_0(X^t)\le \iota_n) \\
& +[\log(f(X^t)+1)-\log(r_0(X^t)+1)]\mathds{1}(\iota_n^{-1}\le r_0(X^t)\le \iota_n).
\end{aligned}
$$
Then, it follows that
$$
\begin{aligned}
|g(f, T)-g_{\iota_n}(f, T)| \le & |\log(f(X^s)+1)-\log(r_0(X^s)+1)|\mathds{1}(r_0(X^s)>\iota_n) \\
& +|\log(f(X^s)+1)-\log(r_0(X^s)+1)|\mathds{1}(r_0(X^s)<\iota_n^{-1}) \\
& +|\log r_0(X^t)-\log f(X^t)|\mathds{1}(r_0(X^t)>\iota_n) \\
& +|\log r_0(X^t)-\log f(X^t)|\mathds{1}(r_0(X^t)<\iota_n^{-1}) \\
& +|\log(f(X^t)+1)-\log(r_0(X^t)+1)|\mathds{1}(r_0(X^t)>\iota_n) \\
& +|\log(f(X^t)+1)-\log(r_0(X^t)+1)|\mathds{1}(r_0(X^t)<\iota_n^{-1}) \\
\le & [\log(\bar{\delta}_n+1)+\log(r_0(X^s)+1)]\mathds{1}(r_0(X^s)>\iota_n) \\
& +[\log(\bar{\delta}_n+1)+1]\mathds{1}(r_0(X^s)<\iota_n^{-1}) \\
& +[\log r_0(X^t)+\log \bar{\delta}_n]\mathds{1}(r_0(X^t)>\iota_n) \\
& +[-\log r_0(X^t)+\log \bar{\delta}_n]\mathds{1}(r_0(X^t)<\iota_n^{-1}) \\
& +[\log(\bar{\delta}_n+1)+\log(r_0(X^t)+1)]\mathds{1}(r_0(X^t)>\iota_n) \\
& +[\log(\bar{\delta}_n+1)+1]\mathds{1}(r_0(X^t)<\iota_n^{-1}).
\end{aligned}
$$
Taking expectation, we have
$$
\begin{aligned}
& \mathbb{E}[J(\hat{r}_{\mathrm{LR}})-2J_n(\hat{r}_{\mathrm{LR}})+J(r_0)] \\
=& \mathbb{E}_{\mathcal{D}_n}\left\{\mathbb{E}_{T}[g(\hat{r}_{\mathrm{LR}}, T)]-\frac 2n\sum_{i=1}^ng(\hat{r}_{\mathrm{LR}}, T_i)\right\} \\
\le & \mathbb{E}_{\mathcal{D}_n}\left\{\mathbb{E}_{T}[g_{\iota_n}(\hat{r}_{\mathrm{LR}}, T)]-\frac 2n\sum_{i=1}^ng_{\iota_n}(\hat{r}_{\mathrm{LR}}, T_i)\right\} \\
& +3\mathbb{E}_P\left\{[\log(\bar{\delta}_n+1)+\log(r_0(X)+1)]\mathds{1}(r_0(X)>\iota_n)\right\} \\
& +3\mathbb{E}_P\left\{[\log(\bar{\delta}_n+1)+1]\mathds{1}(r_0(X)<\iota_n^{-1})\right\} \\
& +3\mathbb{E}_Q\left\{[\log r_0(X)+\log \bar{\delta}_n]\mathds{1}(r_0(X)>\iota_n)\right\} \\
& +3\mathbb{E}_Q\left\{[-\log r_0(X)+\log \bar{\delta}_n]\mathds{1}(r_0(X)<\iota_n^{-1})\right\} \\
& +3\mathbb{E}_Q\left\{[\log(\bar{\delta}_n+1)+\log(r_0(X)+1)]\mathds{1}(r_0(X)>\iota_n)\right\} \\
& +3\mathbb{E}_Q\left\{[\log(\bar{\delta}_n+1)+1]\mathds{1}(r_0(X)<\iota_n^{-1})\right\}.
\end{aligned}
$$
Specifically, for $n\ge 3$, observe that firstly,
$$
\begin{aligned}
& \mathbb{E}_P\left\{[\log(\bar{\delta}_n+1)+\log(r_0(X)+1)]\mathds{1}(r_0(X)>\iota_n)\right\} \\
\le & \mathbb{E}_P\left\{[\log(\bar{\delta}_n+1)+r_0(X)]\mathds{1}(r_0(X)>\iota_n)\right\} \\
\le & \log(\bar{\delta}_n+1)\mathbb{E}_P[\exp(\varsigma r_0(X)/2)]\exp(-\varsigma\iota_n/2)+\mathbb{E}_P\left\{r_0(X)\mathds{1}(r_0(X)>\iota_n)\right\} \\
\le & \log(\bar{\delta}_n+1)\mathbb{E}_P[\exp(\varsigma r_0(X)/2)]\exp(-\varsigma\iota_n/2)+\frac{2}{\varsigma}\mathbb{E}_P\left\{\exp(\varsigma r_0(X))\right\}\exp(-\varsigma\iota_n/2),
\end{aligned}
$$
secondly,
$$
\begin{aligned}
& \mathbb{E}_P\left\{[\log(\bar{\delta}_n+1)+1]\mathds{1}(r_0(X)<\iota_n^{-1})\right\} \\
=& [\log(\bar{\delta}_n+1)+1]\mathbb{E}_P\left[\mathds{1}(X\in\mathcal{X}^t)\mathds{1}(r_0(X)<\iota_n^{-1})\right] \\
=& [\log(\bar{\delta}_n+1)+1]\mathbb{E}_P\left[\mathds{1}(X\in\mathcal{X}^t)\mathds{1}(r_0(X)^{-1}\mathds{1}(X\in\mathcal{X}^t)>\iota_n)\right] \\
\le & [\log(\bar{\delta}_n+1)+1]\mathbb{E}[\exp(\varsigma r_0(X)^{-1}\mathds{1}(X\in\mathcal{X}^t)/2)]\exp(-\varsigma\iota_n/2),
\end{aligned}
$$
thirdly,
$$
\begin{aligned}
& \mathbb{E}_Q\left\{[\log r_0(X)+\log \bar{\delta}_n]\mathds{1}(r_0(X)>\iota_n)\right\} \\
\le & \mathbb{E}_Q\left\{[r_0(X)+\log \bar{\delta}_n]\mathds{1}(r_0(X)>\iota_n)\right\} \\
=& \mathbb{E}_P\left\{[r_0(X)+\log \bar{\delta}_n]r_0(X)\mathds{1}(r_0(X)>\iota_n)\right\} \\
\le & \frac{2}{\varsigma}\log(\bar{\delta}_n)\mathbb{E}_P[\exp(\varsigma r_0(X))]\exp(-\varsigma\iota_n/2)+\frac{16}{\varsigma^2}\mathbb{E}_P\left[\exp(\varsigma r_0(X))\right]\exp(-\varsigma\iota_n/2),
\end{aligned}
$$
fourthly,
$$
\begin{aligned}
& \mathbb{E}_Q\left\{[-\log r_0(X)+\log \bar{\delta}_n]\mathds{1}(r_0(X)<\iota_n^{-1})\right\} \\
=& \mathbb{E}_P\left\{[-\log r_0(X)+\log \bar{\delta}_n]r_0(X)\mathds{1}(X\in\mathcal{X}^t)\mathds{1}(r_0(X)<\iota_n^{-1})\right\} \\
\le & (1+\log \bar{\delta}_n)\mathbb{E}_P\left[\mathds{1}(X\in\mathcal{X}^t)\mathds{1}(r_0(X)<\iota_n^{-1})\right] \\
= & (1+\log \bar{\delta}_n)\mathbb{E}_P\left[\mathds{1}(X\in\mathcal{X}^t)\mathds{1}(r_0(X)^{-1}\mathds{1}(X\in\mathcal{X}^t)>\iota_n)\right] \\
\le & (1+\log \bar{\delta}_n)\mathbb{E}_P[\exp(\varsigma r_0(X)^{-1}\mathds{1}(X\in\mathcal{X}^t)/2)]\exp(-\varsigma\iota_n/2),
\end{aligned}
$$
fifthly,
$$
\begin{aligned}
& \mathbb{E}_Q\left\{[\log(\bar{\delta}_n+1)+\log(r_0(X)+1)]\mathds{1}(r_0(X)>\iota_n)\right\} \\
=& \mathbb{E}_P\left\{[\log(\bar{\delta}_n+1)+\log(r_0(X)+1)]r_0(X)\mathds{1}(r_0(X)>\iota_n)\right\} \\
\le & \mathbb{E}_P\left\{[\log(\bar{\delta}_n+1)+r_0(X)]r_0(X)\mathds{1}(r_0(X)>\iota_n)\right\} \\
\le & \frac{2}{\varsigma}\log(\bar{\delta}_n+1)\mathbb{E}_P[\exp(\varsigma r_0(X))]\exp(-\varsigma\iota_n/2)+\frac{16}{\varsigma^2}\mathbb{E}_P[\exp(\varsigma r_0(X))]\exp(-\varsigma\iota_n/2),
\end{aligned}
$$
and lastly,
$$
\begin{aligned}
& \mathbb{E}_Q\left\{[\log(\bar{\delta}_n+1)+1]\mathds{1}(r_0(X)<\iota_n^{-1})\right\} \\
=& [\log(\bar{\delta}_n+1)+1]\mathbb{E}_P\left[r_0(X)\mathds{1}(r_0(X)<\iota_n^{-1})\right] \\
\le & [\log(\bar{\delta}_n+1)+1]\mathbb{E}_P\left[\mathds{1}(r_0(X)^{-1}\mathds{1}(X\in\mathcal{X}^t)>\iota_n)\right] \\
\le & [\log(\bar{\delta}_n+1)+1]\mathbb{E}[\exp(\varsigma r_0(X)^{-1}\mathds{1}(X\in\mathcal{X}^t)/2)]\exp(-\varsigma\iota_n/2).
\end{aligned}
$$
Here, we have applied the inequalities $a\le \exp(a)$ and $\mathds{1}(a>0)\le \exp(a)$ for $a\in\mathbb{R}$, as well as that $-a\log a\le 1$ for $a\in\mathbb{R}_+$. Consequently, we obtain
$$
\begin{aligned}
& \mathbb{E}[J(\hat{r}_{\mathrm{LR}})-2J_n(\hat{r}_{\mathrm{LR}})+J(r_0)] \\
\le & \mathbb{E}_{\mathcal{D}_n}\left\{\mathbb{E}_{T}[g_{\iota_n}(\hat{r}_{\mathrm{LR}}, T)]-\frac 2n\sum_{i=1}^ng_{\iota_n}(\hat{r}_{\mathrm{LR}}, T_i)\right\}+c_1(1+\log \bar{\delta}_n)\exp(-\varsigma\iota_n/2),
\end{aligned}
$$
where $c_1$ is a constant which depends only on $\varsigma, \mathbb{E}_P\left\{\exp(\varsigma r_0(X))\right\}$ and $\mathbb{E}_P[\exp(\varsigma r_0(X)^{-1}\mathds{1}(X\in\mathcal{X}^t))]$. Recall that $\bar{\delta}_n=(\log n)^{1+\kappa}$ with $\kappa\in (0, 1]$. Setting $\iota_n$ to $(2\varsigma^{-1}\log n)\vee 1$ yields
$$
\begin{aligned}
& \mathbb{E}[J(\hat{r}_{\mathrm{LR}})-2J_n(\hat{r}_{\mathrm{LR}})+J(r_0)] \\
\le & \mathbb{E}_{\mathcal{D}_n}\left\{\mathbb{E}_{T}[g_{\iota_n}(\hat{r}_{\mathrm{LR}}, T)]-\frac 2n\sum_{i=1}^ng_{\iota_n}(\hat{r}_{\mathrm{LR}}, T_i)\right\}+c_1(1+\log n)n^{-1}.
\end{aligned}
$$
Then, we proceed to verify the conditions in Theorem \ref{thm: gen_thm_11.4_gyorfi}. Notice that
$$
\begin{aligned}
\sup_{f\in\mathcal{F}_{\mathrm{NN}}, T\in\mathbb{R}^{2d}}|g_{\iota_n}(f, T)| &\le 3[\log(\bar{\delta}_n+1)+\log(\iota_n+1)] \\
&\le c_2(1+\log n),
\end{aligned}
$$
where $c_2$ is a constant which depends only on $\varsigma$ and $\kappa$. Furthermore, for any $f\in\mathcal{F}_{\mathrm{NN}}$,
$$
\mathbb{E}[g_{\iota_n}(f, T)]=\mathbb{E}_P[D_{\mathrm{LR}}(r_0(X)\| f(X))\mathds{1}(\iota_n^{-1}\le r_0(X)\le \iota_n)].
$$
The smoothness of $\varphi_{\mathrm{LR}}$ then implies that for sufficiently large $n$ such that $\bar{\delta}_n\ge \iota_n$, we have
$$
\mathbb{E}[g_{\iota_n}(f, T)]\ge \frac{1}{\bar{\delta}_n(\bar{\delta}_n+1)}\mathbb{E}_P\left\{[r_0(X)-f(X)]^2\mathds{1}(\iota_n^{-1}\le r_0(X)\le \iota_n)\right\}.
$$
Hence, it follows that
$$
\begin{aligned}
& \mathbb{E}[g_{\iota_n}(f, T)^2] \\
\le & 3\mathbb{E}_P\left\{[\log(f(X)+1)-\log(r_0(X)+1)]^2\mathds{1}(\iota_n^{-1}\le r_0(X)\le \iota_n)\right\} \\
& +3\mathbb{E}_Q\left\{[\log r_0(X)-\log f(X)]^2\mathds{1}(\iota_n^{-1}\le r_0(X)\le \iota_n)\right\} \\
& +3\mathbb{E}_Q\left\{[\log(f(X)+1)-\log(r_0(X)+1)]^2\mathds{1}(\iota_n^{-1}\le r_0(X)\le \iota_n)\right\} \\
=& 3\mathbb{E}_P\left\{[\log(f(X)+1)-\log(r_0(X)+1)]^2\mathds{1}(\iota_n^{-1}\le r_0(X)\le \iota_n)\right\} \\
& +3\mathbb{E}_P\left\{[\log r_0(X)-\log f(X)]^2r_0(X)\mathds{1}(\iota_n^{-1}\le r_0(X)\le \iota_n)\right\} \\
& +3\mathbb{E}_P\left\{[\log(f(X)+1)-\log(r_0(X)+1)]^2r_0(X)\mathds{1}(\iota_n^{-1}\le r_0(X)\le \iota_n)\right\} \\
\le & 3(1+\bar{\delta}_n^2\iota_n+\iota_n)\mathbb{E}_P\left\{[r_0(X)-f(X)]^2\mathds{1}(\iota_n^{-1}\le r_0(X)\le \iota_n)\right\} \\
\le & 18(\log n)^{5+5\kappa}\mathbb{E}[g_{\iota_n}(f, T)],
\end{aligned}
$$
provided that $\bar{\delta}_n\ge \iota_n$. Therefore, Theorem \ref{thm: gen_thm_11.4_gyorfi} suggests that, for sufficiently large $n$ such that $\bar{\delta}_n\ge \iota_n$, with $n\ge \mathrm{Pdim}(\mathcal{F}_{\mathrm{NN}})$, and for arbitrary $t>0$, we have
$$
\begin{aligned}
& \mathbb{P}_{\mathcal{D}_n}\left\{\mathbb{E}_{T}[g_{\iota_n}(\hat{r}_{\mathrm{LR}}, T)]-\frac 2n\sum_{i=1}^ng_{\iota_n}(\hat{r}_{\mathrm{LR}}, T_i)\ge t\right\} \\
\le & \mathbb{P}_{\mathcal{D}_n}\left\{\mathbb{E}_{T}[g_{\iota_n}(\hat{r}_{\mathrm{LR}}, T)]-\frac 1n\sum_{i=1}^ng_{\iota_n}(\hat{r}_{\mathrm{LR}}, T_i)\ge \frac 12\left\{\frac t2+\frac t2+\mathbb{E}_{T}[g_{\iota_n}(\hat{r}_{\mathrm{LR}}, T)]\right\}\right\} \\
\le & \mathbb{P}_{\mathcal{D}_n}\left(\exists f\in\mathcal{F}_{\mathrm{NN}}: \mathbb{E}[g_{\iota_n}(f, T)]-\frac 1n\sum_{i=1}^ng_{\iota_n}(f, T_i)\ge \frac 12\left\{\frac t2+\frac t2+\mathbb{E}_{T}[g_{\iota_n}(f, T)]\right\}\right) \\
\le & 14\mathcal{N}_n\left(c_3t, \|\cdot \|_{\infty}, \{g_{\iota_n}(f, \cdot): \mathcal{X}^t\times \mathcal{X}^s\to\mathbb{R}, f\in\mathcal{F}_{\mathrm{NN}}\}\right)\exp\left(-\frac{nt}{c_4(\log n)^{5+5\kappa}}\right),
\end{aligned}
$$
where $c_3, c_4$ are universal constants and $\mathcal{X}^s, \mathcal{X}^t$ represents the domain of $X^s, X^t$, respectively. Subsequently, we bound the covering number. Fix $\{x^s_1, \dots, x^s_n\}\subset (\mathcal{X}^s)^n$ and $\{x^t_1, \dots, x^t_n\}\subset (\mathcal{X}^t)^n$. Let $\mathcal{C}=\{x^s_1, \dots, x^s_n, x^t_1, \dots, x^t_n\}$, and let $h^{\sharp}=\{h_1, \dots, h_k\}$ be an $\epsilon$-covering set of $\mathcal{F}_{\mathrm{NN}|\mathcal{C}}$ where $h_i=f_{i|\mathcal{C}}$ for some $f_i\in\mathcal{F}_{\mathrm{NN}} (i=1, \dots, k)$, such that for any $f\in \mathcal{F}_{\mathrm{NN}}$, there exists $h^*=f^*_{|\mathcal{C}}\in h^{\sharp}$ satisfying $\|h^*-f_{|\mathcal{C}}\|_{\infty}<\epsilon$. This indicates
$$
\begin{aligned}
& |g_{\iota_n}(f, (x^s_i, x^t_i))-g_{\iota_n}(f^*, (x^s_i, x^t_i))| \\
\le & |f(x^s_i)-f^*(x^s_i)|+\bar{\delta}_n|f(x^t_i)-f^*(x^t_i)|+|f(x^t_i)-f^*(x^t_i)| \\
\le & (\bar{\delta}_n+2)\epsilon.
\end{aligned}
$$
Therefore,
$$
\mathcal{N}_n(c_3t, \|\cdot \|_{\infty}, \{g_{\iota_n}(f, \cdot): \mathcal{X}^t\times \mathcal{X}^s\to\mathbb{R}, f\in\mathcal{F}_{\mathrm{NN}}\})\le \mathcal{N}_{2n}(c_3t/(\bar{\delta}_n+2), \|\cdot \|_{\infty}, \mathcal{F}_{\mathrm{NN}}).
$$
Then, with Lemma \ref{lem: thm_12.2_anthony} and Lemma \ref{lem: thm_7_bartlett}, for sufficiently large $n$ with $n\ge \mathrm{Pdim}(\mathcal{F}_{\mathrm{NN}})$ and any $a_n\ge 1/n$, we have
$$
\begin{aligned}
& \mathbb{E}_{\mathcal{D}_n}\left\{\mathbb{E}_{T}[g_{\iota_n}(\hat{r}_{\mathrm{LR}}, T)]-\frac 2n\sum_{i=1}^ng_{\iota_n}(\hat{r}_{\mathrm{LR}}, T_i)\right\} \\
\le & a_n+14\int_{a_n}^{\infty}\mathcal{N}_{2n}(c_3t/(\bar{\delta}_n+2), \|\cdot \|_{\infty}, \mathcal{F}_{\mathrm{NN}})\exp\left(-\frac{nt}{c_4(\log n)^{5+5\kappa}}\right)\mathrm{d}t \\
\le & a_n+14\mathcal{N}_{2n}(c_3a_n/(\bar{\delta}_n+2), \|\cdot \|_{\infty}, \mathcal{F}_{\mathrm{NN}})\int_{a_n}^{\infty}\exp\left(-\frac{nt}{c_4(\log n)^{5+5\kappa}}\right)\mathrm{d}t \\
\le & a_n+14\left(c_5n^2\bar{\delta}_n^2\right)^{c_6SL\log S}\cdot \frac{c_4(\log n)^{5+5\kappa}}{n}\exp\left(-\frac{na_n}{c_4(\log n)^{5+5\kappa}}\right),
\end{aligned}
$$
where $c_5$ and $c_6$ are universal constants. Choose
$$
a_n=\frac{c_4c_6(\log n)^{5+5\kappa}}{n}SL\log S\log\left(c_5n^2\bar{\delta}_n^2\right).
$$
For sufficiently large $n$, we have
$$
\mathbb{E}_{\mathcal{D}_n}\left\{\mathbb{E}_{T}[g_{\iota_n}(\hat{r}_{\mathrm{LR}}, T)]-\frac 2n\sum_{i=1}^ng_{\iota_n}(\hat{r}_{\mathrm{LR}}, T_i)\right\}\le \frac{c_7SL\log S(\log n)^{6+5\kappa}}{n},
$$
where $c_7$ is a constant not depending on $S, L$ and $n$. This completes the proof.
\end{proof}

\begin{lemma}\label{lem: dre_approximation_error_bound_lr}
Assume that
\begin{enumerate}[label=(\roman*)]
    \item $r_0(x)\in \mathcal{H}^{\beta_r}_{\mathrm{Loc}}(\mathbb{R}^d, B_u)$ with $\beta_r>0$ and $B_u\le c(u^m+1)$ for some universal constants $c>0$, $m\ge 0$;
    \item $r_0(X^s)$, $r_0(X^s)^{-1}\mathds{1}(X^s\in\mathcal{X}^t)$ and $\|X^s\|_{\infty}$ are sub-exponentially distributed random variables.
\end{enumerate}
Suppose that the depth $L$ and width $M$ of $\mathcal{F}_{\mathrm{NN}}$ are expressed as
$$
\begin{aligned}
L &= 21(\lfloor\beta_r\rfloor+1)^2S_1\lceil\log_2(8S_1)\rceil+2d+3, \\
M &= 38(\lfloor\beta_r\rfloor+1)^2d^{\lfloor\beta_r\rfloor+1}S_2\lceil\log_2(8S_2)\rceil,
\end{aligned}
$$
for any $S_1, S_2\in\mathbb{N}_+$. Let $\bar{\delta}=\bar{\delta}_n=(\log n)^{1+\kappa}$ and $\underline{\delta}=\underline{\delta}_n=(\log n)^{-1-\kappa}$ for arbitrarily fixed $\kappa\in (0, 1]$. Then, for sufficiently large $n$, it follows that
$$
\begin{aligned}
& \inf_{f\in\mathcal{F}_{\mathrm{NN}}}\mathbb{E}_PD_{\mathrm{LR}}(r_0(X)\| f(X)) \\
\le & c^*\left\{\left[(\lfloor\beta_r\rfloor+1)^2d^{\lfloor\beta_r\rfloor+(\beta_r\vee 1)/2}(S_1S_2)^{-2\beta_r/d}(\log n)^m\right]^2(\log n)^{1+\kappa}+\frac{(\log n)^2}{n}\right\},
\end{aligned}
$$
where $c^*$ is a constant not depending on $S_1, S_2$ and $n$.
\end{lemma}

\begin{proof}[Proof of Lemma \ref{lem: dre_approximation_error_bound_lr}]
For any $\iota_n\ge 1$, observe that
$$
\begin{aligned}
\mathbb{E}_PD_{\mathrm{LR}}(r_0(X)\| f(X))=& \mathbb{E}_P\left[D_{\mathrm{LR}}(r_0(X)\| f(X))\mathds{1}(X\in\mathcal{X}^t)\right] \\
=& \mathbb{E}_P\left[D_{\mathrm{LR}}(r_0(X)\| f(X))\mathds{1}(X\in\mathcal{X}^t)\mathds{1}(\iota_n^{-1}\le r_0(X)\le \iota_n)\right] \\
& +\mathbb{E}_P\left[D_{\mathrm{LR}}(r_0(X)\| f(X))\mathds{1}(X\in\mathcal{X}^t)\mathds{1}(r_0(X)<\iota_n^{-1})\right] \\
& +\mathbb{E}_P\left[D_{\mathrm{LR}}(r_0(X)\| f(X))\mathds{1}(X\in\mathcal{X}^t)\mathds{1}(r_0(X)>\iota_n)\right].
\end{aligned}
$$
On the one hand, for $n\ge 3$, it follows that
$$
\begin{aligned}
& \mathbb{E}_P\left[D_{\mathrm{LR}}(r_0(X)\| f(X))\mathds{1}(X\in\mathcal{X}^t)\mathds{1}(r_0(X)<\iota_n^{-1})\right] \\
=& \mathbb{E}_P\Big(\big\{r_0(X)\log r_0(X)-[r_0(X)+1]\log(r_0(X)+1)+\log(f(X)+1) \\
& -r_0(X)\log f(X)+r_0(X)\log(f(X)+1)\big\}\mathds{1}(X\in\mathcal{X}^t)\mathds{1}(r_0(X)<\iota_n^{-1})\Big) \\
\le & 3\log(\bar{\delta}_n+1)\mathbb{E}_P\left[\mathds{1}(X\in\mathcal{X}^t)\mathds{1}(r_0(X)<\iota_n^{-1})\right] \\
\le & 3\log(\bar{\delta}_n+1)\mathbb{E}_P\left[\mathds{1}(r_0(X)^{-1}\mathds{1}(X\in\mathcal{X}^t)>\iota_n)\right] \\
\le & 3\log(\bar{\delta}_n+1)\mathbb{E}_P[\exp(\varsigma r_0(X)^{-1}\mathds{1}(X\in\mathcal{X}^t)/2)]\exp(-\varsigma\iota_n/2),
\end{aligned}
$$
and
$$
\begin{aligned}
& \mathbb{E}_P\left[D_{\mathrm{LR}}(r_0(X)\| f(X))\mathds{1}(X\in\mathcal{X}^t)\mathds{1}(r_0(X)>\iota_n)\right] \\
=& \mathbb{E}_P\bigg(\big\{r_0(X)\log r_0(X)-[r_0(X)+1]\log(r_0(X)+1)+\log(f(X)+1) \\
& -r_0(X)\log f(X)+r_0(X)\log(f(X)+1)\big\}\mathds{1}(r_0(X)>\iota_n)\bigg) \\
\le & \mathbb{E}_P\left\{\left[1+2r_0(X)+2r_0(X)^2+\log(\bar{\delta}_n+1)+2r_0(X)\log(\bar{\delta}_n+1)\right]\mathds{1}(r_0(X)>\iota_n)\right\} \\
\le & \mathbb{E}_P\Bigg\{\Big[1+\frac{4}{\varsigma}\exp(\varsigma r_0(X)/2)+\frac{32}{\varsigma^2}\exp(\varsigma r_0(X)/2)+\log(\bar{\delta}_n+1) \\
& +\frac{4}{\varsigma}\exp(\varsigma r_0(X)/2)\log(\bar{\delta}_n+1)\Big]\exp(\varsigma r_0(X)/2)\Bigg\}\exp(-\varsigma\iota_n/2) \\
\le & c_1\left[1+\log(\bar{\delta}_n+1)\right]\exp(-\varsigma\iota_n/2),
\end{aligned}
$$
where $c_1$ is a constant which depends only on $\varsigma$ and $\mathbb{E}_P[\exp(\varsigma r_0(X))]$. On the other hand, we first notice that
$$
\begin{aligned}
& \mathbb{E}_P\left[D_{\mathrm{LR}}(r_0(X)\| f(X))\mathds{1}(X\in\mathcal{X}^t)\mathds{1}(\iota_n^{-1}\le r_0(X)\le \iota_n)\right] \\
=& \mathbb{E}_P\left[D_{\mathrm{LR}}(r_0(X)\| f(X))\mathds{1}(X\in\mathcal{X}^t)\mathds{1}(\iota_n^{-1}\le r_0(X)\le \iota_n)\mathds{1}(\|X\|_{\infty}\le\iota_n)\right] \\
& +\mathbb{E}_P\left[D_{\mathrm{LR}}(r_0(X)\| f(X))\mathds{1}(X\in\mathcal{X}^t)\mathds{1}(\iota_n^{-1}\le r_0(X)\le \iota_n)\mathds{1}(\|X\|_{\infty}>\iota_n)\right] \\
\le & \mathbb{E}_P\left[D_{\mathrm{LR}}(r_0(X)\| f(X))\mathds{1}(X\in\mathcal{X}^t)\mathds{1}(\iota_n^{-1}\le r_0(X)\le \iota_n)\mathds{1}(\|X\|_{\infty}\le\iota_n)\right] \\
& +\mathbb{E}_P\big\{[\log(f(X)+1)-r_0(X)\log f(X)+r_0(X)\log(f(X)+1)] \\
& \cdot\mathds{1}(\iota_n^{-1}\le r_0(X)\le \iota_n)\mathds{1}(\|X\|_{\infty}>\iota_n)\big\} \\
\le & \mathbb{E}_P\left[D_{\mathrm{LR}}(r_0(X)\| f(X))\mathds{1}(X\in\mathcal{X}^t)\mathds{1}(\iota_n^{-1}\le r_0(X)\le \iota_n)\mathds{1}(\|X\|_{\infty}\le\iota_n)\right] \\
& +3\iota_n\log(\bar{\delta}_n+1)\mathbb{E}_P\mathds{1}(\|X\|_{\infty}>\iota_n) \\
\le & \mathbb{E}_P\left[D_{\mathrm{LR}}(r_0(X)\| f(X))\mathds{1}(X\in\mathcal{X}^t)\mathds{1}(\iota_n^{-1}\le r_0(X)\le \iota_n)\mathds{1}(\|X\|_{\infty}\le\iota_n)\right] \\
& +3\iota_n\log(\bar{\delta}_n+1)\mathbb{E}_P[\exp(\varsigma\|X\|_{\infty}/2)]\exp(-\varsigma\iota_n/2) \\
\le & \frac 12(\bar{\delta}_n\vee \iota_n)\mathbb{E}_P\left\{[r_0(X)-f(X)]^2\mathds{1}(\iota_n^{-1}\le r_0(X)\le \iota_n)\mathds{1}(\|X\|_{\infty}\le \iota_n)\right\} \\
& +3\iota_n\log(\bar{\delta}_n+1)\mathbb{E}_P[\exp(\varsigma\|X\|_{\infty}/2)]\exp(-\varsigma\iota_n/2).
\end{aligned}
$$
Then, we focus on the region $\{x: \|x\|_{\infty}\le \iota_n\}=[-\iota_n, \iota_n]^d$. Let $r_0^*(x)=r_0(2\iota_nx-\iota_n\mathrm{1}_d)$ for $x\in [0, 1]^d$. Lemma \ref{lem: thm_3.3_jiao} demonstrates that for any $S_1, S_2\in\mathbb{N}_+$, there exists a function $f^*$ implemented by a ReLU network with depth $L^*=21(\lfloor\beta_r\rfloor+1)^2S_1\lceil\log_2(8S_1)\rceil+2d$, width $M^*=38(\lfloor\beta_r\rfloor+1)^2d^{\lfloor\beta_r\rfloor+1}S_2\lceil\log_2(8S_2)\rceil$, such that
$$
|f^*(x)-r_0^*(x)|\le 18c(\iota_n^m+1)(\lfloor\beta_r\rfloor+1)^2d^{\lfloor\beta_r\rfloor+(\beta_r\vee 1)/2}(S_1S_2)^{-2\beta_r/d},
$$
for all $x\in [0, 1]^d\backslash \Omega([0, 1]^d, K, \Delta)$. Here,
$$
\Omega([0, 1]^d, K, \Delta)=\bigcup_{i=1}^d\left\{x=(x_1, \dots, x_d)^{\top}: x_i\in\bigcup_{k=1}^{K-1}(k/K-\Delta, k/K)\right\},
$$
where $K=\lceil(S_1S_2)^{2/d}\rceil$ and $\Delta$ is an arbitrary scalar in $(0, 1/(3K)]$. Let $f^{\dagger}(x)=f^*((x+\iota_n\mathrm{1}_d)/(2\iota_n))$ for $x\in[-\iota_n, \iota_n]^d$. We obtain that
$$
|f^{\dagger}(x)-r_0(x)|\le 18c(\iota_n^m+1)(\lfloor\beta_r\rfloor+1)^2d^{\lfloor\beta_r\rfloor+(\beta_r\vee 1)/2}(S_1S_2)^{-2\beta_r/d},
$$
for all $x\in [-\iota_n, \iota_n]^d\backslash \Omega^{\dagger}$, where $\Omega^{\dagger}=\{x: (x+\iota_n\mathrm{1}_d)/(2\iota_n)\in \Omega([0, 1]^d, K, \Delta)\}$. Furthermore, note that
$$
f^{\dagger}(x)=f^*\left(\frac{x+\iota_n\mathrm{1}_d}{2\iota_n}\right)=f^*\left(\mathrm{relu}\left(\frac{x+\iota_n\mathrm{1}_d}{2\iota_n}\right)-\mathrm{relu}\left(-\frac{x+\iota_n\mathrm{1}_d}{2\iota_n}\right)\right),
$$
which is implemented by a neural network with ReLU activations, depth $L^{\dagger}=L^*+1$, and width $M^{\dagger}=M^*$. In addition, let
$$
f^{\ddagger}(x)=
\begin{cases}
    \bar{\delta}_n, & f^{\dagger}(x)>\bar{\delta}_n, \\
    f^{\dagger}(x), & \underline{\delta}_n\le f^{\dagger}(x)\le \bar{\delta}_n, \\
    \underline{\delta}_n, & f^{\dagger}(x)<\underline{\delta}_n.
\end{cases}
$$
A straightforward calculation shows that
$$
f^{\ddagger}(x)=\mathrm{relu}(-\mathrm{relu}(-f^{\dagger}(x)+\bar{\delta}_n)+\bar{\delta}_n-\underline{\delta}_n)+\underline{\delta}_n,
$$
indicating that $f^{\ddagger}(x)$ can be implemented by a ReLU network with depth $L=L^*+3$ and width $M=M^*$. Due to the arbitrariness of $\Delta$, when $\bar{\delta}_n\ge \iota_n$, it follows that
$$
\begin{aligned}
& \inf_{f\in \mathcal{F}_{\mathrm{NN}}}\mathbb{E}_P\left\{[r_0(X)-f(X)]^2\mathds{1}(\iota_n^{-1}\le r_0(X)\le \iota_n)\mathds{1}(\|X\|_{\infty}\le \iota_n)\right\} \\
\le & \mathbb{E}_P\left\{[r_0(X)-f^{\ddagger}(X)]^2\mathds{1}(\iota_n^{-1}\le r_0(X)\le \iota_n)\mathds{1}(\|X\|_{\infty}\le \iota_n)\right\} \\
\le & \left[18c(\iota_n^m+1)(\lfloor\beta_r\rfloor+1)^2d^{\lfloor\beta_r\rfloor+(\beta_r\vee 1)/2}(S_1S_2)^{-2\beta_r/d}\right]^2.
\end{aligned}
$$
Recall that $\bar{\delta}_n=(\log n)^{1+\kappa}$, and set $\iota_n$ to $(2\varsigma^{-1}\log n)\vee 1$. We conclude that for sufficiently large $n$ satisfying $\bar{\delta}_n\ge \iota_n$ and $\log n\ge \varsigma/2$, it holds that
$$
\begin{aligned}
& \inf_{f\in\mathcal{F}_{\mathrm{NN}}}\mathbb{E}_PD_{\mathrm{LR}}(r_0(X)\| f(X)) \\
\le & \frac 12(\bar{\delta}_n\vee \iota_n)\inf_{f\in \mathcal{F}_{\mathrm{NN}}}\mathbb{E}_P\left\{[r_0(X)-f(X)]^2\mathds{1}(\iota_n^{-1}\le r_0(X)\le \iota_n)\mathds{1}(\|X\|_{\infty}\le \iota_n)\right\} \\
& +c_2\iota_n\log(\bar{\delta}_n+1)\exp(-\varsigma\iota_n/2) \\
\le & c_3\left\{\left[(\lfloor\beta_r\rfloor+1)^2d^{\lfloor\beta_r\rfloor+(\beta_r\vee 1)/2}(S_1S_2)^{-2\beta_r/d}(\log n)^m\right]^2(\log n)^{1+\kappa}+\frac{(\log n)^2}{n}\right\},
\end{aligned}
$$
where $c_2, c_3$ are constants not depending on $S_1, S_2$ and $n$.
\end{proof}

\begin{proof}[Proof of Theorem 3.3]
To commence, we notice that Lemmas \ref{lem: dre_error_decomposition_lr}, \ref{lem: dre_stochatic_error_bound_lr} and \ref{lem: dre_approximation_error_bound_lr} indicate
$$
\begin{aligned}
& \mathbb{E}D_{\mathrm{LR}}(r_0(X^s)\| \hat{r}_{\mathrm{LR}}(X^s)) \\
\le & \frac{c_1SL\log S(\log n)^{6+5\kappa}}{n} \\
& +c_2\left\{\left[(\lfloor\beta_r\rfloor+1)^2d^{\lfloor\beta_r\rfloor+(\beta_r\vee 1)/2}(S_1S_2)^{-2\beta_r/d}(\log n)^m\right]^2(\log n)^{1+\kappa}+\frac{(\log n)^2}{n}\right\},
\end{aligned}
$$
where $c_1, c_2$ are constants not depending on $S, L, S_1, S_2$ and $n$, and $S_1, S_2$ satisfy the conditions that the network depth $L=21(\lfloor\beta_r\rfloor+1)^2S_1\lceil\log_2(8S_1)\rceil+2d+3$, network width $M=38(\lfloor\beta_r\rfloor+1)^2d^{\lfloor\beta_r\rfloor+1}S_2\lceil\log_2(8S_2)\rceil$, for sufficiently large $n$ and $n\ge \mathrm{Pdim}(\mathcal{F}_{\mathrm{NN}})$. Therefore, by letting $S_1=\mathcal{O}(n^{d/(2d+4\beta_r)})$ and $S_2=\mathcal{O}(1)$,  we obtain
$$
M=\mathcal{O}(1), \quad L=\mathcal{O}\left(n^{\frac{d}{2d+4\beta_r}}\log n\right), \quad S=\mathcal{O}(M^2L)=\mathcal{O}\left(n^{\frac{d}{2d+4\beta_r}}\log n\right),
$$
yielding
$$
\mathbb{E}D_{\mathrm{LR}}(r_0(X^s)\| \hat{r}_{\mathrm{LR}}(X^s))\le c_3n^{-\frac{2\beta_r}{d+2\beta_r}}(\log n)^{(9+5\kappa)\vee (2m+1+\kappa)},
$$
where $c_3$ is a constant not depending on $n$, for $n\ge 3$.
Furthermore, note that for any $\iota_n\ge 1$, it follows that
$$
\begin{aligned}
\mathcal{R}^s(\hat{r}_{\mathrm{LR}}) =& \mathbb{E}\left[r_0(X^s)-\hat{r}_{\mathrm{LR}}(X^s)\right]^2 \\
=& \mathbb{E}\left\{\left[r_0(X^s)-\hat{r}_{\mathrm{LR}}(X^s)\right]^2\mathds{1}(\iota_n^{-1}\le r_0(X^s)\le \iota_n)\right\} \\
& +\mathbb{E}\left\{\left[r_0(X^s)-\hat{r}_{\mathrm{LR}}(X^s)\right]^2\mathds{1}(X^s\in\mathcal{X}^t)\mathds{1}(r_0(X^s)<\iota_n^{-1})\right\} \\
& +\mathbb{E}\left\{\left[r_0(X^s)-\hat{r}_{\mathrm{LR}}(X^s)\right]^2\mathds{1}(r_0(X^s)>\iota_n)\right\}.
\end{aligned}
$$
Specifically, on one hand, the smoothness of $\varphi_{\mathrm{LR}}$ demonstrates that
$$
\begin{aligned}
& \mathbb{E}\left\{\left[r_0(X^s)-\hat{r}_{\mathrm{LR}}(X^s)\right]^2\mathds{1}(\iota_n^{-1}\le r_0(X^s)\le \iota_n)\right\} \\
\le & 2(\bar{\delta}_n\vee \iota_n)[(\bar{\delta}_n\vee \iota_n)+1]\mathbb{E}\left[D_{\mathrm{LR}}(r_0(X^s)\| \hat{r}_{\mathrm{LR}}(X^s))\mathds{1}(\iota_n^{-1}\le r_0(X^s)\le \iota_n)\right] \\
\le & 2(\bar{\delta}_n\vee \iota_n)[(\bar{\delta}_n\vee \iota_n)+1]\mathbb{E}D_{\mathrm{LR}}(r_0(X^s)\| \hat{r}_{\mathrm{LR}}(X^s)).
\end{aligned}
$$
On the other hand, observe that
$$
\begin{aligned}
& \mathbb{E}\left\{\left[r_0(X^s)-\hat{r}_{\mathrm{LR}}(X^s)\right]^2\mathds{1}(X^s\in\mathcal{X}^t)\mathds{1}(r_0(X^s)<\iota_n^{-1})\right\} \\
\le & 2(1+\bar{\delta}_n^2)\mathbb{E}_P\left[\mathds{1}(r_0(X)\mathds{1}(X\in\mathcal{X}^t)<\iota_n^{-1})\right] \\
\le & 2(1+\bar{\delta}_n^2)\mathbb{E}_P\mathds{1}(r_0(X)^{-1}\mathds{1}(X\in\mathcal{X}^t)>\iota_n) \\
\le & 2(1+\bar{\delta}_n^2)\mathbb{E}[\exp(\varsigma r_0(X)^{-1}\mathds{1}(X\in\mathcal{X}^t)/2)]\exp(-\varsigma\iota_n/2),
\end{aligned}
$$
and
$$
\begin{aligned}
& \mathbb{E}\left\{\left[r_0(X^s)-\hat{r}_{\mathrm{LR}}(X^s)\right]^2\mathds{1}(r_0(X^s)>\iota_n)\right\} \\
\le & 2\mathbb{E}\left[r_0(X^s)^2\mathds{1}(r_0(X^s)>\iota_n)\right]+2\bar{\delta}_n^2\mathbb{E}\mathds{1}(r_0(X^s)>\iota_n) \\
\le & \frac{32}{\varsigma^2}\mathbb{E}_P[\exp(\varsigma r_0(X))]\exp(-\varsigma\iota_n/2)+2\bar{\delta}_n^2\mathbb{E}_P[\exp(\varsigma r_0(X)/2)]\exp(-\varsigma\iota_n/2).
\end{aligned}
$$
Hence, by letting $\iota_n=(2\varsigma^{-1}\log n)\vee 1$, we have for $n\ge 3$,
$$
\mathcal{R}^s(\hat{r}_{\mathrm{LR}})\le c_4n^{-\frac{2\beta_r}{d+2\beta_r}}(\log n)^{(11+7\kappa)\vee (2m+3+3\kappa)},
$$
where $c_4$ is a constant not depending on $n$. In addition, we note that
$$
\begin{aligned}
\mathcal{R}^t(\hat{r}_{\mathrm{LR}}) &= \mathbb{E}\left[\hat{r}_{\mathrm{LR}}(X^t)-r_0(X^t)\right]^2 \\
&= \mathbb{E}\left\{\left[\hat{r}_{\mathrm{LR}}(X^s)-r_0(X^s)\right]^2r_0(X^s)\right\}.
\end{aligned}
$$
Similar truncation operation suggests that, for any $\varrho_n>0$, we have
$$
\begin{aligned}
\mathcal{R}^t(\hat{r}_{\mathrm{LR}}) =& \mathbb{E}\left\{\left[\hat{r}_{\mathrm{LR}}(X^s)-r_0(X^s)\right]^2r_0(X^s)\right\} \\
=& \mathbb{E}\left\{\left[\hat{r}_{\mathrm{LR}}(X^s)-r_0(X^s)\right]^2r_0(X^s)\mathds{1}(r_0(X^s)\le \varrho_n)\right\} \\
& +\mathbb{E}\left\{\left[\hat{r}_{\mathrm{LR}}(X^s)-r_0(X^s)\right]^2r_0(X^s)\mathds{1}(r_0(X^s)> \varrho_n)\right\} \\
\le & \varrho_n\mathbb{E}\left\{\left[\hat{r}_{\mathrm{LR}}(X^s)-r_0(X^s)\right]^2\right\} \\
& +\mathbb{E}\left\{\left[\hat{r}_{\mathrm{LR}}(X^s)-r_0(X^s)\right]^2r_0(X^s)\mathds{1}(r_0(X^s)> \varrho_n)\right\} \\
=& \varrho_n\mathcal{R}^s(\hat{r}_{\mathrm{LR}})+\mathbb{E}\left\{\left[\hat{r}_{\mathrm{LR}}(X^s)-r_0(X^s)\right]^2r_0(X^s)\mathds{1}(r_0(X^s)> \varrho_n)\right\}.
\end{aligned}
$$
Subsequently, note that
$$
\begin{aligned}
& \mathbb{E}\left\{\left[\hat{r}_{\mathrm{LR}}(X^s)-r_0(X^s)\right]^2r_0(X^s)\mathds{1}(r_0(X^s)> \varrho_n)\right\} \\
\le & 2\bar{\delta}_n^2\mathbb{E}_P\left[r_0(X)\mathds{1}(r_0(X)> \varrho_n)\right]+2\mathbb{E}_P\left[r_0(X)^3\mathds{1}(r_0(X)> \varrho_n)\right] \\
\le & \frac{4}{\varsigma}\bar{\delta}_n^2\mathbb{E}_P[\exp(\varsigma r_0(X))]\exp(-\varsigma\varrho_n/2)+\frac{432}{\varsigma^3}\mathbb{E}_P[\exp(\varsigma r_0(X))]\exp(-\varsigma\varrho_n/2) \\
\le & c_5[(\log n)^{2+2\kappa}+1]\exp(-\varsigma\varrho_n/2),
\end{aligned}
$$
where $c_5$ is a constant which depends only on $\varsigma$ and $\mathbb{E}_P[\exp(\varsigma r_0(X))]$. Hence, by taking $\varrho_n=(2\varsigma^{-1}\log n)\vee 1$, we obtain that
$$
\mathcal{R}^t(\hat{r}_{\mathrm{LR}})\le [(2\varsigma^{-1}\log n)\vee 1]\mathcal{R}^s(\hat{r}_{\mathrm{LR}})+\frac{c_5[(\log n)^{2+2\kappa}+1]}{n}.
$$
This completes the proof.
\end{proof}

\subsection{Proof of Lemma 4.1}\label{subsec: proof_of_lem_sub_exp_convergence}

\begin{proof}[Proof of Lemma 4.1]
For any $\iota_n>0$, observe that
$$
\begin{aligned}
\mathbb{E}\left(\|U_n-U\|_2^2|V|\right) &= \mathbb{E}\left(\|U_n-U\|_2^2|V|\mathds{1}(|V|\le \iota_n)\right)+\mathbb{E}\left(\|U_n-U\|_2^2|V|\mathds{1}(|V|>\iota_n)\right) \\
&\le \iota_n\gamma_n+\mathbb{E}\left(\|U_n-U\|_2^2|V|\mathds{1}(|V|>\iota_n)\right) \\
&\le \iota_n\gamma_n+4\varsigma^{-1}\mathbb{E}\left(\|U_n-U\|_2^2\exp(\varsigma|V|/4)\mathds{1}(|V|>\iota_n)\right) \\
&\le \iota_n\gamma_n+4\varsigma^{-1}\mathbb{E}\left(\|U_n-U\|_2^2\exp(\varsigma|V|/2)\right)\exp(-\varsigma\iota_n/4),
\end{aligned}
$$
where we have applied the inequalities that $a\le \exp(a)$ and $\mathds{1}(a>0)\le \exp(a)$. Let $U_{n, (j)}$ and $U_{(j)}$ be the $j$-th entry of $U_n$ and $U$, respectively, for $j=1, \dots, d$. Note that
$$
\begin{aligned}
\mathbb{E}\left(\|U_n-U\|_2^2\exp(\varsigma|V|/2)\right) &=\sum_{j=1}^d\mathbb{E}\left[\left(U_{n, (j)}-U_{(j)}\right)^2\exp(\varsigma|V|/2)\right] \\
&\le \sum_{j=1}^d\left[\mathbb{E}\left(U_{n, (j)}-U_{(j)}\right)^4\mathbb{E}\exp(\varsigma|V|)\right]^{1/2} \\
&\le \sum_{j=1}^d\left[8\mathbb{E}\left(U_{n, (j)}^4+U_{(j)}^4\right)\mathbb{E}\exp(\varsigma|V|)\right]^{1/2} \\
&\le d\left[8\left(\xi_n^4+\mathbb{E}\|U\|_{\infty}^4\right)\mathbb{E}\exp(\varsigma|V|)\right]^{1/2} \\
&\le c_3d(\xi_n^2+1),
\end{aligned}
$$
where $c_3$ is a constant only depending on $\mathbb{E}\|U\|_{\infty}^4$ and $\mathbb{E}\exp(\varsigma|V|)$. Therefore, we have
$$
\mathbb{E}\left(\|U_n-U\|_2^2|V|\right)\le \iota_n\gamma_n+4c_3\varsigma^{-1}d(\xi_n^2+1)\exp(-\varsigma\iota_n/4).
$$
Let $\iota_n=4\varsigma^{-1}\log n$. Then, for $n\ge 2$, it follows that
$$
\mathbb{E}\left(\|U_n-U\|_2^2|V|\right)\le 4\varsigma^{-1}\gamma_n\log n+\frac{4c_3\varsigma^{-1}d(\xi_n^2+1)}{n}.
$$
This completes the proof.
\end{proof}

\subsection{Proof of Proposition 4.3}\label{subsec: proof_of_prop_poly_dr}

\begin{proof}[Proof of Proposition 4.3]
For any $\iota_N>0$, observe that
$$
\begin{aligned}
& \mathbb{E}\left\|\hat{\theta}_N(X^t)-\theta_0(X^t)\right\|_2^2 \\
=& \mathbb{E}\left[\left\|\hat{\theta}_N(X^s)-\theta_0(X^s)\right\|_2^2\cdot r_0(X^s)\right] \\
=& \mathbb{E}\left[\left\|\hat{\theta}_N(X^s)-\theta_0(X^s)\right\|_2^2\cdot r_0(X^s)\mathds{1}(\|X^s\|_{\infty}\le \iota_N)\right] \\
& +\mathbb{E}\left[\left\|\hat{\theta}_N(X^s)-\theta_0(X^s)\right\|_2^2\cdot r_0(X^s)\mathds{1}(\|X^s\|_{\infty}>\iota_N)\right] \\
\le & G(\iota_N)\mathbb{E}\left\|\hat{\theta}_N(X^s)-\theta_0(X^s)\right\|_2^2 \\
& +\mathbb{E}\left[\left\|\hat{\theta}_N(X^s)-\theta_0(X^s)\right\|_2^2\cdot r_0(X^s)\mathds{1}(\|X^s\|_{\infty}>\iota_N)\right].
\end{aligned}
$$
Let $\hat{\theta}_{N, (j)}(X^s)$ and $\theta_{0, (j)}(X^s)$ be the $j$-th component of $\hat{\theta}_N(X^s)$ and $\theta_0(X^s)$, respectively, for $j=1, \dots, k$. By using Cauchy-Schwarz inequality twice, we have
$$
\begin{aligned}
& \mathbb{E}\left[\left\|\hat{\theta}_N(X^s)-\theta_0(X^s)\right\|_2^2\cdot r_0(X^s)\mathds{1}(\|X^s\|_{\infty}>\iota_N)\right] \\
=& \sum_{j=1}^k\mathbb{E}\left\{\left[\hat{\theta}_{N, (j)}(X^s)-\theta_{0, (j)}(X^s)\right]_2^2\cdot r_0(X^s)\mathds{1}(\|X^s\|_{\infty}>\iota_N)\right\} \\
\le & \sum_{j=1}^k\mathbb{E}\left\{\left[\hat{\theta}_{N, (j)}(X^s)-\theta_{0, (j)}(X^s)\right]_2^2\cdot r_0(X^s)\exp(\varsigma\|X^s\|_{\infty}/4)\right\}\exp(-\varsigma\iota_N/4) \\
\le & \sum_{j=1}^k\left(\mathbb{E}\left\{\left[\hat{\theta}_{N, (j)}(X^s)-\theta_{0, (j)}(X^s)\right]_2^4\exp(\varsigma\|X^s\|_{\infty}/2)\right\}\mathbb{E}\left[r_0(X^s)^2\right]\right)^{1/2}\exp(-\varsigma\iota_N/4) \\
\le & \sum_{j=1}^k\left(\left\{\mathbb{E}\left[\hat{\theta}_{N, (j)}(X^s)-\theta_{0, (j)}(X^s)\right]_2^8\mathbb{E}\exp(\varsigma\|X^s\|_{\infty})\right\}^{1/2}\mathbb{E}\left[r_0(X^s)^2\right]\right)^{1/2}\exp(-\varsigma\iota_N/4) \\
\le & k\left(\left\{128\left(\xi_N^8+\mathbb{E}\left\|\theta_0(X^s)\right\|^8\right)\mathbb{E}\exp(\varsigma\|X^s\|_{\infty})\right\}^{1/2}\mathbb{E}\left[r_0(X^s)^2\right]\right)^{1/2}\exp(-\varsigma\iota_N/4) \\
\le & c_3k\left(\xi_N^2+1\right)\exp(-\varsigma\iota_N/4),
\end{aligned}
$$
where $c_3$ is a constant only depending on $\mathbb{E}\|\theta_0(X^s)\|^8$, $\mathbb{E}\exp(\varsigma\|X^s\|_{\infty})$ and $\mathbb{E}[r_0(X^s)^2]$. Hence, let $\iota_N=4\varsigma^{-1}\log N$ and we obtain the result.
\end{proof}

\subsection{Proof of Proposition 4.4}\label{subsec: proof_of_prop_exp_dr}

\begin{proof}[Proof of Proposition 4.4]
For any $\iota_N>0$, observe that
$$
\begin{aligned}
& \mathbb{E}\left\|\hat{\theta}_N(X^t)-\theta_0(X^t)\right\|_2^2 \\
=& \mathbb{E}\left[\left\|\hat{\theta}_N(X^s)-\theta_0(X^s)\right\|_2^2\cdot r_0(X^s)\right] \\
=& \mathbb{E}\left[\left\|\hat{\theta}_N(X^s)-\theta_0(X^s)\right\|_2^2\cdot r_0(X^s)\mathds{1}(\|X^s\|_{\infty}\le \iota_N)\right] \\
& +\mathbb{E}\left[\left\|\hat{\theta}_N(X^s)-\theta_0(X^s)\right\|_2^2\cdot r_0(X^s)\mathds{1}(\|X^s\|_{\infty}>\iota_N)\right] \\
\le & G(\iota_N)\mathbb{E}\left\|\hat{\theta}_N(X^s)-\theta_0(X^s)\right\|_2^2 \\
& +\mathbb{E}\left[\left\|\hat{\theta}_N(X^s)-\theta_0(X^s)\right\|_2^2\cdot r_0(X^s)\mathds{1}(\|X^s\|_{\infty}>\iota_N)\right].
\end{aligned}
$$
Let $\hat{\theta}_{N, (j)}(X^s)$ and $\theta_{0, (j)}(X^s)$ be the $j$-th component of $\hat{\theta}_N(X^s)$ and $\theta_0(X^s)$, respectively, for $j=1, \dots, k$. By using Cauchy-Schwarz inequality twice, we have
$$
\begin{aligned}
& \mathbb{E}\left[\left\|\hat{\theta}_N(X^s)-\theta_0(X^s)\right\|_2^2\cdot r_0(X^s)\mathds{1}(\|X^s\|_{\infty}>\iota_N)\right] \\
=& \sum_{j=1}^k\mathbb{E}\left\{\left[\hat{\theta}_{N, (j)}(X^s)-\theta_{0, (j)}(X^s)\right]_2^2\cdot r_0(X^s)\mathds{1}(\|X^s\|_{\infty}>\iota_N)\right\} \\
\le & \sum_{j=1}^k\mathbb{E}\left\{\left[\hat{\theta}_{N, (j)}(X^s)-\theta_{0, (j)}(X^s)\right]_2^2\cdot r_0(X^s)\exp(\varsigma\|X^s\|_{\infty}^2/4)\right\}\exp(-\varsigma\iota_N^2/4) \\
\le & \sum_{j=1}^k\left(\mathbb{E}\left\{\left[\hat{\theta}_{N, (j)}(X^s)-\theta_{0, (j)}(X^s)\right]_2^4\exp(\varsigma\|X^s\|_{\infty}^2/2)\right\}\mathbb{E}\left[r_0(X^s)^2\right]\right)^{1/2}\exp(-\varsigma\iota_N^2/4) \\
\le & \sum_{j=1}^k\left(\left\{\mathbb{E}\left[\hat{\theta}_{N, (j)}(X^s)-\theta_{0, (j)}(X^s)\right]_2^8\mathbb{E}\exp(\varsigma\|X^s\|_{\infty}^2)\right\}^{1/2}\mathbb{E}\left[r_0(X^s)^2\right]\right)^{1/2}\exp(-\varsigma\iota_N^2/4) \\
\le & k\left(\left\{128\left(\xi_N^8+\mathbb{E}\left\|\theta_0(X^s)\right\|^8\right)\mathbb{E}\exp(\varsigma\|X^s\|_{\infty}^2)\right\}^{1/2}\mathbb{E}\left[r_0(X^s)^2\right]\right)^{1/2}\exp(-\varsigma\iota_N^2/4) \\
\le & c_3k\left(\xi_N^2+1\right)\exp(-\varsigma\iota_N^2/4),
\end{aligned}
$$
where $c_3$ is a constant only depending on $\mathbb{E}\|\theta_0(X^s)\|^8$, $\mathbb{E}\exp(\varsigma\|X^s\|_{\infty}^2)$ and $\mathbb{E}[r_0(X^s)^2]$. Hence, let $\iota_N=2(\varsigma^{-1}\log N)^{1/2}$ and we obtain the result.
\end{proof}

\subsection{Proof of Theorem 5.1}\label{subsec: proof_of_thm_reg_convergence_rate}

For any $X^s$-square-integrable function $f: \mathbb{R}^{d_x}\to\mathbb{R}^{d_y}$ such that $\mathbb{E}\|f(X^s)\|_2^2<\infty$, define
$$
\begin{aligned}
K^{\mathrm{reg}}(f) &= \mathbb{E}_P\|Y-f(X)\|_2^2, \\
K^{\mathrm{reg}}_N(f) &= \frac 1N\sum_{i=1}^N\left\|Y^s_i-f(X^s_i)\right\|_2^2.
\end{aligned}
$$
Here, $\mathbb{E}_P[h(X, Y)]\equiv \mathbb{E}[h(X^s, Y^s)]$ for any $(X^s, Y^s)$-integrable function $h$, where the expectation is taken with respect to $(X^s, Y^s)$.

\begin{lemma}\label{lem: reg_error_decomposition}
Assume that $\|Y^s\|_{\infty}$ and $\|f_0(X^s)\|_{\infty}$ attain a finite second moment. Then,
$$
\begin{aligned}
& \mathbb{E}\|\hat{f}^s_N(X^s)-f_0(X^s)\|_2^2 \\
\le & \mathbb{E}[K^{\mathrm{reg}}(\hat{f}^s_N)-2K^{\mathrm{reg}}_N(\hat{f}^s_N)+K^{\mathrm{reg}}(f_0)]+2\inf_{f\in\mathcal{F}_{\mathrm{NN}}^{d_y}}\mathbb{E}_P\|f(X)-f_0(X)\|_2^2.
\end{aligned}
$$
\end{lemma}

\begin{proof}[Proof of Lemma \ref{lem: reg_error_decomposition}]
Given that $\|Y^s\|_{\infty}$ and $\|f_0(X^s)\|_{\infty}$ have a finite second moment, we have $\mathbb{E}\|Y^s\|_2^2<\infty$ and $\mathbb{E}\|f_0(X^s)\|_2^2<\infty$. For any $f\in\mathcal{F}_{\mathrm{NN}}^{d_y}$, we have
$$
\begin{aligned}
& \mathbb{E}\|\hat{f}^s_N(X^s)-f_0(X^s)\|_2^2 \\
=& \mathbb{E}[K^{\mathrm{reg}}(\hat{f}^s_N)-K^{\mathrm{reg}}(f_0)] \\
\le & \mathbb{E}[K^{\mathrm{reg}}(\hat{f}^s_N)-K^{\mathrm{reg}}(f_0)]+2\mathbb{E}[K^{\mathrm{reg}}_N(f)-K^{\mathrm{reg}}_N(\hat{f}^s_N)] \\
=& \mathbb{E}[K^{\mathrm{reg}}(\hat{f}^s_N)-K^{\mathrm{reg}}(f_0)]+2\mathbb{E}[K^{\mathrm{reg}}_N(f)-K^{\mathrm{reg}}_N(f_0)+K^{\mathrm{reg}}_N(f_0)-K^{\mathrm{reg}}_N(\hat{f}^s_N)] \\
=& \mathbb{E}[K^{\mathrm{reg}}(\hat{f}^s_N)-2K^{\mathrm{reg}}_N(\hat{f}^s_N)+K^{\mathrm{reg}}(f_0)]+2[K^{\mathrm{reg}}(f)-K^{\mathrm{reg}}(f_0)] \\
=& \mathbb{E}[K^{\mathrm{reg}}(\hat{f}^s_N)-2K^{\mathrm{reg}}_N(\hat{f}^s_N)+K^{\mathrm{reg}}(f_0)]+2\mathbb{E}_P\|f(X)-f_0(X)\|_2^2.
\end{aligned}
$$
This indicates that $\mathcal{R}^s(\hat{f}^s_N)\le \mathbb{E}[K^{\mathrm{reg}}(\hat{f}^s_N)-2K^{\mathrm{reg}}_N(\hat{f}^s_N)+K^{\mathrm{reg}}(f_0)]+2\inf_{f\in\mathcal{F}_{\mathrm{NN}}^{d_y}}\mathbb{E}_P\|f(X)-f_0(X)\|_2^2$.
\end{proof}

\begin{lemma}\label{lem: reg_stochatic_error_bound}
Assume that $\|Y^s\|_{\infty}$ is sub-exponentially distributed. Let $\bar{\delta}=\bar{\delta}_N=(\log N)^{1+\kappa}$, with an arbitrarily fixed $\kappa\in (0, 1]$, and let $\underline{\delta}=\underline{\delta}_N=-(\log N)^{1+\kappa}$. Then, for sufficiently large $N$ and $N\ge \mathrm{Pdim}(\mathcal{F}_{\mathrm{NN}})$, it follows that
$$
\mathbb{E}[K^{\mathrm{reg}}(\hat{f}^s_N)-2K^{\mathrm{reg}}_N(\hat{f}^s_N)+K^{\mathrm{reg}}(f_0)]\le \frac{c^*SL\log S(\log N)^{5+4\kappa}}{N},
$$
where $c^*$ is a constant not depending on $S, L$ and $N$.
\end{lemma}

\begin{proof}[Proof of Lemma \ref{lem: reg_stochatic_error_bound}]
Let $T_i=(X^s_i, Y^s_i)$ for $i=1, \dots, N$, $\mathcal{D}_N=\{T_1, \dots, T_N\}$, and $T=(X^s, Y^s)$ be an independent copy of $T_1$. Firstly, we have
$$
\begin{aligned}
& \mathbb{E}[K^{\mathrm{reg}}(\hat{f}^s_N)-2K^{\mathrm{reg}}_N(\hat{f}^s_N)+K^{\mathrm{reg}}(f_0)] \\
=& \mathbb{E}_{\mathcal{D}_N}[K^{\mathrm{reg}}(\hat{f}^s_N)-2K^{\mathrm{reg}}_N(\hat{f}^s_N)+K^{\mathrm{reg}}(f_0)] \\
=& \mathbb{E}_{\mathcal{D}_N}\{K^{\mathrm{reg}}(\hat{f}^s_N)-K^{\mathrm{reg}}(f_0)-2[K^{\mathrm{reg}}_N(\hat{f}^s_N)-K^{\mathrm{reg}}_N(f_0)]\} \\
=& \mathbb{E}_{\mathcal{D}_N}\left\{\mathbb{E}_{T}[g(\hat{f}^s_N, T)]-\frac 2N\sum_{i=1}^Ng(\hat{f}^s_N, T_i)\right\},
\end{aligned}
$$
where $g(f, T)=\|Y^s-f(X^s)\|_2^2-\|Y^s-f_0(X^s)\|_2^2$ for $f\in\mathcal{F}_{\mathrm{NN}}^{d_y}$. For a $d_y$-dimensional vector $v$, denote its $j$-th component as $v_{(j)}$; additionally, we denote the $j$-th output coordinate of a function $f: \mathbb{R}^{d_x}\to\mathbb{R}^{d_y}$ as $f_{(j)}$, with $j\in\{1, \dots, d_y\}$. Furthermore, for any measurable function $h: \mathbb{R}^{d_x}\to\mathbb{R}$, define
$$
g_j(h, T)=[Y^s_{(j)}-h(X^s)]^2-[Y^s_{(j)}-f_0(X^s)_{(j)}]^2, \quad \text{for } j=1, \dots, d_y.
$$
It is then clear that $g(f, T)=\sum_{j=1}^{d_y}g_j(f_{(j)}, T)$. Hence, we obtain
$$
\begin{aligned}
& \mathbb{E}[K^{\mathrm{reg}}(\hat{f}^s_N)-2K^{\mathrm{reg}}_N(\hat{f}^s_N)+K^{\mathrm{reg}}(f_0)] \\
=& \mathbb{E}_{\mathcal{D}_N}\left\{\mathbb{E}_{T}[g(\hat{f}^s_N, T)]-\frac 2N\sum_{i=1}^Ng(\hat{f}^s_N, T_i)\right\} \\
=& \sum_{j=1}^{d_y}\mathbb{E}_{\mathcal{D}_N}\left\{\mathbb{E}_{T}[g_j(\hat{f}^s_{N, (j)}, T)]-\frac 2N\sum_{i=1}^Ng_j(\hat{f}^s_{N, (j)}, T_i)\right\}.
\end{aligned}
$$
Subsequently, let us fix an arbitrary $j\in\{1, \dots, d_y\}$. For any $\iota_N>0$, we let $U=Y^s_{(j)}\mathds{1}(\|Y^s\|_{\infty}\le \iota_N)$ and $V=\mathbb{E}[Y^s_{(j)}\mathds{1}(\|Y^s\|_{\infty}\le \iota_N)|X^s]$. Then, for any measurable function $h: \mathbb{R}^{d_x}\to\mathbb{R}$, define
$$
g_{j, \iota_N}(h, T)=[U-h(X^s)]^2-(U-V)^2=[V-h(X^s)][2U-h(X^s)-V].
$$
It follows that
$$
\begin{aligned}
& |g_j(f_{(j)}, T)-g_{j, \iota_N}(f_{(j)}, T)| \\
=& \left|[Y^s_{(j)}-f(X^s)_{(j)}]^2-[U-f(X^s)_{(j)}]^2-[Y^s_{(j)}-f_0(X^s)_{(j)}]^2+(U-V)^2\right| \\
\le & \left|(Y^s_{(j)}-U)[Y^s_{(j)}+U-2f(X^s)_{(j)}]\right| \\
& +\left|[Y^s_{(j)}-U-f_0(X^s)_{(j)}+V][Y^s_{(j)}+U-f_0(X^s)_{(j)}-V]\right| \\
\le & \left|Y^s_{(j)}\mathds{1}(\|Y^s\|_{\infty}>\iota_N)[Y^s_{(j)}+U-2f(X^s)_{(j)}]\right| \\
& +\left|Y^s_{(j)}\mathds{1}(\|Y^s\|_{\infty}>\iota_N)[Y^s_{(j)}+U-f_0(X^s)_{(j)}-V]\right| \\
& +\left|\mathbb{E}[Y^s_{(j)}\mathds{1}(\|Y^s\|_{\infty}>\iota_N)|X^s][Y^s_{(j)}+U-f_0(X^s)_{(j)}-V]\right| \\
\le & \left(2\|Y^s\|_{\infty}+\|f_0(X^s)\|_{\infty}+3\iota_N+2\bar{\delta}_N\right)\|Y^s\|_{\infty}\mathds{1}(\|Y^s\|_{\infty}>\iota_N) \\
& +\left(\|Y^s\|_{\infty}+\|f_0(X^s)\|_{\infty}+2\iota_N\right)\mathbb{E}\left[\|Y^s\|_{\infty}\mathds{1}(\|Y^s\|_{\infty}>\iota_N)|X^s\right].
\end{aligned}
$$
Taking expectation, we have
$$
\begin{aligned}
& \mathbb{E}_{\mathcal{D}_N}\left\{\mathbb{E}_{T}[g_j(\hat{f}^s_{N, (j)}, T)]-\frac 2N\sum_{i=1}^Ng_j(\hat{f}^s_{N, (j)}, T_i)\right\} \\
\le & \mathbb{E}_{\mathcal{D}_N}\left\{\mathbb{E}_{T}[g_{j, \iota_N}(\hat{f}^s_{N, (j)}, T)]-\frac 2N\sum_{i=1}^Ng_{j, \iota_N}(\hat{f}^s_{N, (j)}, T_i)\right\} \\
& +3\mathbb{E}_P\left[\left(2\|Y\|_{\infty}+\|f_0(X)\|_{\infty}+3\iota_N+2\bar{\delta}_N\right)\|Y\|_{\infty}\mathds{1}(\|Y\|_{\infty}>\iota_N)\right] \\
& +3\mathbb{E}_P\left\{\left(\|Y\|_{\infty}+\|f_0(X)\|_{\infty}+2\iota_N\right)\mathbb{E}_P\left[\|Y\|_{\infty}\mathds{1}(\|Y\|_{\infty}>\iota_N)|X\right]\right\}.
\end{aligned}
$$
Specifically, observe that
$$
\begin{aligned}
& \mathbb{E}_P\left[\left(2\|Y\|_{\infty}+\|f_0(X)\|_{\infty}+3\iota_N+2\bar{\delta}_N\right)\|Y\|_{\infty}\mathds{1}(\|Y\|_{\infty}>\iota_N)\right] \\
=& 2\mathbb{E}_P\left[\|Y\|_{\infty}^2\mathds{1}(\|Y\|_{\infty}>\iota_N)\right]+\mathbb{E}_P\left[\|f_0(X)\|_{\infty}\|Y\|_{\infty}\mathds{1}(\|Y\|_{\infty}>\iota_N)\right] \\
& +(3\iota_N+2\bar{\delta}_N)\mathbb{E}_P\left[\|Y\|_{\infty}\mathds{1}(\|Y\|_{\infty}>\iota_N)\right] \\
\le & \frac{32}{\varsigma^2}\mathbb{E}_P[\exp(\varsigma\|Y\|_{\infty})]\exp(-\varsigma\iota_N/2)+\frac{4}{\varsigma}\mathbb{E}_P[\|f_0(X)\|_{\infty}\exp(\varsigma\|Y\|_{\infty}/2)]\exp(-\varsigma\iota_N/4) \\
& +\frac{2}{\varsigma}(3\iota_N+2\bar{\delta}_N)\mathbb{E}_P[\exp(\varsigma\|Y\|_{\infty})]\exp(-\varsigma\iota_N/2) \\
\le & \frac{32}{\varsigma^2}\mathbb{E}_P[\exp(\varsigma\|Y\|_{\infty})]\exp(-\varsigma\iota_N/2) \\
& +\frac{4}{\varsigma}\left\{\mathbb{E}_P(\|f_0(X)\|_{\infty}^2)\mathbb{E}_P[\exp(\varsigma\|Y\|_{\infty})]\right\}^{1/2}\exp(-\varsigma\iota_N/4) \\
& +\frac{2}{\varsigma}(3\iota_N+2\bar{\delta}_N)\mathbb{E}_P[\exp(\varsigma\|Y\|_{\infty})]\exp(-\varsigma\iota_N/2),
\end{aligned}
$$
and
$$
\begin{aligned}
& \mathbb{E}_P\left\{\left(\|Y\|_{\infty}+\|f_0(X)\|_{\infty}+2\iota_N\right)\mathbb{E}_P\left[\|Y\|_{\infty}\mathds{1}(\|Y\|_{\infty}>\iota_N)|X\right]\right\} \\
=& \mathbb{E}_P\left\{\|Y\|_{\infty}\mathbb{E}_P\left[\|Y\|_{\infty}\mathds{1}(\|Y\|_{\infty}>\iota_N)|X\right]\right\}+\mathbb{E}_P\left[\|f_0(X)\|_{\infty}\|Y\|_{\infty}\mathds{1}(\|Y\|_{\infty}>\iota_N)\right] \\
& +2\iota_N\mathbb{E}_P\left[\|Y\|_{\infty}\mathds{1}(\|Y\|_{\infty}>\iota_N)\right] \\
\le & \left\{\mathbb{E}_P(\|Y\|_{\infty}^2)\mathbb{E}_P[\|Y\|_{\infty}^2\mathds{1}(\|Y\|_{\infty}>\iota_N)]\right\}^{1/2}+\mathbb{E}_P\left[\|f_0(X)\|_{\infty}\|Y\|_{\infty}\mathds{1}(\|Y\|_{\infty}>\iota_N)\right] \\
& +2\iota_N\mathbb{E}_P\left[\|Y\|_{\infty}\mathds{1}(\|Y\|_{\infty}>\iota_N)\right] \\
\le & \frac{4}{\varsigma}\left\{\mathbb{E}_P(\|Y\|_{\infty}^2)\mathbb{E}_P[\exp(\varsigma\|Y\|_{\infty})]\right\}^{1/2}\exp(-\varsigma\iota_N/4) \\
& +\frac{4}{\varsigma}\left\{\mathbb{E}_P(\|f_0(X)\|_{\infty}^2)\mathbb{E}_P[\exp(\varsigma\|Y\|_{\infty})]\right\}^{1/2}\exp(-\varsigma\iota_N/4) \\
& +\frac{4}{\varsigma}\iota_N\mathbb{E}_P[\exp(\varsigma\|Y\|_{\infty})]\exp(-\varsigma\iota_N/2).
\end{aligned}
$$
Here, we have applied the inequalities $a\le \exp(a)$ and $\mathds{1}(a>0)\le \exp(a)$ for $a\in\mathbb{R}$. Note that $\mathbb{E}_P\|f_0(X)\|_{\infty}^2<\infty$ since $\|Y^s\|_{\infty}$ has a finite second moment and $f_0(X^s)=\mathbb{E}(Y^s|X^s)$. As a consequence, we obtain
$$
\begin{aligned}
& \mathbb{E}_{\mathcal{D}_N}\left\{\mathbb{E}_{T}[g_j(\hat{f}^s_{N, (j)}, T)]-\frac 2N\sum_{i=1}^Ng_j(\hat{f}^s_{N, (j)}, T_i)\right\} \\
\le & \mathbb{E}_{\mathcal{D}_N}\left\{\mathbb{E}_{T}[g_{j, \iota_N}(\hat{f}^s_{N, (j)}, T)]-\frac 2N\sum_{i=1}^Ng_{j, \iota_N}(\hat{f}^s_{N, (j)}, T_i)\right\}+c_1(\iota_N+\bar{\delta}_N+1)\exp(-\varsigma\iota_N/4),
\end{aligned}
$$
where $c_1$ is a constant which depends only on $\varsigma$ and $\mathbb{E}_P[\exp(\varsigma\|Y\|_{\infty})]$. Recall that $\bar{\delta}_N=(\log N)^{1+\kappa}$. Set $\iota_N$ to $(4\varsigma^{-1}\log N)\vee 1$. Then, for sufficiently large $N$ such that $\bar{\delta}_N\ge \iota_N$, it holds that
$$
\begin{aligned}
& \mathbb{E}_{\mathcal{D}_N}\left\{\mathbb{E}_{T}[g_j(\hat{f}^s_{N, (j)}, T)]-\frac 2N\sum_{i=1}^Ng_j(\hat{f}^s_{N, (j)}, T_i)\right\} \\
\le & \mathbb{E}_{\mathcal{D}_N}\left\{\mathbb{E}_{T}[g_{j, \iota_N}(\hat{f}^s_{N, (j)}, T)]-\frac 2N\sum_{i=1}^Ng_{j, \iota_N}(\hat{f}^s_{N, (j)}, T_i)\right\}+3c_1(\log N)^{1+\kappa}N^{-1}.
\end{aligned}
$$
Furthermore, we proceed to verify the conditions in Theorem \ref{thm: gen_thm_11.4_gyorfi}. Notice that
$$
\sup_{f\in\mathcal{F}_{\mathrm{NN}}^{d_y}, T\in\mathbb{R}^{d_x+d_y}}|g_{j, \iota_N}(f_{(j)}, T)|\le 6\iota_N^2+2\bar{\delta}_N^2\le 8\bar{\delta}_N^2=8(\log N)^{2+2\kappa},
$$
whenever $\bar{\delta}_N\ge \iota_N$. Furthermore, for any $f\in\mathcal{F}_{\mathrm{NN}}^{d_y}$,
$$
\begin{aligned}
\mathbb{E}[g_{j, \iota_N}(f_{(j)}, T)] &= \mathbb{E}\left\{[V-f(X^s)_{(j)}][2U-f(X^s)_{(j)}-V]\right\} \\
&= \mathbb{E}[V-f(X^s)_{(j)}]^2,
\end{aligned}
$$
and
$$
\begin{aligned}
\mathbb{E}[g_{j, \iota_N}(f, T)^2] &= \mathbb{E}\left\{[V-f(X^s)_{(j)}]^2[2U-f(X^s)_{(j)}-V]^2\right\} \\
&\le (3\iota_N+\bar{\delta}_N)^2\mathbb{E}[V-f(X^s)_{(j)}]^2 \\
&\le 16\bar{\delta}_N^2\mathbb{E}[g_{j, \iota_N}(f_{(j)}, T)] \\
&= 16(\log N)^{2+2\kappa}\mathbb{E}[g_{j, \iota_N}(f_{(j)}, T)],
\end{aligned}
$$
provided that $N$ is sufficiently large such that $\bar{\delta}_N\ge \iota_N$. Hence, Theorem \ref{thm: gen_thm_11.4_gyorfi} suggests that, for sufficiently large $N$ such that $\bar{\delta}_N\ge \iota_N$, with $N\ge \mathrm{Pdim}(\mathcal{F}_{\mathrm{NN}})$, and for arbitrary $t>0$, we have
$$
\begin{aligned}
& \mathbb{P}_{\mathcal{D}_N}\left\{\mathbb{E}_{T}[g_{j, \iota_N}(\hat{f}^s_{N, (j)}, T)]-\frac 2N\sum_{i=1}^Ng_{j, \iota_N}(\hat{f}^s_{N, (j)}, T_i)\ge t\right\} \\
\le & \mathbb{P}_{\mathcal{D}_N}\left\{\mathbb{E}_{T}[g_{j, \iota_N}(\hat{f}^s_{N, (j)}, T)]-\frac 1N\sum_{i=1}^Ng_{j, \iota_N}(\hat{f}^s_{N, (j)}, T_i)\ge \frac 12\left\{\frac t2+\frac t2+\mathbb{E}_{T}[g_{j, \iota_N}(\hat{f}^s_{N, (j)}, T)]\right\}\right\} \\
\le & \mathbb{P}_{\mathcal{D}_N}\left(\exists h\in\mathcal{F}_{\mathrm{NN}}: \mathbb{E}[g_{j, \iota_N}(h, T)]-\frac 1N\sum_{i=1}^Ng_{j, \iota_N}(h, T_i)\ge \frac 12\left\{\frac t2+\frac t2+\mathbb{E}_{T}[g_{j, \iota_N}(h, T)]\right\}\right) \\
\le & 14\mathcal{N}_N\left(c_2t, \|\cdot \|_{\infty}, \{g_{j, \iota_N}(h, \cdot): \mathbb{R}^{d_x}\times \mathbb{R}^{d_y}\to\mathbb{R}, h\in\mathcal{F}_{\mathrm{NN}}\}\right)\exp\left(-\frac{Nt}{c_3(\log N)^{4+4\kappa}}\right),
\end{aligned}
$$
where $c_2, c_3$ are universal constants. Subsequently, we bound the covering number. Fix $\{x_1, \dots, x_N\}\subset (\mathcal{R}^{d_x})^N$ and $\{y_1, \dots, y_N\}\subset (\mathcal{R}^{d_y})^N$. Let $\mathcal{C}=\{x_1, \dots, x_N\}$, and let $w^{\sharp}=\{w_1, \dots, w_k\}$ be an $\epsilon$-covering set of $\mathcal{F}_{\mathrm{NN}|\mathcal{C}}$ where $w_i=h_{i|\mathcal{C}}$ for some $h_i\in\mathcal{F}_{\mathrm{NN}}\ (i=1, \dots, k)$, such that for any $h\in \mathcal{F}_{\mathrm{NN}}$, there exists $w^*=h^*_{|\mathcal{C}}\in w^{\sharp}$ satisfying $\|w^*-h_{|\mathcal{C}}\|_{\infty}<\epsilon$. This indicates
$$
\begin{aligned}
& |g_{j, \iota_N}(h, (x_i, y_i))-g_{j, \iota_N}(h^*, (x_i, y_i))| \\
\le & |h^*(x_i)-h(x_i)|\cdot\left|2y_{i, (j)}\mathds{1}(\|y_i\|_{\infty}\le \iota_N)-h^*(x_i)-h(x_i)\right| \\
\le & 2(\iota_N+\bar{\delta}_N)\epsilon \\
\le & 4\bar{\delta}_N\epsilon,
\end{aligned}
$$
whenever $\bar{\delta}_N\ge \iota_N$. Therefore,
$$
\mathcal{N}_N(c_2t, \|\cdot \|_{\infty}, \{g_{j, \iota_N}(h, \cdot): \mathbb{R}^{d_x}\times \mathbb{R}^{d_y}\to\mathbb{R}, h\in\mathcal{F}_{\mathrm{NN}}\})\le \mathcal{N}_N(c_2t/(4\bar{\delta}_N), \|\cdot \|_{\infty}, \mathcal{F}_{\mathrm{NN}}).
$$
Then, with Lemma \ref{lem: thm_12.2_anthony} and Lemma \ref{lem: thm_7_bartlett}, for sufficiently large $N$ with $N\ge \mathrm{Pdim}(\mathcal{F}_{\mathrm{NN}})$ and any $a_N\ge 1/N$, we have
$$
\begin{aligned}
& \mathbb{E}_{\mathcal{D}_N}\left\{\mathbb{E}_{T}[g_{j, \iota_N}(\hat{f}^s_{N, (j)}, T)]-\frac 2N\sum_{i=1}^Ng_{j, \iota_N}(\hat{f}^s_{N, (j)}, T_i)\right\} \\
\le & a_N+14\int_{a_N}^{\infty}\mathcal{N}_N(c_2t/(4\bar{\delta}_N), \|\cdot \|_{\infty}, \mathcal{F}_{\mathrm{NN}})\exp\left(-\frac{Nt}{c_3(\log N)^{4+4\kappa}}\right)\mathrm{d}t \\
\le & a_N+14\mathcal{N}_N(c_2a_N/(4\bar{\delta}_N), \|\cdot \|_{\infty}, \mathcal{F}_{\mathrm{NN}})\int_{a_N}^{\infty}\exp\left(-\frac{Nt}{c_3(\log N)^{4+4\kappa}}\right)\mathrm{d}t \\
\le & a_N+14\left(c_4N^2\bar{\delta}_N^2\right)^{c_5SL\log S}\cdot \frac{c_3(\log N)^{4+4\kappa}}{N}\exp\left(-\frac{Na_N}{c_3(\log N)^{4+4\kappa}}\right),
\end{aligned}
$$
where $c_4$ and $c_5$ are universal constants. Choose
$$
a_N=\frac{c_3c_5(\log N)^{4+4\kappa}}{N}SL\log S\log\left(c_4N^2\bar{\delta}_N^2\right).
$$
For sufficiently large $N$, we have
$$
\mathbb{E}_{\mathcal{D}_N}\left\{\mathbb{E}_{T}[g_{j, \iota_N}(\hat{f}^s_{N, (j)}, T)]-\frac 2N\sum_{i=1}^Ng_{j, \iota_N}(\hat{f}^s_{N, (j)}, T_i)\right\}\le \frac{c_6SL\log S(\log N)^{5+4\kappa}}{N},
$$
where $c_6$ is a constant not depending on $S, L$ and $N$. Noticing the arbitrariness of $j$, we complete the proof.
\end{proof}

\begin{lemma}\label{lem: reg_approximation_error_bound}
Assume that
\begin{enumerate}[label=(\roman*)]
    \item $e_j^{\top}f_0\in \mathcal{H}^{\beta_f}_{\mathrm{Loc}}(\mathbb{R}^d, B_u)$ with $\beta_f>0$ and $B_u\le c(u^m+1)$ for some universal constants $c>0$, $m\ge 0$, and for any $j\in\{1, \dots, d_y\}$, where $e_j$ denotes a $d_y$-dimensional one-hot vector with the $j$-th component equal to 1 and all other components equal to 0;
    \item $\|f_0(X^s)\|_{\infty}$ and $\|X^s\|_{\infty}$ are sub-exponentially distributed random variables.
\end{enumerate}
Suppose that the depth $L$ and width $M$ of $\mathcal{F}_{\mathrm{NN}}^{d_y}$ are expressed as
$$
\begin{aligned}
L &= 21(\lfloor\beta_f\rfloor+1)^2S_1\lceil\log_2(8S_1)\rceil+2d_x+3, \\
M &= 38d_y(\lfloor\beta_f\rfloor+1)^2d_x^{\lfloor\beta_f\rfloor+1}S_2\lceil\log_2(8S_2)\rceil,
\end{aligned}
$$
for any $S_1, S_2\in\mathbb{N}_+$. Let $\bar{\delta}=\bar{\delta}_N=(\log N)^{1+\kappa}$, with an arbitrarily fixed $\kappa\in (0, 1]$, and let $\underline{\delta}=\underline{\delta}_N=-(\log N)^{-1-\kappa}$. Then, for sufficiently large $N$, it follows that
$$
\begin{aligned}
& \inf_{f\in\mathcal{F}_{\mathrm{NN}}^{d_y}}\mathbb{E}_P\|f(X)-f_0(X)\|_2^2 \\
\le & c^*\left\{\left[(\lfloor\beta_f\rfloor+1)^2d_x^{\lfloor\beta_f\rfloor+(\beta_f\vee 1)/2}(S_1S_2)^{-2\beta_f/d_x}(\log N)^m\right]^2+\frac{(\log N)^{2+2\kappa}}{N}\right\},
\end{aligned}
$$
where $c^*$ is a constant not depending on $S_1, S_2$ and $N$.
\end{lemma}

\begin{proof}[Proof of Lemma \ref{lem: reg_approximation_error_bound}]
For any $\iota_N>0$, observe that
$$
\begin{aligned}
& \mathbb{E}_P\|f(X)-f_0(X)\|_2^2\\
=& \mathbb{E}_P\left\{\|f(X)-f_0(X)\|_2^2\mathds{1}(\|X\|_{\infty}\le \iota_N)\right\}+\mathbb{E}_P\left\{\|f(X)-f_0(X)\|_2^2\mathds{1}(\|X\|_{\infty}>\iota_N)\right\}.
\end{aligned}
$$
For clarity, denote the $j$-th output coordinate of a function $f: \mathbb{R}^{d_x}\to\mathbb{R}^{d_y}$ as $f_{(j)}$, with $j\in\{1, \dots, d_y\}$. On the one hand, it follows that
$$
\begin{aligned}
& \mathbb{E}_P\left\{\|f(X)-f_0(X)\|_2^2\mathds{1}(\|X\|_{\infty}>\iota_N)\right\} \\
=& \sum_{j=1}^{d_y}\mathbb{E}_P\left\{[f_{(j)}(X)-f_{0, (j)}(X)]^2\mathds{1}(\|X\|_{\infty}>\iota_N)\right\} \\
\le & 2d_y\bar{\delta}_N^2\mathbb{E}_P[\mathds{1}(\|X\|_{\infty}>\iota_N)]+2d_y\mathbb{E}_P[\|f_0(X)\|_{\infty}^2\mathds{1}(\|X\|_{\infty}>\iota_N)] \\
\le & 2d_y\bar{\delta}_N^2\mathbb{E}_P[\exp(\varsigma\|X\|_{\infty}/2)]\exp(-\varsigma\iota_N/2) \\
& +2d_y\mathbb{E}_P[\|f_0(X)\|_{\infty}^2\exp(\varsigma\|X\|_{\infty}/2)]\exp(-\varsigma\iota_N/2) \\
\le & 2d_y\bar{\delta}_N^2\mathbb{E}_P[\exp(\varsigma\|X\|_{\infty}/2)]\exp(-\varsigma\iota_N/2) \\
& +2d_y\left\{\mathbb{E}_P[\|f_0(X)\|_{\infty}^4]\mathbb{E}_P[\exp(\varsigma\|X\|_{\infty})]\right\}^{1/2}\exp(-\varsigma\iota_N/2) \\
\le & c_1(\bar{\delta}_N^2+1)\exp(-\varsigma\iota_N/2),
\end{aligned}
$$
where $c_1$ is a constant which depends only on $d_y, \varsigma, \mathbb{E}_P[\exp(\varsigma\|X\|_{\infty})]$ and $\mathbb{E}_P[\|f_0(X)\|_{\infty}^4]$. On the other hand, we first notice that
$$
\begin{aligned}
& \mathbb{E}_P\left\{\|f(X)-f_0(X)\|_2^2\mathds{1}(\|X\|_{\infty}\le \iota_N)\right\} \\
=& \mathbb{E}_P\left\{\|f(X)-f_0(X)\|_2^2\mathds{1}(\|f_0(X)\|_{\infty}\le \iota_N)\mathds{1}(\|X\|_{\infty}\le \iota_N)\right\} \\
& +\mathbb{E}_P\left\{\|f(X)-f_0(X)\|_2^2\mathds{1}(\|f_0(X)\|_{\infty}>\iota_N)\mathds{1}(\|X\|_{\infty}\le \iota_N)\right\} \\
\le & \mathbb{E}_P\left\{\|f(X)-f_0(X)\|_2^2\mathds{1}(\|f_0(X)\|_{\infty}\le \iota_N)\mathds{1}(\|X\|_{\infty}\le \iota_N)\right\} \\
& +2d_y\bar{\delta}_N^2\mathbb{E}_P\left[\mathds{1}(\|f_0(X)\|_{\infty}>\iota_N)\right]+2d_y\mathbb{E}_P\left[\|f_0(X)\|_{\infty}^2\mathds{1}(\|f_0(X)\|_{\infty}>\iota_N)\right] \\
\le & \mathbb{E}_P\left\{\|f(X)-f_0(X)\|_2^2\mathds{1}(\|f_0(X)\|_{\infty}\le \iota_N)\mathds{1}(\|X\|_{\infty}\le \iota_N)\right\} \\
& +2d_y\bar{\delta}_N^2\mathbb{E}_P[\exp(\varsigma \|f_0(X)\|_{\infty}/2)]\exp(-\varsigma\iota_N/2) \\
& +\frac{32}{\varsigma^2}d_y\mathbb{E}_P[\exp(\varsigma \|f_0(X)\|_{\infty})]\exp(-\varsigma\iota_N/2) \\
\le & \mathbb{E}_P\left\{\|f(X)-f_0(X)\|_2^2\mathds{1}(\|f_0(X)\|_{\infty}\le \iota_N)\mathds{1}(\|X\|_{\infty}\le \iota_N)\right\}+c_2(\bar{\delta}_N^2+1)\exp(-\varsigma\iota_N/2),
\end{aligned}
$$
where $c_2$ is a constant which depends only on $d_y, \varsigma$ and $\mathbb{E}_P[\exp(\varsigma \|f_0(X)\|_{\infty})]$. Then, we focus on $\{x: \|x\|_{\infty}\le \iota_N\}=[-\iota_N, \iota_N]^{d_x}$. Fix an arbitrary $j\in \{1, \dots, d_y\}$. Let $h_j(x)=f_{0, (j)}(2\iota_Nx-\iota_N\mathrm{1}_{d_x})$ for $x\in [0, 1]^{d_x}$. Lemma \ref{lem: thm_3.3_jiao} demonstrates that for any $S_1, S_2\in\mathbb{N}_+$, there exists a function $h_j^*$ implemented by a ReLU network with depth $L^*=21(\lfloor\beta_f\rfloor+1)^2S_1\lceil\log_2(8S_1)\rceil+2d_x$, width $M^*=38(\lfloor\beta_f\rfloor+1)^2d_x^{\lfloor\beta_f\rfloor+1}S_2\lceil\log_2(8S_2)\rceil$, such that
$$
|h_j^*(x)-h_j(x)|\le 18c(\iota_N^m+1)(\lfloor\beta_f\rfloor+1)^2d_x^{\lfloor\beta_f\rfloor+(\beta_f\vee 1)/2}(S_1S_2)^{-2\beta_f/d_x},
$$
for all $x\in [0, 1]^{d_x}\backslash \Omega([0, 1]^{d_x}, K, \Delta)$. Here,
$$
\Omega([0, 1]^{d_x}, K, \Delta)=\bigcup_{i=1}^{d_x}\left\{x=(x_1, \dots, x_{d_x})^{\top}: x_i\in\bigcup_{k=1}^{K-1}(k/K-\Delta, k/K)\right\},
$$
where $K=\lceil(S_1S_2)^{2/d_x}\rceil$ and $\Delta$ is an arbitrary scalar in $(0, 1/(3K)]$. Let $h_j^{\dagger}(x)=h_j^*((x+\iota_N\mathrm{1}_{d_x})/(2\iota_N))$ for $x\in[-\iota_N, \iota_N]^{d_x}$. We obtain that
$$
|h_j^{\dagger}(x)-f_{0, (j)}(x)|\le 18c(\iota_N^m+1)(\lfloor\beta_f\rfloor+1)^2d_x^{\lfloor\beta_f\rfloor+(\beta_f\vee 1)/2}(S_1S_2)^{-2\beta_f/d_x},
$$
for all $x\in [-\iota_N, \iota_N]^{d_x}\backslash \Omega^{\dagger}$, where $\Omega^{\dagger}=\{x: (x+\iota_N\mathrm{1}_{d_x})/(2\iota_N)\in \Omega([0, 1]^{d_x}, K, \Delta)\}$. Furthermore, note that
$$
h_j^{\dagger}(x)=h_j^*\left(\frac{x+\iota_N\mathrm{1}_{d_x}}{2\iota_N}\right)=h_j^*\left(\mathrm{relu}\left(\frac{x+\iota_N\mathrm{1}_{d_x}}{2\iota_N}\right)-\mathrm{relu}\left(-\frac{x+\iota_N\mathrm{1}_{d_x}}{2\iota_N}\right)\right),
$$
which is implemented by a neural network with ReLU activations, depth $L^{\dagger}=L^*+1$, and width $M^{\dagger}=M^*$. In addition, let
$$
h_j^{\ddagger}(x)=
\begin{cases}
    \bar{\delta}_N, & h_j^{\dagger}(x)>\bar{\delta}_N, \\
    h_j^{\dagger}(x), & \underline{\delta}_N\le h_j^{\dagger}(x)\le \bar{\delta}_N, \\
    \underline{\delta}_N, & h_j^{\dagger}(x)<\underline{\delta}_N.
\end{cases}
$$
A straightforward calculation shows that
$$
h_j^{\ddagger}(x)=\mathrm{relu}(-\mathrm{relu}(-h_j^{\dagger}(x)+\bar{\delta}_N)+\bar{\delta}_N)-\mathrm{relu}(-\mathrm{relu}(h_j^{\dagger}(x)-\underline{\delta}_N)-\underline{\delta}_N),
$$
indicating that $h_j^{\ddagger}(x)$ can be implemented by a ReLU network with depth $L^{\ddagger}=L^*+3$ and width $M^{\ddagger}=M^*$. Due to the arbitrariness of $\Delta$, when $\bar{\delta}_N\ge \iota_N$, it follows that
$$
\begin{aligned}
& \mathbb{E}_P\left\{[h_j^{\ddagger}(X)-f_{0, (j)}(X)]^2\mathds{1}(\|f_0(X)\|_{\infty}\le \iota_N)\mathds{1}(\|X\|_{\infty}\le \iota_N)\right\} \\
\le & \left[18c(\iota_N^m+1)(\lfloor\beta_f\rfloor+1)^2d_x^{\lfloor\beta_f\rfloor+(\beta_f\vee 1)/2}(S_1S_2)^{-2\beta_f/d_x}\right]^2.
\end{aligned}
$$
Let $f^{\ddagger}(x)=(h_1^{\ddagger}(x), \dots, h_{d_y}^{\ddagger}(x))^{\top}$. It is straightforward to verify that $f^{\ddagger}$ can be implemented by a ReLU network in $\mathcal{F}_{\mathrm{NN}}^{d_y}$ with depth $L=L^*+3$ and width $M=d_yM^*$. Hence, we have
$$
\begin{aligned}
& \inf_{f\in\mathcal{F}_{\mathrm{NN}}^{d_y}}\mathbb{E}_P\left\{\|f(X)-f_0(X)\|_2^2\mathds{1}(\|f_0(X)\|_{\infty}\le \iota_N)\mathds{1}(\|X\|_{\infty}\le \iota_N)\right\} \\
\le & \mathbb{E}_P\left\{\|f^{\ddagger}(X)-f_0(X)\|_2^2\mathds{1}(\|f_0(X)\|_{\infty}\le \iota_N)\mathds{1}(\|X\|_{\infty}\le \iota_N)\right\} \\
=& \sum_{j=1}^{d_y}\mathbb{E}_P\left\{[h_j^{\ddagger}(X)-f_{0, (j)}(X)]^2\mathds{1}(\|f_0(X)\|_{\infty}\le \iota_N)\mathds{1}(\|X\|_{\infty}\le \iota_N)\right\} \\
\le & d_y\left[18c(\iota_N^m+1)(\lfloor\beta_f\rfloor+1)^2d_x^{\lfloor\beta_f\rfloor+(\beta_f\vee 1)/2}(S_1S_2)^{-2\beta_f/d_x}\right]^2,
\end{aligned}
$$
provided that $\bar{\delta}_N\ge \iota_N$. Recall that $\bar{\delta}_N=(\log N)^{1+\kappa}$, and set $\iota_N$ to $(2\varsigma^{-1}\log N)\vee 1$. We conclude that for sufficiently large $N$ satisfying $\bar{\delta}_N\ge \iota_N$, it holds that
$$
\begin{aligned}
& \inf_{f\in\mathcal{F}_{\mathrm{NN}}^{d_y}}\mathbb{E}_P\|f(X)-f_0(X)\|_2^2 \\
\le & \inf_{f\in \mathcal{F}_{\mathrm{NN}}^{d_y}}\mathbb{E}_P\left\{\|f(X)-f_0(X)\|_2^2\mathds{1}(\|f_0(X)\|_{\infty}\le \iota_N)\mathds{1}(\|X\|_{\infty}\le \iota_N)\right\} \\
& +(c_1+c_2)(\bar{\delta}_N^2+1)\exp(-\varsigma\iota_N/2) \\
\le & c_3\left\{\left[(\lfloor\beta_f\rfloor+1)^2d_x^{\lfloor\beta_f\rfloor+(\beta_f\vee 1)/2}(S_1S_2)^{-2\beta_f/d_x}(\log N)^m\right]^2+\frac{(\log N)^{2+2\kappa}}{N}\right\},
\end{aligned}
$$
where $c_3$ is a constant not depending on $S_1, S_2$ and $N$.
\end{proof}

\begin{proof}[Proof of Theorem 5.1]
To commence, we notice that $\|f_0(X^s)\|_{\infty}$ is sub-exponentially distributed provided that $\|Y^s\|_{\infty}$ is a sub-exponential random variable (consider Jensen's inequality). Then, Lemma \ref{lem: reg_error_decomposition}, Lemma \ref{lem: reg_stochatic_error_bound} and Lemma \ref{lem: reg_approximation_error_bound} indicate
$$
\begin{aligned}
& \mathbb{E}\|\hat{f}^s_N(X^s)-f_0(X^s)\|_2^2 \\
\le & \frac{c_1SL\log S(\log N)^{5+4\kappa}}{N} \\
& +c_2\left\{\left[(\lfloor\beta_f\rfloor+1)^2d_x^{\lfloor\beta_f\rfloor+(\beta_f\vee 1)/2}(S_1S_2)^{-2\beta_f/d_x}(\log N)^m\right]^2+\frac{(\log N)^{2+2\kappa}}{N}\right\},
\end{aligned}
$$
where $c_1, c_2$ are constants not depending on $S, L, S_1, S_2$ and $N$, and $S_1, S_2$ satisfy the conditions that the network depth $L=21(\lfloor\beta_f\rfloor+1)^2S_1\lceil\log_2(8S_1)\rceil+2d_x+3$, network width $M=38d_y(\lfloor\beta_f\rfloor+1)^2d_x^{\lfloor\beta_f\rfloor+1}S_2\lceil\log_2(8S_2)\rceil$, for sufficiently large $N$ and $N\ge \mathrm{Pdim}(\mathcal{F}_{\mathrm{NN}})$. Therefore, by letting $S_1=\mathcal{O}(N^{d_x/(2d_x+4\beta_f)})$ and $S_2=\mathcal{O}(1)$,  we obtain
$$
M=\mathcal{O}(1), \quad L=\mathcal{O}\left(N^{\frac{d_x}{2d_x+4\beta_f}}\log N\right), \quad S=\mathcal{O}(M^2L)=\mathcal{O}\left(N^{\frac{d_x}{2d_x+4\beta_f}}\log N\right),
$$
yielding
$$
\mathbb{E}\|\hat{f}^s_N(X^s)-f_0(X^s)\|_2^2\le c_3N^{-\frac{2\beta_f}{d_x+2\beta_f}}(\log N)^{(8+4\kappa)\vee (2m)},
$$
where $c_3$ is a constant not depending on $N$, for $N\ge 2$. Furthermore, observe that
$$
\begin{aligned}
\mathbb{E}\|\hat{f}^s_N(X^t)-f_0(X^t)\|_2^2=\mathbb{E}\left\{\|\hat{f}^s_N(X^s)-f_0(X^s)\|_2^2\cdot r_0(X^s)\right\}.
\end{aligned}
$$
As $\|\hat{f}^s_N(X^s)\|_{\infty}$ is bounded by $\bar{\delta}_N=(\log N)^{1+\kappa}$, $\|f_0(X^s)\|_{\infty}$ is a sub-exponential random variable (hence it possesses a finite fourth moment), and $r_0(X^s)$ is presumed to be sub-exponentially distributed, by Corollary 4.2, we conclude that
$$
\mathbb{E}\|\hat{f}^s_N(X^t)-f_0(X^t)\|_2^2\le c_4N^{-\frac{2\beta_f}{d_x+2\beta_f}}(\log N)^{(8+4\kappa)\vee (2m)+1}+\frac{c_5d_y[(\log N)^{2+2\kappa}+1]}{N},
$$
for $N\ge 2$, where $c_4, c_5$ are constants which do not depend on $N$. This completes the proof.
\end{proof}

\subsection{Proof of Lemma 5.2}\label{subsec: proof_of_lem_wasserstein_convergence}

In this subsection, we abbreviate $\mathbb{E}[h(X^s, Y^s)]$ to $\mathbb{E}_P[h(X, Y)]$ for any $(X^s, Y^s)$-integrable function $h$, whenever the expectation is taken with respect to $(X^s, Y^s)$.

\begin{proof}[Proof of Lemma 5.2]
We first observe that
$$
\begin{aligned}
\mathcal{E}^t &= \int \mathbb{E}\left[W_2^2(\rho_{0, x}\|\hat{\rho}^s_x)\right]q(x)\mathrm{d}x \\
&= \int \mathbb{E}\left[W_2^2(\rho_{0, x}\|\hat{\rho}^s_x)\right]p(x)r_0(x)\mathrm{d}x \\
&= \mathbb{E}\left[W_2^2(\rho_{0, X^s}\|\hat{\rho}^s_{X^s})r_0(X^s)\right].
\end{aligned}
$$
For any $\iota_N>0$, it follows that
$$
\begin{aligned}
\mathcal{E}^t =& \mathbb{E}\left[W_2^2(\rho_{0, X^s}\|\hat{\rho}^s_{X^s})r_0(X^s)\right] \\
=& \mathbb{E}\left[W_2^2(\rho_{0, X^s}\|\hat{\rho}^s_{X^s})r_0(X^s)\mathds{1}(r_0(X^s)\le \iota_N)\right] \\
& +\mathbb{E}\left[W_2^2(\rho_{0, X^s}\|\hat{\rho}^s_{X^s})r_0(X^s)\mathds{1}(r_0(X^s)>\iota_N)\right] \\
\le & \iota_N\mathcal{E}^s+\mathbb{E}\left[W_2^2(\rho_{0, X^s}\|\hat{\rho}^s_{X^s})r_0(X^s)\mathds{1}(r_0(X^s)>\iota_N)\right].
\end{aligned}
$$
Fix any $x\in\mathcal{X}^s$. Consider the following two ODEs
$$
\begin{aligned}
\mathrm{d}Z_{\tau} &= v_0(x, Z_{\tau}, \tau)\mathrm{d}\tau, \quad Z_0\sim N(0, I_{d_y}), \\
\mathrm{d}\hat{Z}_{\tau} &= \hat{v}^s_N(x, \hat{Z}_{\tau}, \tau)\mathrm{d}\tau, \quad \hat{Z}_0\sim N(0, I_{d_y}).
\end{aligned}
$$
We denote the particles at time $\tau\in [0, 1]$ as $Z_{\tau}(x, z)$ and $\hat{Z}_{\tau}(x, z)$ given $Z_0=z$ and $\hat{Z}_0=z$, respectively. Then, it follows that
$$
\begin{aligned}
W_2^2(\rho_{0, x}\|\hat{\rho}^s_x) \le & \int \left\|Z_1(x, z)-\hat{Z}_1(x, z)\right\|_2^2\cdot\lambda(z)\mathrm{d}z \\
\le & 2\int\left\|Z_1(x, z)\right\|_2^2\cdot \lambda(z)\mathrm{d}z+2\int\left\|\hat{Z}_1(x, z)\right\|_2^2\cdot \lambda(z)\mathrm{d}z \\
=& 2\mathbb{E}\left(\|Y^s\|_2^2\big|X^s=x\right)+2\int\left\|\hat{Z}_1(x, z)\right\|_2^2\cdot \lambda(z)\mathrm{d}z,
\end{aligned}
$$
where $\lambda(\cdot)$ represents the density function of $d_y$-dimensional standard Gaussian distribution. Furthermore, define
$$
\begin{aligned}
H_1(x) &= \mathbb{E}\left(\|Y^s\|_2^2\big|X^s=x\right), \\
H_2(x, \tau) &= \int\left\|\hat{Z}_{\tau}(x, z)\right\|_2^2\cdot \lambda(z)\mathrm{d}z, \quad \text{for } \tau\in [0, 1].
\end{aligned}
$$
Observe that
$$
\begin{aligned}
\frac{\mathrm{d}}{\mathrm{d}\tau}H_2(x, \tau) &= 2\int\left\langle \hat{v}^s_N(x, \hat{Z}_{\tau}(x, z), \tau), \hat{Z}_{\tau}(x, z)\right\rangle\cdot \lambda(z)\mathrm{d}z \\
&\le \int\left\|\hat{v}^s_N(x, \hat{Z}_{\tau}(x, z), \tau)\right\|_2^2\cdot \lambda(z)\mathrm{d}z+\int\left\|\hat{Z}_{\tau}(x, z)\right\|_2^2\cdot \lambda(z)\mathrm{d}z \\
&\le d_y\max(\bar{\delta}^2, \underline{\delta}^2)+H_2(x, \tau).
\end{aligned}
$$
Let $\delta^*=\max(\bar{\delta}^2, \underline{\delta}^2)$. By Lemma \ref{lem: grownwall_inequality}, we obtain
$$
\int\left\|\hat{Z}_1(x, z)\right\|_2^2\cdot \lambda(z)\mathrm{d}z=H_2(x, 1)\le d_ye(\delta^*+1).
$$
Hence, we conclude that $W_2^2(\rho_{0, x}\|\hat{\rho}^s_x)\le 2H_1(x)+2d_ye(\delta^*+1)$ for any $x\in\mathcal{X}^s$. As a consequence, it holds that
$$
\begin{aligned}
\mathcal{E}^t \le & \iota_N\mathcal{E}^s+\mathbb{E}\left[W_2^2(\rho_{0, X^s}\|\hat{\rho}^s_{X^s})r_0(X^s)\mathds{1}(r_0(X^s)>\iota_N)\right] \\
\le & \iota_N\mathcal{E}^s+2\mathbb{E}_P[H_1(X)r_0(X)\mathds{1}(r_0(X)>\iota_N)]+2d_ye(\delta^*+1)\mathbb{E}_P[r_0(X^s)\mathds{1}(r_0(X^s)>\iota_N)] \\
\le & \iota_N\mathcal{E}^s+\frac{8}{\varsigma}\mathbb{E}_P[H_1(X)\exp(\varsigma r_0(X)/2)]\exp(-\varsigma\iota_N/4) \\
& +\frac{4d_ye}{\varsigma}(\delta^*+1)\mathbb{E}_P[\exp(\varsigma r_0(X))]\exp(-\varsigma\iota_N/2) \\
\le & \iota_N\mathcal{E}^s+\frac{8}{\varsigma}\left\{\mathbb{E}_P[H_1(X)^2]\mathbb{E}_P[\exp(\varsigma r_0(X))]\right\}^{1/2}\exp(-\varsigma\iota_N/4) \\
& +\frac{4d_ye}{\varsigma}(\delta^*+1)\mathbb{E}_P[\exp(\varsigma r_0(X))]\exp(-\varsigma\iota_N/2).
\end{aligned}
$$
Here, $H_1(X^s)$ attains a finite second moment since $\|Y^s\|_2$ is presumed to have a finite fourth moment. Set $\iota_N$ to $(4\varsigma^{-1}\log N)\vee 1$. We have
$$
\begin{aligned}
\mathcal{E}^t \le & \iota_N\mathcal{E}^s+\frac{8}{\varsigma}\left\{\mathbb{E}_P[H_1(X)^2]\mathbb{E}_P[\exp(\varsigma r_0(X))]\right\}^{1/2}\exp(-\varsigma\iota_N/4) \\
& +\frac{4d_ye}{\varsigma}(\delta^*+1)\mathbb{E}_P[\exp(\varsigma r_0(X))]\exp(-\varsigma\iota_N/2) \\
\le & \left[(4\varsigma^{-1}\log N)\vee 1\right]\mathcal{E}^s+\frac{8}{\varsigma}\left\{\mathbb{E}_P[H_1(X)^2]\mathbb{E}_P[\exp(\varsigma r_0(X))]\right\}^{1/2}N^{-1} \\
& +\frac{4d_ye}{\varsigma}(\delta^*+1)\mathbb{E}_P[\exp(\varsigma r_0(X))]N^{-2}.
\end{aligned}
$$
This completes the proof.
\end{proof}

\subsection{Proof of Theorem 5.4}\label{subsec: proof_of_thm_gen_convergence_rate}

For any function $f\in\mathcal{L}^2(X^s, Y^s)$, define
$$
\begin{aligned}
K^{\mathrm{gen}}(f) &= \int_0^1\mathbb{E}_P\left\|\dot{a}_{\tau}\eta+\dot{b}_{\tau}Y-f(X, Y_{\tau}, \tau)\right\|_2^2\mathrm{d}\tau, \\
K^{\mathrm{gen}}_N(f) &= \frac{1}{N}\sum_{i=1}^N\left\|\dot{a}_{\tau_i}\eta_i+\dot{b}_{\tau_i}Y^s_i-f(X^s_i, Y^s_{i, \tau_i}, \tau_i)\right\|_2^2.
\end{aligned}
$$
Here, $\mathbb{E}_P[h(X, Y, \eta, Y_{\tau})]\equiv \mathbb{E}[h(X^s, Y^s, \eta, Y^s_{\tau})]$ for any $(X^s, Y^s, \eta)$-integrable function $h$, where the expectation is taken with respect to $(X^s, Y^s, \eta)$for any nonrandom $\tau\in [0, 1]$.

\begin{lemma}\label{lem: gen_error_decomposition}
Assume that $\|Y^s\|_{\infty}$ attains a finite second moment. Then,
$$
\begin{aligned}
& \int_0^1\mathbb{E}\|\hat{v}^s_N(X^s, Y^s_{\tau}, \tau)-v_0(X^s, Y^s_{\tau}, \tau)\|_2^2\mathrm{d}\tau \\
\le & \mathbb{E}\left[K^{\mathrm{gen}}(\hat{v}^s_N)-2K^{\mathrm{gen}}_N(\hat{v}^s_N)+K^{\mathrm{gen}}(v_0)\right] \\
& +2\inf_{f\in\mathcal{F}_{\mathrm{NN}, \Lambda}^{d_y}}\int_0^1\mathbb{E}_P\|f(X, Y_{\tau}, \tau)-v_0(X, Y_{\tau}, \tau)\|_2^2\mathrm{d}\tau.
\end{aligned}
$$
\end{lemma}

\begin{proof}[Proof of Lemma \ref{lem: gen_error_decomposition}]
Recall that $Y^s_{\tau}=a_{\tau}\eta+b_{\tau}Y^s$ and $v_0(X^s, Y^s_{\tau}, \tau)=\mathbb{E}(\dot{a}_{\tau}\eta+\dot{b}_{\tau}Y^s|Y^s_{\tau}, X^s)$. For any $\tau\in [0, 1]$, we have
$$
\mathbb{E}\|v_0(X^s, Y^s_{\tau}, \tau)\|_2^2 \le \mathbb{E}_P\|\dot{a}_{\tau}\eta+\dot{b}_{\tau}Y\|_2^2 \le 2\dot{a}_{\tau}^2\mathbb{E}\|\eta\|_2^2+2\dot{b}_{\tau}^2\mathbb{E}_P\|Y\|_2^2.
$$
Given that $\|Y^s\|_{\infty}$ has a finite second moment, we obtain $\mathbb{E}_P\|Y\|_2^2<\infty$ and hence $v_0\in\mathcal{L}^2(X^s, Y^s)$. Furthermore, as $a_{\tau}$ and $b_{\tau}$ are presumed to be continuously differentiable, notice that
$$
\int_0^1\mathbb{E}\|v_0(X^s, Y^s_{\tau}, \tau)\|_2^2\mathrm{d}\tau\le \int_0^1\left(2\dot{a}_{\tau}^2\mathbb{E}\|\eta\|_2^2+2\dot{b}_{\tau}^2\mathbb{E}_P\|Y\|_2^2\right)\mathrm{d}\tau<\infty.
$$
Next, for any $f\in\mathcal{F}_{\mathrm{NN}, \Lambda}^{d_y}$, it follows that
$$
\begin{aligned}
& \int_0^1\mathbb{E}\|\hat{v}^s_N(X^s, Y^s_{\tau}, \tau)-v_0(X^s, Y^s_{\tau}, \tau)\|_2^2\mathrm{d}\tau \\
=& \mathbb{E}[K^{\mathrm{gen}}(\hat{v}^s_N)-K^{\mathrm{gen}}(v_0)] \\
\le & \mathbb{E}[K^{\mathrm{gen}}(\hat{v}^s_N)-K^{\mathrm{gen}}(v_0)]+2\mathbb{E}[K^{\mathrm{gen}}_N(f)-K^{\mathrm{gen}}_N(\hat{v}^s_N)] \\
=& \mathbb{E}[K^{\mathrm{gen}}(\hat{v}^s_N)-K^{\mathrm{gen}}(v_0)]+2\mathbb{E}[K^{\mathrm{gen}}_N(f)-K^{\mathrm{gen}}_N(v_0)+K^{\mathrm{gen}}_N(v_0)-K^{\mathrm{gen}}_N(\hat{v}^s_N)] \\
=& \mathbb{E}[K^{\mathrm{gen}}(\hat{v}^s_N)-2K^{\mathrm{gen}}_N(\hat{v}^s_N)+K^{\mathrm{gen}}(v_0)]+2[K^{\mathrm{gen}}(f)-K^{\mathrm{gen}}(v_0)] \\
=& \mathbb{E}[K^{\mathrm{gen}}(\hat{v}^s_N)-2K^{\mathrm{gen}}_N(\hat{v}^s_N)+K^{\mathrm{gen}}(v_0)]+2\int_0^1\mathbb{E}_P\|f(X, Y_{\tau}, \tau)-v_0(X, Y_{\tau}, \tau)\|_2^2\mathrm{d}\tau.
\end{aligned}
$$
Taking infimum with respect to $f\in \mathcal{F}_{\mathrm{NN}, \Lambda}^{d_y}$ on the both sides, we obtain the result.
\end{proof}

\begin{lemma}\label{lem: gen_stochatic_error_bound}
Assume that $\|Y^s\|_{\infty}$ is a sub-Gaussian random variable. Let $\bar{\delta}=\bar{\delta}_N=(\log N)^{(1+\kappa)/2}$, with an arbitrarily fixed $\kappa\in (0, 1)$, and let $\underline{\delta}=\underline{\delta}_N=-(\log N)^{(1+\kappa)/2}$. Then, for sufficiently large $N$ and $N\ge \mathrm{Pdim}(\mathcal{F}_{\mathrm{NN}})$, it follows that
$$
\mathbb{E}[K^{\mathrm{gen}}(\hat{v}^s_N)-2K^{\mathrm{gen}}_N(\hat{v}^s_N)+K^{\mathrm{gen}}(v_0)]\le \frac{c^*SL\log S(\log N)^{3+2\kappa}}{N},
$$
where $c^*$ is a constant not depending on $S, L$ and $N$.
\end{lemma}

\begin{proof}[Proof of Lemma \ref{lem: gen_stochatic_error_bound}]
Let $T_i=(X^s_i, Y^s_i, \eta_i, \tau_i)$ for $i=1, \dots, N$, $\mathcal{D}_N=\{T_1, \dots, T_N\}$, and $T=(X^s, Y^s, \eta, \tau)$ be an independent copy of $T_1$. Firstly, we have
$$
\begin{aligned}
& \mathbb{E}[K^{\mathrm{gen}}(\hat{v}^s_N)-2K^{\mathrm{gen}}_N(\hat{v}^s_N)+K^{\mathrm{gen}}(v_0)] \\
=& \mathbb{E}_{\mathcal{D}_N}[K^{\mathrm{gen}}(\hat{v}^s_N)-2K^{\mathrm{gen}}_N(\hat{v}^s_N)+K^{\mathrm{gen}}(v_0)] \\
=& \mathbb{E}_{\mathcal{D}_N}\{K^{\mathrm{gen}}(\hat{v}^s_N)-K^{\mathrm{gen}}(v_0)-2[K^{\mathrm{gen}}_N(\hat{v}^s_N)-K^{\mathrm{gen}}_N(v_0)]\} \\
=& \mathbb{E}_{\mathcal{D}_N}\left\{\mathbb{E}_{T}[g(\hat{v}^s_N, T)]-\frac 2N\sum_{i=1}^Ng(\hat{v}^s_N, T_i)\right\},
\end{aligned}
$$
where $g(f, T)=\|\dot{a}_{\tau}\eta+\dot{b}_{\tau}Y^s-f(X^s, Y^s_{\tau}, \tau)\|_2^2-\|\dot{a}_{\tau}\eta+\dot{b}_{\tau}Y^s-v_0(X^s, Y^s_{\tau}, \tau)\|_2^2$ for $f\in\mathcal{F}_{\mathrm{NN}, \Lambda}^{d_y}$. For a $d_y$-dimensional vector $v$, denote its $j$-th component as $v_{(j)}$; additionally, we denote the $j$-th output coordinate of a function $f: \mathbb{R}^{d_x}\times\mathbb{R}^{d_y}\times [0, 1]\to\mathbb{R}^{d_y}$ as $f_{(j)}$, with $j\in\{1, \dots, d_y\}$. Furthermore, for any measurable function $h: \mathbb{R}^{d_x}\times\mathbb{R}^{d_y}\times [0, 1]\to\mathbb{R}$, define
$$
g_j(h, T)=[\dot{a}_{\tau}\eta_{(j)}+\dot{b}_{\tau}Y^s_{(j)}-h(X^s, Y^s_{\tau}, \tau)]^2-[\dot{a}_{\tau}\eta_{(j)}+\dot{b}_{\tau}Y^s_{(j)}-v_0(X^s, Y^s_{\tau}, \tau)_{(j)}]^2,
$$
for $j=1, \dots, d_y$. It is then clear that $g(f, T)=\sum_{j=1}^{d_y}g_j(f_{(j)}, T)$. Hence, we obtain
$$
\begin{aligned}
& \mathbb{E}[K^{\mathrm{gen}}(\hat{v}^s_N)-2K^{\mathrm{gen}}_N(\hat{v}^s_N)+K^{\mathrm{gen}}(v_0)] \\
=& \mathbb{E}_{\mathcal{D}_N}\left\{\mathbb{E}_{T}[g(\hat{v}^s_N, T)]-\frac 2N\sum_{i=1}^Ng(\hat{v}^s_N, T_i)\right\} \\
=& \sum_{j=1}^{d_y}\mathbb{E}_{\mathcal{D}_N}\left\{\mathbb{E}_{T}[g_j(\hat{v}^s_{N, (j)}, T)]-\frac 2N\sum_{i=1}^Ng_j(\hat{v}^s_{N, (j)}, T_i)\right\}.
\end{aligned}
$$
Subsequently, let us fix an arbitrary $j\in\{1, \dots, d_y\}$. For any $\iota_N\ge 1$, we let
$$
\begin{aligned}
U_0 &= \dot{a}_{\tau}\eta_{(j)}+\dot{b}_{\tau}Y^s_{(j)}, \\
U_1 &= (\dot{a}_{\tau}\eta_{(j)}+\dot{b}_{\tau}Y^s_{(j)})\mathds{1}(\|\eta\|_{\infty}+\|Y^s\|_{\infty}\le \iota_N), \\
U_2 &= \|\eta\|_{\infty}+\|Y^s\|_{\infty}, \\
V_0 &= v_0(X^s, Y^s_{\tau}, \tau)_{(j)}, \\
V_1 &= \mathbb{E}\left[(\dot{a}_{\tau}\eta_{(j)}+\dot{b}_{\tau}Y^s_{(j)})\mathds{1}(\|\eta\|_{\infty}+\|Y^s\|_{\infty}\le \iota_N)\Big|Y^s_{\tau}, X^s, \tau\right].
\end{aligned}
$$
Then, for any measurable function $h: \mathbb{R}^{d_x}\times\mathbb{R}^{d_y}\times [0, 1]\to\mathbb{R}$, define
$$
\begin{aligned}
g_{j, \iota_N}(h, T) &= [U_1-h(X^s, Y^s_{\tau}, \tau)]^2-(U_1-V_1)^2 \\
&= [V_1-h(X^s, Y^s_{\tau}, \tau)][2U_1-h(X^s, Y^s_{\tau}, \tau)-V_1].
\end{aligned}
$$
Since $\eta$ is a standard Gaussian random vector, $\|\eta\|_{\infty}^2$ follows a sub-exponential distribution. To see this, note that $\|\eta\|_{\infty}^2=\max(\eta_{(1)}^2, \dots, \eta_{(d_y)}^2)$ and $\eta_{(j)}^2$ follows a Gamma distribution with a shape parameter 1/2 and a rate parameter 1/2 for all $j\in \{1, \dots, d_y\}$. Hence, for any $\omega\in (0, 1/2)$, we have $\mathbb{E}[\exp(\omega\eta_{(j)}^2)]=(1-2\omega)^{-1/2}<\infty$ for all $j\in \{1, \dots, d_y\}$. This indicates that
$$
\begin{aligned}
\mathbb{E}[\exp(\omega d_y^{-1}\|\eta\|_{\infty}^2)] &\le \mathbb{E}\left[\exp\left(\frac{1}{d_y}\sum_{j=1}^{d_y}\omega\eta_{(j)}^2\right)\right] \\
&\le \frac{1}{d_y}\sum_{j=1}^{d_y}\mathbb{E}[\exp(\omega\eta_{(j)}^2)]=(1-2\omega)^{-1/2}<\infty.
\end{aligned}
$$
On the other hand, we have assumed that $\|Y^s\|_{\infty}$ is a sub-Gaussian random variable. Consequently, there exists a constant $\varsigma$ (which possibly associates with $d_y$) such that $\mathbb{E}[\exp(\varsigma\|\eta\|_{\infty}^2)]<\infty$ and $\mathbb{E}_P[\exp(\varsigma\|Y\|_{\infty}^2)]<\infty$. Hence, observe that
$$
\mathbb{E}[\exp(\varsigma U_2^2/2)]\le \mathbb{E}_P[\exp(\varsigma\|\eta\|_{\infty}^2+\varsigma\|Y\|_{\infty}^2)]=\mathbb{E}[\exp(\varsigma\|\eta\|_{\infty}^2)]\mathbb{E}_P[\exp(\varsigma\|Y\|_{\infty}^2)]<\infty.
$$
Additionally, as $a_{\tau}$ and $b_{\tau}$ are presumed to be continuously differentiable over $[0, 1]$, we let $\gamma=\max_{\xi\in [0, 1]}\max(|\dot{a}_{\xi}|, |\dot{b}_{\xi}|)$. Then, it follows that
$$
\begin{aligned}
& |g_j(f_{(j)}, T)-g_{j, \iota_N}(f_{(j)}, T)| \\
\le & \left|[U_0-f(X^s, Y^s_{\tau}, \tau)_{(j)}]^2-[U_1-f(X^s, Y^s_{\tau}, \tau)_{(j)}]^2\right|+\left|(U_0-V_0)^2-(U_1-V_1)^2\right| \\
\le & \left|(U_1-U_0)[U_0+U_1-2f(X^s, Y^s_{\tau}, \tau)_{(j)}]\right| \\
& +\left|(U_0-U_1+V_1-V_0)(U_0+U_1-V_0-V_1)\right| \\
\le & \gamma U_2\mathds{1}(U_2>\iota_N)(\gamma U_2+\gamma\iota_N+2\bar{\delta}_N) \\
& +\gamma U_2\mathds{1}(U_2>\iota_N)[\gamma U_2+2\gamma\iota_N+\|v_0(X^s, Y^s_{\tau}, \tau)\|_{\infty}] \\
& +\gamma\mathbb{E}[U_2\mathds{1}(U_2>\iota_N)|Y^s_{\tau}, X^s, \tau][\gamma U_2+2\gamma\iota_N+\|v_0(X^s, Y^s_{\tau}, \tau)\|_{\infty}] \\
=& \gamma U_2[2\gamma U_2+3\gamma\iota_N+2\bar{\delta}_N+\|v_0(X^s, Y^s_{\tau}, \tau)\|_{\infty}]\mathds{1}(U_2>\iota_N) \\
& +\gamma\mathbb{E}[U_2\mathds{1}(U_2>\iota_N)|Y^s_{\tau}, X^s, \tau][\gamma U_2+2\gamma\iota_N+\|v_0(X^s, Y^s_{\tau}, \tau)\|_{\infty}].
\end{aligned}
$$
Taking expectation, we have
$$
\begin{aligned}
& \mathbb{E}_{\mathcal{D}_N}\left\{\mathbb{E}_{T}[g_j(\hat{v}^s_{N, (j)}, T)]-\frac 2N\sum_{i=1}^Ng_j(\hat{v}^s_{N, (j)}, T_i)\right\} \\
\le & \mathbb{E}_{\mathcal{D}_N}\left\{\mathbb{E}_{T}[g_{j, \iota_N}(\hat{v}^s_{N, (j)}, T)]-\frac 2N\sum_{i=1}^Ng_{j, \iota_N}(\hat{v}^s_{N, (j)}, T_i)\right\} \\
& +3\gamma\mathbb{E}\left\{U_2[2\gamma U_2+3\gamma\iota_N+2\bar{\delta}_N+\|v_0(X^s, Y^s_{\tau}, \tau)\|_{\infty}]\mathds{1}(U_2>\iota_N)\right\} \\
& +3\gamma\mathbb{E}\left\{\mathbb{E}[U_2\mathds{1}(U_2>\iota_N)|Y^s_{\tau}, X^s, \tau][\gamma U_2+2\gamma\iota_N+\|v_0(X^s, Y^s_{\tau}, \tau)\|_{\infty}]\right\}.
\end{aligned}
$$
Specifically, observe that
$$
\begin{aligned}
& \mathbb{E}\left\{U_2[2\gamma U_2+3\gamma\iota_N+2\bar{\delta}_N+\|v_0(X^s, Y^s_{\tau}, \tau)\|_{\infty}]\mathds{1}(U_2>\iota_N)\right\} \\
\le & \frac{8\gamma}{\varsigma}\mathbb{E}[\exp(\varsigma U_2^2/2)]\exp(-\varsigma\iota_N^2/4)+\frac{4}{\varsigma}(3\gamma\iota_N+2\bar{\delta}_N)\mathbb{E}[\exp(\varsigma U_2^2/2)]\exp(-\varsigma\iota_N^2/4) \\
& +\frac{8}{\varsigma}\mathbb{E}[\|v_0(X^s, Y^s_{\tau}, \tau)\|_{\infty}\exp(\varsigma U_2^2/4)]\exp(-\varsigma\iota_N^2/8) \\
\le &\frac{8\gamma}{\varsigma}\mathbb{E}[\exp(\varsigma U_2^2/2)]\exp(-\varsigma\iota_N^2/4)+\frac{4}{\varsigma}(3\gamma\iota_N+2\bar{\delta}_N)\mathbb{E}[\exp(\varsigma U_2^2/2)]\exp(-\varsigma\iota_N^2/4) \\
& +\frac{8}{\varsigma}\left\{\mathbb{E}[\|v_0(X^s, Y^s_{\tau}, \tau)\|_{\infty}^2]\mathbb{E}[\exp(\varsigma U_2^2/2)]\right\}^{1/2}\exp(-\varsigma\iota_N^2/8) \\
\le &\frac{8\gamma}{\varsigma}\mathbb{E}[\exp(\varsigma U_2^2/2)]\exp(-\varsigma\iota_N^2/4)+\frac{4}{\varsigma}(3\gamma\iota_N+2\bar{\delta}_N)\mathbb{E}[\exp(\varsigma U_2^2/2)]\exp(-\varsigma\iota_N^2/4) \\
& +\frac{8\gamma}{\varsigma}\left\{\mathbb{E}(U_2^2)\mathbb{E}[\exp(\varsigma U_2^2/2)]\right\}^{1/2}\exp(-\varsigma\iota_N^2/8),
\end{aligned}
$$
and
$$
\begin{aligned}
& \mathbb{E}\left\{\mathbb{E}[U_2\mathds{1}(U_2>\iota_N)|Y^s_{\tau}, X^s, \tau][\gamma U_2+2\gamma\iota_N+\|v_0(X^s, Y^s_{\tau}, \tau)\|_{\infty}]\right\} \\
\le & \gamma\left\{\mathbb{E}(U_2^2)\mathbb{E}[U_2^2\mathds{1}(U_2>\iota_N)]\right\}^{1/2}+2\gamma\iota_N\mathbb{E}[U_2\mathds{1}(U_2>\iota_N)] \\
& +\mathbb{E}[\|v_0(X^s, Y^s_{\tau}, \tau)\|_{\infty}U_2\mathds{1}(U_2>\iota_N)] \\
\le & \frac{2\gamma}{\varsigma^{1/2}}\left\{\mathbb{E}(U_2^2)\mathbb{E}[\exp(\varsigma U_2^2/2)\right\}^{1/2}\exp(-\varsigma\iota_N^2/8)+\frac{8\gamma}{\varsigma}\iota_N\mathbb{E}[\exp(\varsigma U_2^2/2)\exp(-\varsigma\iota_N^2/4) \\
& +\frac{8}{\varsigma}\mathbb{E}[\|v_0(X^s, Y^s_{\tau}, \tau)\|_{\infty}\exp(\varsigma U_2^2/4)]\exp(-\varsigma\iota_N^2/8) \\
\le & \frac{2\gamma}{\varsigma^{1/2}}\left\{\mathbb{E}(U_2^2)\mathbb{E}[\exp(\varsigma U_2^2/2)\right\}^{1/2}\exp(-\varsigma\iota_N^2/8)+\frac{8\gamma}{\varsigma}\iota_N\mathbb{E}[\exp(\varsigma U_2^2/2)\exp(-\varsigma\iota_N^2/4) \\
& +\frac{8\gamma}{\varsigma}\left\{\mathbb{E}(U_2^2)\mathbb{E}[\exp(\varsigma U_2^2/2)]\right\}^{1/2}\exp(-\varsigma\iota_N^2/8).
\end{aligned}
$$
Here, we have applied the inequalities $a\le \exp(a)$ and $\mathds{1}(a>0)\le \exp(a)$ for $a\in\mathbb{R}$. As a consequence, we obtain
$$
\begin{aligned}
& \mathbb{E}_{\mathcal{D}_N}\left\{\mathbb{E}_{T}[g_j(\hat{v}^s_{N, (j)}, T)]-\frac 2N\sum_{i=1}^Ng_j(\hat{v}^s_{N, (j)}, T_i)\right\} \\
\le & \mathbb{E}_{\mathcal{D}_N}\left\{\mathbb{E}_{T}[g_{j, \iota_N}(\hat{v}^s_{N, (j)}, T)]-\frac 2N\sum_{i=1}^Ng_{j, \iota_N}(\hat{v}^s_{N, (j)}, T_i)\right\}+c_1(\iota_N+\bar{\delta}_N+1)\exp(-\varsigma\iota_N^2/8),
\end{aligned}
$$
where $c_1$ is a constant which depends only on $\gamma, \varsigma$ and $\mathbb{E}_P[\exp(\varsigma\|Y\|_{\infty}^2)]$. Recall that $\bar{\delta}_N=(\log N)^{(1+\kappa)/2}$. Set $\iota_N$ to $(8\varsigma^{-1}\log N)^{1/2}\vee 1$. Then, for sufficiently large $N$ such that $\bar{\delta}_N\ge \iota_N$, it holds that
$$
\begin{aligned}
& \mathbb{E}_{\mathcal{D}_N}\left\{\mathbb{E}_{T}[g_j(\hat{v}^s_{N, (j)}, T)]-\frac 2N\sum_{i=1}^Ng_j(\hat{v}^s_{N, (j)}, T_i)\right\} \\
\le & \mathbb{E}_{\mathcal{D}_N}\left\{\mathbb{E}_{T}[g_{j, \iota_N}(\hat{v}^s_{N, (j)}, T)]-\frac 2N\sum_{i=1}^Ng_{j, \iota_N}(\hat{v}^s_{N, (j)}, T_i)\right\}+3c_1(\log N)^{(1+\kappa)/2}N^{-1}.
\end{aligned}
$$
Furthermore, we proceed to verify the conditions in Theorem \ref{thm: gen_thm_11.4_gyorfi}. Notice that
$$
\sup_{f\in\mathcal{F}_{\mathrm{NN}, \Lambda}^{d_y}, T\in\mathbb{R}^{d_x+2d_y}\times [0, 1]}|g_{j, \iota_N}(f_{(j)}, T)|\le 6\gamma^2\iota_N^2+2\bar{\delta}_N^2\le 8\bar{\delta}_N^2=8(\log N)^{1+\kappa},
$$
whenever $\bar{\delta}_N\ge (\gamma\vee 1)\iota_N$. Furthermore, for any $f\in\mathcal{F}_{\mathrm{NN}, \Lambda}^{d_y}$,
$$
\begin{aligned}
\mathbb{E}[g_{j, \iota_N}(f_{(j)}, T)] &= \mathbb{E}\left\{[V_1-f(X^s, Y^s_{\tau}, \tau)_{(j)}][2U_1-f(X^s, Y^s_{\tau}, \tau)_{(j)}-V_1]\right\} \\
&= \mathbb{E}[V_1-f(X^s, Y^s_{\tau}, \tau)_{(j)}]^2,
\end{aligned}
$$
and
$$
\begin{aligned}
\mathbb{E}[g_{j, \iota_N}(f, T)^2] &= \mathbb{E}\left\{[V_1-f(X^s, Y^s_{\tau}, \tau)_{(j)}]^2[2U_1-f(X^s, Y^s_{\tau}, \tau)_{(j)}-V_1]^2\right\} \\
&\le (3\gamma\iota_N+\bar{\delta}_N)^2\mathbb{E}[V_1-f(X^s, Y^s_{\tau}, \tau)_{(j)}]^2 \\
&\le 16\bar{\delta}_N^2\mathbb{E}[g_{j, \iota_N}(f_{(j)}, T)] \\
&= 16(\log N)^{1+\kappa}\mathbb{E}[g_{j, \iota_N}(f_{(j)}, T)],
\end{aligned}
$$
provided that $N$ is sufficiently large such that $\bar{\delta}_N\ge (\gamma\vee 1)\iota_N$. Hence, Theorem \ref{thm: gen_thm_11.4_gyorfi} suggests that, for sufficiently large $N$ such that $\bar{\delta}_N\ge (\gamma\vee 1)\iota_N$, with $N\ge \mathrm{Pdim}(\mathcal{F}_{\mathrm{NN}})$, and for arbitrary $t>0$, we have
$$
\begin{aligned}
& \mathbb{P}_{\mathcal{D}_N}\left\{\mathbb{E}_{T}[g_{j, \iota_N}(\hat{v}^s_{N, (j)}, T)]-\frac 2N\sum_{i=1}^Ng_{j, \iota_N}(\hat{v}^s_{N, (j)}, T_i)\ge t\right\} \\
\le & \mathbb{P}_{\mathcal{D}_N}\left\{\mathbb{E}_{T}[g_{j, \iota_N}(\hat{v}^s_{N, (j)}, T)]-\frac 1N\sum_{i=1}^Ng_{j, \iota_N}(\hat{v}^s_{N, (j)}, T_i)\ge \frac 12\left\{\frac t2+\frac t2+\mathbb{E}_{T}[g_{j, \iota_N}(\hat{v}^s_{N, (j)}, T)]\right\}\right\} \\
\le & \mathbb{P}_{\mathcal{D}_N}\left(\exists h\in\mathcal{F}_{\mathrm{NN}}: \mathbb{E}[g_{j, \iota_N}(h, T)]-\frac 1N\sum_{i=1}^Ng_{j, \iota_N}(h, T_i)\ge \frac 12\left\{\frac t2+\frac t2+\mathbb{E}_{T}[g_{j, \iota_N}(h, T)]\right\}\right) \\
\le & 14\mathcal{N}_N\left(c_2t, \|\cdot \|_{\infty}, \{g_{j, \iota_N}(h, \cdot): \mathbb{R}^{d_x}\times \mathbb{R}^{d_y}\times \mathbb{R}^{d_y}\times [0, 1]\to\mathbb{R}, h\in\mathcal{F}_{\mathrm{NN}}\}\right)\exp\left(-\frac{Nt}{c_3(\log N)^{2+2\kappa}}\right),
\end{aligned}
$$
where $c_2, c_3$ are universal constants. Subsequently, we bound the covering number. Fix $\{x_1, \dots, x_N\}\subset (\mathcal{R}^{d_x})^N$,  $\{y_1, \dots, y_N\}\subset (\mathcal{R}^{d_y})^N$, $\{z_1, \dots, z_N\}\subset (\mathcal{R}^{d_y})^N$ and $\{\xi_1, \dots, \xi_N\}\in [0, 1]^N$. Let $\mathcal{C}=\{(x_1, \zeta_1, \xi_1), \dots, (x_N, \zeta_N, \xi_N)\}$ where $\zeta_i=a_{\xi_i}z_i+b_{\xi_i}y_i$ for $i=1, \dots, N$. Let $w^{\sharp}=\{w_1, \dots, w_k\}$ be an $\epsilon$-covering set of $\mathcal{F}_{\mathrm{NN}|\mathcal{C}}$ where $w_i=h_{i|\mathcal{C}}$ for some $h_i\in\mathcal{F}_{\mathrm{NN}}\ (i=1, \dots, k)$, such that for any $h\in \mathcal{F}_{\mathrm{NN}}$, there exists $w^*=h^*_{|\mathcal{C}}\in w^{\sharp}$ satisfying $\|w^*-h_{|\mathcal{C}}\|_{\infty}<\epsilon$. This indicates
$$
\begin{aligned}
& |g_{j, \iota_N}(h, (x_i, y_i, z_i, \xi_i))-g_{j, \iota_N}(h^*, (x_i, y_i, z_i, \xi_i))| \\
\le & |h^*(x_i, \zeta_i, \xi_i)-h(x_i, \zeta_i, \xi_i)|\cdot(2\gamma\iota_N+2\bar{\delta}_N) \\
\le & 2(\gamma\iota_N+\bar{\delta}_N)\epsilon \\
\le & 4\bar{\delta}_N\epsilon,
\end{aligned}
$$
whenever $\bar{\delta}_N\ge (\gamma\vee 1)\iota_N$. Therefore,
$$
\begin{aligned}
& \mathcal{N}_N(c_2t, \|\cdot \|_{\infty}, \{g_{j, \iota_N}(h, \cdot): \mathbb{R}^{d_x}\times \mathbb{R}^{d_y}\times \mathbb{R}^{d_y}\times [0, 1]\to\mathbb{R}, h\in\mathcal{F}_{\mathrm{NN}}\}) \\
\le & \mathcal{N}_N(c_2t/(4\bar{\delta}_N), \|\cdot \|_{\infty}, \mathcal{F}_{\mathrm{NN}}).
\end{aligned}
$$
Then, with Lemma \ref{lem: thm_12.2_anthony} and Lemma \ref{lem: thm_7_bartlett}, for sufficiently large $N$ with $N\ge \mathrm{Pdim}(\mathcal{F}_{\mathrm{NN}})$ and any $a_N\ge 1/N$, we have
$$
\begin{aligned}
& \mathbb{E}_{\mathcal{D}_N}\left\{\mathbb{E}_{T}[g_{j, \iota_N}(\hat{v}^s_{N, (j)}, T)]-\frac 2N\sum_{i=1}^Ng_{j, \iota_N}(\hat{v}^s_{N, (j)}, T_i)\right\} \\
\le & a_N+14\int_{a_N}^{\infty}\mathcal{N}_N(c_2t/(4\bar{\delta}_N), \|\cdot \|_{\infty}, \mathcal{F}_{\mathrm{NN}})\exp\left(-\frac{Nt}{c_3(\log N)^{2+2\kappa}}\right)\mathrm{d}t \\
\le & a_N+14\mathcal{N}_N(c_2a_N/(4\bar{\delta}_N), \|\cdot \|_{\infty}, \mathcal{F}_{\mathrm{NN}})\int_{a_N}^{\infty}\exp\left(-\frac{Nt}{c_3(\log N)^{2+2\kappa}}\right)\mathrm{d}t \\
\le & a_N+14\left(c_4N^2\bar{\delta}_N^2\right)^{c_5SL\log S}\cdot \frac{c_3(\log N)^{2+2\kappa}}{N}\exp\left(-\frac{Na_N}{c_3(\log N)^{2+2\kappa}}\right),
\end{aligned}
$$
where $c_4$ and $c_5$ are universal constants. Choose
$$
a_N=\frac{c_3c_5(\log N)^{2+2\kappa}}{N}SL\log S\log\left(c_4N^2\bar{\delta}_N^2\right).
$$
For sufficiently large $N$, we have
$$
\mathbb{E}_{\mathcal{D}_N}\left\{\mathbb{E}_{T}[g_{j, \iota_N}(\hat{v}^s_{N, (j)}, T)]-\frac 2N\sum_{i=1}^Ng_{j, \iota_N}(\hat{v}^s_{N, (j)}, T_i)\right\}\le \frac{c_6SL\log S(\log N)^{3+2\kappa}}{N},
$$
where $c_6$ is a constant not depending on $S, L$ and $N$. Noticing the arbitrariness of $j$, we complete the proof.
\end{proof}

\begin{lemma}\label{lem: gen_approximation_error_bound}
Assume that
\begin{enumerate}[label=(\roman*)]
    \item $e_j^{\top}v_0\in \mathcal{W}^{1, \infty}_{\mathrm{Gen}}(\mathbb{R}^{d_x+d_y}, B_u)$ with $B_u\le c(u^m+1)$ for some universal constants $c>0$, $m\in [0, 1]$, and for any $j\in\{1, \dots, d_y\}$, where $e_j$ denotes a $d_y$-dimensional one-hot vector with the $j$-th component equal to 1 and all other components equal to 0;
    \item $\|Y^s\|_{\infty}$ and $\|X^s\|_{\infty}$ follow sub-Gaussian distributions.
\end{enumerate}
Suppose that the depth $L$ and width $M$ of $\mathcal{F}_{\mathrm{NN}, \Lambda}^{d_y}$ satisfy
$$
\begin{aligned}
L &\le C_1(d_x+d_y+1)^2S_1\log S_1+3, \\
M &\le C_22^{d_x+d_y+1}d_y(d_x+d_y+1)S_2\log S_2,
\end{aligned}
$$
for any $S_1, S_2\in\mathbb{N}_+$, where $C_1$ and $C_2$ are universal constants. Let $\bar{\delta}=\bar{\delta}_N=(\log N)^{(1+\kappa)/2}$, with an arbitrarily fixed $\kappa\in (0, 1)$, and let $\underline{\delta}=\underline{\delta}_N=-(\log N)^{-(1+\kappa)/2}$, $\Lambda=\Lambda_N=(\log N)^{(1+\kappa)/2}$. Then, for sufficiently large $N$, it follows that
$$
\begin{aligned}
& \inf_{f\in\mathcal{F}_{\mathrm{NN}, \Lambda}^{d_y}}\int_0^1\mathbb{E}_P\|f(X, Y_{\tau}, \tau)-v_0(X, Y_{\tau}, \tau)\|_2^2\mathrm{d}\tau \\
\le & c^*\left\{\left[(S_1S_2)^{-2/(d_x+d_y+1)}(\log N)^{m/2}\right]^2+\frac{(\log N)^{1+\kappa}}{N}\right\},
\end{aligned}
$$
where $c^*$ is a constant not depending on $S_1, S_2$ and $N$.
\end{lemma}

\begin{proof}[Proof of Lemma \ref{lem: gen_approximation_error_bound}]
Let $\gamma_1=\max_{\tau\in [0, 1]}\max(|a_{\tau}|, |b_{\tau}|)\vee 1$ and
$$
\gamma_2=\max_{\tau\in [0, 1]}\max(|\dot{a}_{\tau}|, |\dot{b}_{\tau}|),
$$
both of which are well-defined as we presume that $a_{\tau}$ and $b_{\tau}$ are continuously differentiable over $[0, 1]$. For any $\iota_N\ge 1$ and $\tau\in [0, 1]$, observe that
$$
\begin{aligned}
& \mathbb{E}_P\|f(X, Y_{\tau}, \tau)-v_0(X, Y_{\tau}, \tau)\|_2^2\\
=& \mathbb{E}_P\left\{\|f(X, Y_{\tau}, \tau)-v_0(X, Y_{\tau}, \tau)\|_2^2\mathds{1}(\|X\|_{\infty}<\iota_N)\mathds{1}((\|\eta\|_{\infty}+\|Y\|_{\infty})<\gamma_1^{-1}\iota_N)\right\} \\
& +\mathbb{E}_P\left\{\|f(X, Y_{\tau}, \tau)-v_0(X, Y_{\tau}, \tau)\|_2^2\mathds{1}(\|X\|_{\infty}\ge \iota_N)\right\} \\
& +\mathbb{E}_P\left\{\|f(X, Y_{\tau}, \tau)-v_0(X, Y_{\tau}, \tau)\|_2^2\mathds{1}(\|X\|_{\infty}<\iota_N)\mathds{1}((\|\eta\|_{\infty}+\|Y\|_{\infty})\ge \gamma_1^{-1}\iota_N)\right\}.
\end{aligned}
$$
As we assume that $\|Y^s\|_{\infty}$ and $\|X^s\|_{\infty}$ follow sub-Gaussian distributions, $\|Y^s\|_{\infty}^2$ and $\|X^s\|_{\infty}^2$ are sub-exponentially distributed random variables. This indicates the existence of a constant $\varsigma$ such that $\mathbb{E}_P[\exp(\varsigma\|Y\|_{\infty}^2)]<\infty$, $\mathbb{E}_P[\exp(\varsigma\|X\|_{\infty}^2)]<\infty$ and $\mathbb{E}[\exp(\varsigma\|\eta\|_{\infty}^2)]<\infty$. Hence, we have
$$
\begin{aligned}
\mathbb{E}_P[\exp(\varsigma(\|\eta\|_{\infty}+\|Y\|_{\infty})^2/2)] &\le \mathbb{E}_P[\exp(\varsigma\|\eta\|_{\infty}^2+\varsigma\|Y\|_{\infty}^2)] \\
&\le \mathbb{E}[\exp(\varsigma\|\eta\|_{\infty}^2)]\mathbb{E}_P[\exp(\varsigma\|Y\|_{\infty}^2)].
\end{aligned}
$$
For clarity, denote the $j$-th output coordinate of a function $f: \mathbb{R}^{d_x}\times \mathbb{R}^{d_y}\times [0, 1]\to\mathbb{R}^{d_y}$ as $f_{(j)}$, with $j\in\{1, \dots, d_y\}$. It then follows that
$$
\begin{aligned}
& \mathbb{E}_P\left\{\|f(X, Y_{\tau}, \tau)-v_0(X, Y_{\tau}, \tau)\|_2^2\mathds{1}(\|X\|_{\infty}\ge \iota_N)\right\} \\
=& \sum_{j=1}^{d_y}\mathbb{E}_P\left\{[f_{(j)}(X, Y_{\tau}, \tau)-v_{0, (j)}(X, Y_{\tau}, \tau)]^2\mathds{1}(\|X\|_{\infty}\ge \iota_N)\right\} \\
\le & 2d_y\bar{\delta}_N^2\mathbb{E}_P[\mathds{1}(\|X\|_{\infty}\ge \iota_N)]+2d_y\mathbb{E}_P[\|v_0(X, Y_{\tau}, \tau)\|_{\infty}^2\mathds{1}(\|X\|_{\infty}\ge \iota_N)] \\
\le & 2d_y\bar{\delta}_N^2\mathbb{E}_P[\exp(\varsigma \|X\|_{\infty}^2/2)]\exp(-\varsigma\iota_N^2/2) \\
& +2d_y\mathbb{E}_P[\|v_0(X, Y_{\tau}, \tau)\|_{\infty}^2\exp(\varsigma \|X\|_{\infty}^2/2)]\exp(-\varsigma\iota_N^2/2) \\
\le & 2d_y\bar{\delta}_N^2\mathbb{E}_P[\exp(\varsigma \|X\|_{\infty}^2/2)]\exp(-\varsigma\iota_N^2/2) \\
& +2d_y\left\{\mathbb{E}_P[\|v_0(X, Y_{\tau}, \tau)\|_{\infty}^4]\mathbb{E}_P[\exp(\varsigma \|X\|_{\infty}^2)]\right\}^{1/2}\exp(-\varsigma\iota_N^2/2) \\
\le & 2d_y\bar{\delta}_N^2\mathbb{E}_P[\exp(\varsigma \|X\|_{\infty}^2/2)]\exp(-\varsigma\iota_N^2/2) \\
& +8d_y\gamma_2^2\left\{\mathbb{E}_P(\|\eta\|_{\infty}^4+\|Y\|_{\infty}^4)\mathbb{E}_P[\exp(\varsigma \|X\|_{\infty}^2)]\right\}^{1/2}\exp(-\varsigma\iota_N^2/2),
\end{aligned}
$$
and
$$
\begin{aligned}
& \mathbb{E}_P\left\{\|f(X, Y_{\tau}, \tau)-v_0(X, Y_{\tau}, \tau)\|_2^2\mathds{1}(\|X\|_{\infty}<\iota_N)\mathds{1}((\|\eta\|_{\infty}+\|Y\|_{\infty})\ge \gamma_1^{-1}\iota_N)\right\} \\
\le & \mathbb{E}_P\left\{\|f(X, Y_{\tau}, \tau)-v_0(X, Y_{\tau}, \tau)\|_2^2\mathds{1}((\|\eta\|_{\infty}+\|Y\|_{\infty})\ge \gamma_1^{-1}\iota_N)\right\} \\
\le & 2d_y\bar{\delta}_N^2\mathbb{E}_P[\mathds{1}((\|\eta\|_{\infty}+\|Y\|_{\infty})\ge \gamma_1^{-1}\iota_N)] \\
& +2d_y\mathbb{E}_P[\|v_0(X, Y_{\tau}, \tau)\|_{\infty}^2\mathds{1}((\|\eta\|_{\infty}+\|Y\|_{\infty})\ge \gamma_1^{-1}\iota_N)] \\
\le & 2d_y\bar{\delta}_N^2\mathbb{E}_P[\exp(\varsigma(\|\eta\|_{\infty}+\|Y\|_{\infty})^2/4)]\exp(-\varsigma\gamma_1^{-2}\iota_N^2/4) \\
& +8d_y\gamma_2^2\left\{\mathbb{E}_P(\|\eta\|_{\infty}^4+\|Y\|_{\infty}^4)\mathbb{E}_P[\exp(\varsigma(\|\eta\|_{\infty}+\|Y\|_{\infty})^2/2)]\right\}^{1/2}\exp(-\varsigma\gamma_1^{-2}\iota_N^2/4) \\
\le & 2d_y\bar{\delta}_N^2\mathbb{E}_P[\exp(\varsigma\|\eta\|_{\infty}^2/2)]\mathbb{E}_P[\exp(\varsigma\|Y\|_{\infty}^2/2)]\exp(-\varsigma\gamma_1^{-2}\iota_N^2/4) \\
& +8d_y\gamma_2^2\left\{\mathbb{E}_P(\|\eta\|_{\infty}^4+\|Y\|_{\infty}^4)\mathbb{E}_P[\exp(\varsigma\|\eta\|_{\infty}^2)]\mathbb{E}_P[\exp(\varsigma\|Y\|_{\infty}^2)]\right\}^{1/2}\exp(-\varsigma\gamma_1^{-2}\iota_N^2/4).
\end{aligned}
$$
Hence, we obtain that
$$
\begin{aligned}
& \mathbb{E}_P\left\{\|f(X, Y_{\tau}, \tau)-v_0(X, Y_{\tau}, \tau)\|_2^2\mathds{1}(\|X\|_{\infty}\ge \iota_N)\right\} \\
& +\mathbb{E}_P\left\{\|f(X, Y_{\tau}, \tau)-v_0(X, Y_{\tau}, \tau)\|_2^2\mathds{1}(\|X\|_{\infty}<\iota_N)\mathds{1}((\|\eta\|_{\infty}+\|Y\|_{\infty})\ge \gamma_1^{-1}\iota_N)\right\}\\
\le & c_1(\bar{\delta}_N^2+1)\exp(-\varsigma\gamma_1^{-2}\iota_N^2/4),
\end{aligned}
$$
where $c_1$ is a constant which depends only on $d_y, \gamma_2, \varsigma, \mathbb{E}_P[\exp(\varsigma \|X\|_{\infty}^2)]$ and $\mathbb{E}_P[\exp(\varsigma \|Y\|_{\infty}^2/2)]$.
Therefore,
$$
\begin{aligned}
& \mathbb{E}_P\|f(X, Y_{\tau}, \tau)-v_0(X, Y_{\tau}, \tau)\|_2^2\\
\le & \mathbb{E}_P\left\{\|f(X, Y_{\tau}, \tau)-v_0(X, Y_{\tau}, \tau)\|_2^2\mathds{1}(\|X\|_{\infty}<\iota_N)\mathds{1}((\|\eta\|_{\infty}+\|Y\|_{\infty})<\gamma_1^{-1}\iota_N)\right\} \\
& +c_1(\bar{\delta}_N^2+1)\exp(-\varsigma\gamma_1^{-2}\iota_N^2/4).
\end{aligned}
$$
Next, let us focus on the region
$$
\{(x, y, \tau): \|x\|_{\infty}<\iota_N, \|y\|_{\infty}<\iota_N, \tau\in (0, 1)\}=(-\iota_N, \iota_N)^{d_x+d_y}\times (0, 1).
$$
Fix an arbitrary $j\in \{1, \dots, d_y\}$. Let
$$
h_j(x, y, \tau)=v_{0, (j)}\left(2\iota_Nx-\iota_N\mathrm{1}_{d_x}, 2\iota_Ny-\iota_N\mathrm{1}_{d_y}, \tau\right),
$$
for $(x, y, \tau)\in (0, 1)^{d_x}\times (0, 1)^{d_y}\times (0, 1)$. By the assumption that each component of $v_0$ belongs to $\mathcal{W}^{1, \infty}_{\mathrm{Gen}}(\mathbb{R}^{d_x+d_y}, B_u)$ with $B_u\le c(u^m+1)$, we have $\|h_j\|_{\mathcal{W}^{1, \infty}}((0, 1)^{d_x+d_y+1})\le c(\iota_N^m+1)$, where the constants $c>0$, $m\in [0, 1]$. Lemma \ref{lem: cor_b.2_gao} demonstrates that for any $S_1, S_2\in\mathbb{N}_+$, there exists a function $h_j^*$ implemented by a ReLU network with depth $L^*\le (d_x+d_y+1)^2S_1\log S_1$, width $M^*\le 2^{d_x+d_y+1}(d_x+d_y+1)S_2\log S_2$, such that $\|h^*\|_{\mathcal{W}^{1, \infty}((0, 1)^{d_x+d_y+1})}\le c_2(\iota_N^m+1)$ and
$$
|h_j^*(x, y, \tau)-h_j(x, y, \tau)|\le c_3(\iota_N^m+1)(S_1S_2)^{-2/(d_x+d_y+1)},
$$
for all $(x, y, \tau)\in (0, 1)^{d_x}\times (0, 1)^{d_y}\times (0, 1)$, where $c_2$ and $c_3$ are constants which depend only on $d_x$ and $d_y$. Let
$$
h_j^{\dagger}(x, y, \tau)=h_j^*\left((x+\iota_N\mathrm{1}_{d_x})/(2\iota_N), (y+\iota_N\mathrm{1}_{d_y})/(2\iota_N), \tau\right),
$$
for $(x, y, \tau)\in (-\iota_N, \iota_N)^{d_x}\times (-\iota_N, \iota_N)^{d_y}\times (0, 1)$. We obtain that
$$
|h_j^{\dagger}(x, y, \tau)-v_{0, (j)}(x, y, \tau)|\le c_3(\iota_N^m+1)(S_1S_2)^{-2/(d_x+d_y+1)},
$$
for all $(x, y, \tau)\in (-\iota_N, \iota_N)^{d_x}\times (-\iota_N, \iota_N)^{d_y}\times (0, 1)$. Furthermore, note that
$$
\begin{aligned}
h_j^{\dagger}(x, y, \tau) =& h_j^*\left(\frac{x+\iota_N\mathrm{1}_{d_x}}{2\iota_N}, \frac{y+\iota_N\mathrm{1}_{d_y}}{2\iota_N}, \tau\right) \\
=& h_j^*\bigg(\mathrm{relu}\left(\frac{x+\iota_N\mathrm{1}_{d_x}}{2\iota_N}\right)-\mathrm{relu}\left(-\frac{x+\iota_N\mathrm{1}_{d_x}}{2\iota_N}\right), \\
& \mathrm{relu}\left(\frac{y+\iota_N\mathrm{1}_{d_y}}{2\iota_N}\right)-\mathrm{relu}\left(-\frac{y+\iota_N\mathrm{1}_{d_y}}{2\iota_N}\right), \\
& \mathrm{relu}(\tau)-\mathrm{relu}(-\tau)\bigg),
\end{aligned}
$$
which is implemented by a neural network with ReLU activations, depth $L^{\dagger}=L^*+1$, width $M^{\dagger}=M^*$ and Lipschitz constant no more than $c_2(d_x+d_y+1)(\iota_N^m+1)$. In addition, let
$$
h_j^{\ddagger}(x, y, \tau)=
\begin{cases}
    \bar{\delta}_N, & h_j^{\dagger}(x, y, \tau)>\bar{\delta}_N, \\
    h_j^{\dagger}(x, y, \tau), & \underline{\delta}_N\le h_j^{\dagger}(x, y, \tau)\le \bar{\delta}_N, \\
    \underline{\delta}_N, & h_j^{\dagger}(x, y, \tau)<\underline{\delta}_N.
\end{cases}
$$
A straightforward calculation shows that
$$
h_j^{\ddagger}(x, y, \tau)=\mathrm{relu}(-\mathrm{relu}(-h_j^{\dagger}(x, y, \tau)+\bar{\delta}_N)+\bar{\delta}_N)-\mathrm{relu}(-\mathrm{relu}(h_j^{\dagger}(x, y, \tau)-\underline{\delta}_N)-\underline{\delta}_N),
$$
indicating that $h_j^{\ddagger}(x, y, \tau)$ can be implemented by a ReLU network with depth $L^{\ddagger}=L^*+3$, width $M^{\ddagger}=M^*$ and Lipschitz constant no more than $2c_2(d_x+d_y+1)(\iota_N^m+1)$. When $\bar{\delta}_N\ge c_2(\iota_N^m+1)$, it follows that for any $\tau\in (0, 1)$,
$$
\begin{aligned}
& \mathbb{E}_P\Big\{[h_j^{\ddagger}(X, Y_{\tau}, \tau)-v_{0, (j)}(X, Y_{\tau}, \tau)]^2\mathds{1}(\|X\|_{\infty}<\iota_N)\mathds{1}((\|\eta\|_{\infty}+\|Y\|_{\infty})<\gamma_1^{-1}\iota_N)\Big\} \\
\le & \left[c_3(\iota_N^m+1)(S_1S_2)^{-2/(d_x+d_y+1)}\right]^2.
\end{aligned}
$$
Let $f^{\ddagger}(x, y, \tau)=(h_1^{\ddagger}(x, y, \tau), \dots, h_{d_y}^{\ddagger}(x, y, \tau))^{\top}$. It is straightforward to verify that $f^{\ddagger}$ can be implemented by a ReLU network in $\mathcal{F}_{\mathrm{NN}, \Lambda}^{d_y}$ with depth $L=L^*+3$, width $M=d_yM^*$ and Lipschitz constant no more than $2c_2d_y(d_x+d_y+1)(\iota_N^m+1)$. Hence, we have for any $\tau\in (0, 1)$,
$$
\begin{aligned}
& \mathbb{E}_P\Big\{\|f^{\ddagger}(X, Y_{\tau}, \tau)-v_0(X, Y_{\tau}, \tau)\|_2^2\mathds{1}(\|X\|_{\infty}<\iota_N)\mathds{1}((\|\eta\|_{\infty}+\|Y\|_{\infty})<\gamma_1^{-1}\iota_N)\Big\} \\
=& \sum_{j=1}^{d_y}\mathbb{E}_P\Big\{[h_j^{\ddagger}(X, Y_{\tau}, \tau)-v_{0, (j)}(X, Y_{\tau}, \tau)]^2\mathds{1}(\|X\|_{\infty}<\iota_N)\mathds{1}((\|\eta\|_{\infty}+\|Y\|_{\infty})<\gamma_1^{-1}\iota_N)\Big\} \\
\le & d_y\left[c_3(\iota_N^m+1)(S_1S_2)^{-2/(d_x+d_y+1)}\right]^2,
\end{aligned}
$$
provided that $\bar{\delta}_N\ge c_2(\iota_N^m+1)$ and $\Lambda_N\ge 2c_2d_y(d_x+d_y+1)(\iota_N^m+1)$. Recall that $\bar{\delta}_N=(\log N)^{(1+\kappa)/2}$, $\Lambda_N=(\log N)^{(1+\kappa)/2}$ and set $\iota_N$ to $[2\varsigma^{-1/2}\gamma_1^{-1}(\log N)^{1/2}]\vee 1$. We conclude that for sufficiently large $N$ satisfying $\bar{\delta}_N\ge c_2(\iota_N^m+1)$, $\Lambda_N\ge 2c_2d_y(d_x+d_y+1)(\iota_N^m+1)$ and $\log N\ge (\varsigma\gamma_1^2/4)\vee 1$, it holds that
$$
\begin{aligned}
& \inf_{f\in\mathcal{F}_{\mathrm{NN}, \Lambda}^{d_y}}\int_0^1\mathbb{E}_P\|f(X, Y_{\tau}, \tau)-v_0(X, Y_{\tau}, \tau)\|_2^2\mathrm{d}\tau \\
\le & \inf_{f\in \mathcal{F}_{\mathrm{NN}, \Lambda}^{d_y}}\int_0^1\mathbb{E}_P\big\{\|f(X, Y_{\tau}, \tau)-v_0(X, Y_{\tau}, \tau)\|_2^2\mathds{1}(\|X\|_{\infty}\le \iota_N)\mathds{1}((\|\eta\|_{\infty}+\|Y\|_{\infty})\le \gamma_1^{-1}\iota_N)\big\}\mathrm{d}\tau \\
& +c_1(\bar{\delta}_N^2+1)\exp(-\varsigma\gamma_1^{-2}\iota_N^2/4) \\
\le & \int_0^1\mathbb{E}_P\big\{\|f^{\ddagger}(X, Y_{\tau}, \tau)-v_0(X, Y_{\tau}, \tau)\|_2^2\mathds{1}(\|X\|_{\infty}\le \iota_N)\mathds{1}((\|\eta\|_{\infty}+\|Y\|_{\infty})\le \gamma_1^{-1}\iota_N)\big\}\mathrm{d}\tau \\
& +c_1(\bar{\delta}_N^2+1)\exp(-\varsigma\gamma_1^{-2}\iota_N^2/4) \\
\le & c_4\left\{\left[(S_1S_2)^{-2/(d_x+d_y+1)}(\log N)^{m/2}\right]^2+\frac{(\log N)^{1+\kappa}}{N}\right\},
\end{aligned}
$$
where $c_4$ is a constant not depending on $S_1, S_2$ and $N$. This completes the proof.
\end{proof}

\begin{proof}[Proof of Theorem 4.8]
To commence, Lemma \ref{lem: gen_error_decomposition}, Lemma \ref{lem: gen_stochatic_error_bound} and Lemma \ref{lem: gen_approximation_error_bound} indicate
$$
\begin{aligned}
& \int_0^1\mathbb{E}\|\hat{v}^s_N(X^s, Y^s_{\tau}, \tau)-v_0(X^s, Y^s_{\tau}, \tau)\|_2^2\mathrm{d}\tau \\
\le & \frac{c_1SL\log S(\log N)^{3+2\kappa}}{N}+c_2\left\{\left[(S_1S_2)^{-2/(d_x+d_y+1)}(\log N)^{m/2}\right]^2+\frac{(\log N)^{1+\kappa}}{N}\right\},
\end{aligned}
$$
where $c_1, c_2$ are constants not depending on $S, L, S_1, S_2$ and $N$, and $S_1, S_2$ satisfy the conditions that the network depth $L\le c_3(d_x+d_y+1)^2S_1\log S_1+3$, network width $M\le c_42^{d_x+d_y+1}d_y(d_x+d_y+1)S_2\log S_2$ for some universal constants $c_3$ and $c_4$, when $N$ is sufficiently large and $N\ge \mathrm{Pdim}(\mathcal{F}_{\mathrm{NN}})$. Therefore, by letting $S_1=\mathcal{O}(N^{(d_x+d_y+1)/(2(d_x+d_y+1)+4)})$ and $S_2=\mathcal{O}(1)$,  we obtain
$$
M=\mathcal{O}(1), \quad L=\mathcal{O}\left(N^{\frac{(d_x+d_y+1)}{2(d_x+d_y+1)+4}}\log N\right), \quad S=\mathcal{O}(M^2L)=\mathcal{O}\left(N^{\frac{(d_x+d_y+1)}{2(d_x+d_y+1)+4}}\log N\right),
$$
yielding
$$
\int_0^1\mathbb{E}\|\hat{v}^s_N(X^s, Y^s_{\tau}, \tau)-v_0(X^s, Y^s_{\tau}, \tau)\|_2^2\mathrm{d}\tau \le c_5N^{-\frac{2}{d_x+d_y+3}}(\log N)^{6+2\kappa},
$$
where $c_5$ is a constant not depending on $N$, for $N\ge 2$.

Next, we proceed to tackle the conditional density estimation error. Fix any $x\in\mathcal{X}^s$. Consider the following two ODEs
$$
\begin{aligned}
\mathrm{d}Z_{\tau} &= v_0(x, Z_{\tau}, \tau)\mathrm{d}\tau, \quad Z_0\sim N(0, I_{d_y}), \\
\mathrm{d}\hat{Z}_{\tau} &= \hat{v}^s_N(x, \hat{Z}_{\tau}, \tau)\mathrm{d}\tau, \quad \hat{Z}_0\sim N(0, I_{d_y}).
\end{aligned}
$$
We denote the particles at time $\tau\in [0, 1]$ as $Z_{\tau}(x, z)$ and $\hat{Z}_{\tau}(x, z)$ given $Z_0=z$ and $\hat{Z}_0=z$, respectively. Note that
$$
W_2^2(\rho_{0, x}\|\hat{\rho}_x^s)\le \int \|Z_1(x, z)-\hat{Z}_1(x, z)\|_2^2\lambda(z)\mathrm{d}z,
$$
where $\lambda(\cdot)$ represents the density function of the $d_y$-dimensional standard Gaussian distribution. For $\tau\in [0, 1]$, define
$$
H_{\tau}(x)=\int \|Z_{\tau}(x, z)-\hat{Z}_{\tau}(x, z)\|_2^2\lambda(z)\mathrm{d}z.
$$
Then, it follows that
$$
\begin{aligned}
\frac{\partial}{\partial \tau}H_{\tau}(x) =& 2\int \left\langle v_0(x, Z_{\tau}(x, z), \tau)-\hat{v}^s_N(x, \hat{Z}_{\tau}(x, z), \tau), Z_{\tau}(x, z)-\hat{Z}_{\tau}(x, z)\right\rangle\lambda(z)\mathrm{d}z \\
=& 2\int \left\langle v_0(x, Z_{\tau}(x, z), \tau)-\hat{v}^s_N(x, Z_{\tau}(x, z), \tau), Z_{\tau}(x, z)-\hat{Z}_{\tau}(x, z)\right\rangle\lambda(z)\mathrm{d}z \\
& +2\int \left\langle \hat{v}^s_N(x, Z_{\tau}(x, z), \tau)-\hat{v}^s_N(x, \hat{Z}_{\tau}(x, z), \tau), Z_{\tau}(x, z)-\hat{Z}_{\tau}(x, z)\right\rangle\lambda(z)\mathrm{d}z.
\end{aligned}
$$
Specifically, we first observe that
$$
\begin{aligned}
& 2\int \left\langle v_0(x, Z_{\tau}(x, z), \tau)-\hat{v}^s_N(x, Z_{\tau}(x, z), \tau), Z_{\tau}(x, z)-\hat{Z}_{\tau}(x, z)\right\rangle\lambda(z)\mathrm{d}z \\
\le & \int \|v_0(x, Z_{\tau}(x, z), \tau)-\hat{v}^s_N(x, Z_{\tau}(x, z), \tau)\|_2^2\lambda(z)\mathrm{d}z+H_{\tau}(x) \\
=& \mathbb{E}_{Y^s_{\tau}|X^s=x}\|v_0(x, Y^s_{\tau}, \tau)-\hat{v}^s_N(x, Y^s_{\tau}, \tau)\|_2^2+H_{\tau}(x).
\end{aligned}
$$
In addition, the Lipschitz continuity of $\hat{v}^s_N$ and Cauchy-Schwarz inequality suggest that
$$
\begin{aligned}
& 2\int \left\langle \hat{v}^s_N(x, Z_{\tau}(x, z), \tau)-\hat{v}^s_N(x, \hat{Z}_{\tau}(x, z), \tau), Z_{\tau}(x, z)-\hat{Z}_{\tau}(x, z)\right\rangle\lambda(z)\mathrm{d}z \\
\le & 2\Lambda_NH_{\tau}(x).
\end{aligned}
$$
Hence, we conclude that
$$
\frac{\partial}{\partial \tau}H_{\tau}(x)\le (1+2\Lambda_N)H_{\tau}(x)+\mathbb{E}_{Y^s_{\tau}|X^s=x}\|v_0(x, Y^s_{\tau}, \tau)-\hat{v}^s_N(x, Y^s_{\tau}, \tau)\|_2^2.
$$
By Lemma \ref{lem: grownwall_inequality}, we have
$$
W_2^2(\rho_{0, x}\|\hat{\rho}_x^s)=H_1(x)\le \exp(1+2\Lambda_N)\int_0^1\mathbb{E}_{Y^s_{\tau}|X^s=x}\|v_0(x, Y^s_{\tau}, \tau)-\hat{v}^s_N(x, Y^s_{\tau}, \tau)\|_2^2\mathrm{d}\tau.
$$
Therefore, it follows that
$$
\begin{aligned}
\mathcal{E}^s=\mathbb{E}\left[W_2^2(\rho_{0, X^s}\|\hat{\rho}^s_{X^s})\right] &\le \exp(1+2\Lambda_N)\int_0^1\mathbb{E}\|v_0(X^s, Y^s_{\tau}, \tau)-\hat{v}^s_N(X^s, Y^s_{\tau}, \tau)\|_2^2\mathrm{d}\tau \\
&\le c_5N^{-\frac{2}{d_x+d_y+3}}(\log N)^{6+2\kappa}\exp\left(1+2(\log N)^{(1+\kappa)/2}\right),
\end{aligned}
$$
for $N\ge 2$. Furthermore, by Lemma 4.6, we conclude that
$$
\mathcal{E}^t\le c_6N^{-\frac{2}{d_x+d_y+3}}(\log N)^{7+2\kappa}\exp\left(1+2(\log N)^{(1+\kappa)/2}\right),
$$
where $c_6$ is a constant not depending on $N$, for $N\ge 2$. This completes the proof.
\end{proof}

\section{Proof of Theorem \ref{thm: gen_thm_11.4_gyorfi}}\label{sec: proof_of_thm_gen_thm_11.4_gyorfi}

\begin{lemma}[Lemma 11.2 in \cite{gyorfi2002distribution}]\label{lem: lem_11.2_gyorfi}
Let $V_1, \dots, V_n$ be independent and identically distributed random variables, $0\le V_i\le B$, $0<\alpha <1$, and $\nu>0$. Then,
\begin{equation}\label{eqn: lem_11.2_gyorfi}
\mathbb{P}\left\{\frac{|\frac 1n\sum_{i=1}^nV_i-\mathbb{E}(V_1)|}{\nu+\frac 1n\sum_{i=1}^nV_i+\mathbb{E}(V_1)}>\alpha\right\}\le \mathbb{P}\left\{\frac{|\frac 1n\sum_{i=1}^nV_i-\mathbb{E}(V_1)|}{\nu+\mathbb{E}(V_1)}>\alpha\right\}<\frac{B}{4\alpha^2\nu n}.
\end{equation}
\end{lemma}

\begin{remark}
If we substitute $B$, the upper bound of $V_i$, to $B_n$ varying with $n$, the right-hand side of Eqn.~\eqref{eqn: lem_11.2_gyorfi} should be modified to $B_n/(4\alpha^2\nu n)$ and the proof stays the same.
\end{remark}

\begin{lemma}\label{lem: gen_thm_11.6_gyorfi}
Let $Z, Z_1, \dots, Z_n$ be independent and identically distributed random vectors with dimensions of $d$, and let $\mathcal{H}$ be a set of nonrandom functions $h: \mathbb{R}^d\to [0, A_n]$, where $A_n>0$ is a nonrandom sequence. Assume $\alpha>0$, $0<\epsilon< 1$. Then, for $n\ge 1$, we have
\begin{equation}\label{eqn: gen_thm_11.6_gyorfi}
\mathbb{P}\left\{\sup_{h\in\mathcal{H}}\frac{\frac 1n\sum_{i=1}^nh(Z_i)-\mathbb{E}[h(Z)]}{\alpha+\frac 1n\sum_{i=1}^nh(Z_i)+\mathbb{E}[h(Z)]}>\epsilon\right\}\le 4\mathbb{E}\left[\mathcal{N}\left(\frac{\alpha\epsilon}{5}, \|\cdot \|_{\infty}, \mathcal{H}_{|\mathcal{D}_n}\right)\right]\exp\left(-\frac{3\epsilon^2\alpha n}{40A_n}\right),
\end{equation}
where $\mathcal{D}_n=\{Z_1, \dots, Z_n\}$.
\end{lemma}

\begin{remark}
Lemma \ref{lem: gen_thm_11.6_gyorfi} is a generalization of Theorem 11.6 in \cite{gyorfi2002distribution}.
\end{remark}

\begin{proof}[Proof of Lemma \ref{lem: gen_thm_11.6_gyorfi}]
The proof contains four steps.

\noindent Step 1. Substitution for the expectation by an empirical mean. Draw a pseudo-sample $\mathcal{D}_n'=\{Z_1', \dots, Z_n'\}$ as an independent copy of $\mathcal{D}_n$. Let $h^*\in\mathcal{G}$ be a function satisfying
$$
\frac 1n\sum_{i=1}^nh^*(Z_i)-\mathbb{E}[h^*(Z)]>\epsilon\left\{\alpha+\frac 1n\sum_{i=1}^nh^*(Z_i)+\mathbb{E}[h^*(Z)]\right\},
$$
if there exists any such function; otherwise, let $h^*$ be an arbitrary element of $\mathcal{H}$. We note that the conditions
$$
\frac 1n\sum_{i=1}^nh(Z_i)-\mathbb{E}[h(Z)]>\epsilon\left\{\alpha+\frac 1n\sum_{i=1}^nh(Z_i)+\mathbb{E}[h(Z)]\right\},
$$
and
$$
\frac 1n\sum_{i=1}^nh(Z_i')-\mathbb{E}[h(Z)]\le \frac{\epsilon}{4}\left\{\alpha+\frac 1n\sum_{i=1}^nh(Z_i')+\mathbb{E}[h(Z)]\right\},
$$
imply
$$
\frac 1n\sum_{i=1}^nh(Z_i)-\frac 1n\sum_{i=1}^nh(Z_i')>\frac{3\epsilon\alpha}{4}+\frac{\epsilon}{n}\sum_{i=1}^nh(Z_i)-\frac{\epsilon}{4n}\sum_{i=1}^nh(Z_i')+\frac{3\epsilon}{4}\mathbb{E}[h(Z)],
$$
which is equivalent to
$$
\begin{aligned}
& \left(1-\frac{5\epsilon}{8}\right)\left(\frac 1n\sum_{i=1}^nh(Z_i)-\frac 1n\sum_{i=1}^nh(Z_i')\right) \\
>& \frac{3\epsilon}{8}\left(2\alpha+\frac 1n\sum_{i=1}^nh(Z_i)+\frac 1n\sum_{i=1}^nh(Z_i')\right)+\frac{3\epsilon}{4}\mathbb{E}[h(Z)].
\end{aligned}
$$
We further obtain that
$$
\frac 1n\sum_{i=1}^nh(Z_i)-\frac 1n\sum_{i=1}^nh(Z_i')>\frac{3\epsilon}{8}\left(2\alpha+\frac 1n\sum_{i=1}^nh(Z_i)+\frac 1n\sum_{i=1}^nh(Z_i')\right),
$$
since $0<1-5\epsilon/8<1$ and $\mathbb{E}[h(Z)]\ge 0$. As a result, it then follows that
$$
\begin{aligned}
& \mathbb{P}\left\{\exists h\in\mathcal{H}: \frac 1n\sum_{i=1}^nh(Z_i)-\frac 1n\sum_{i=1}^nh(Z_i')>\frac{3\epsilon}{8}\left(2\alpha+\frac 1n\sum_{i=1}^nh(Z_i)+\frac 1n\sum_{i=1}^nh(Z_i')\right)\right\} \\
\ge & \mathbb{P}\left\{\frac 1n\sum_{i=1}^nh^*(Z_i)-\frac 1n\sum_{i=1}^nh^*(Z_i')>\frac{3\epsilon}{8}\left(2\alpha+\frac 1n\sum_{i=1}^nh^*(Z_i)+\frac 1n\sum_{i=1}^nh^*(Z_i')\right)\right\} \\
\ge & \mathbb{P}\Bigg(\frac 1n\sum_{i=1}^nh^*(Z_i)-\mathbb{E}[h^*(Z)|\mathcal{D}_n]>\epsilon\left\{\alpha+\frac 1n\sum_{i=1}^nh^*(Z_i)+\mathbb{E}[h^*(Z)|\mathcal{D}_n]\right\}, \\
&\hspace{2em} \frac 1n\sum_{i=1}^nh^*(Z_i')-\mathbb{E}[h^*(Z)|\mathcal{D}_n]\le \frac{\epsilon}{4}\left\{\alpha+\frac 1n\sum_{i=1}^nh^*(Z_i')+\mathbb{E}[h^*(Z)|\mathcal{D}_n]\right\}\Bigg) \\
=& \mathbb{E}\Bigg\{\mathds{1}\left(\frac 1n\sum_{i=1}^nh^*(Z_i)-\mathbb{E}[h^*(Z)|\mathcal{D}_n]>\epsilon\left\{\alpha+\frac 1n\sum_{i=1}^nh^*(Z_i)+\mathbb{E}[h^*(Z)|\mathcal{D}_n]\right\}\right), \\
&\hspace{2em} \times \mathbb{P}\left(\frac 1n\sum_{i=1}^nh^*(Z_i')-\mathbb{E}[h^*(Z)|\mathcal{D}_n]\le \frac{\epsilon}{4}\left\{\alpha+\frac 1n\sum_{i=1}^nh^*(Z_i')+\mathbb{E}[h^*(Z)|\mathcal{D}_n]\right\}\Bigg | \mathcal{D}_n\right)\Bigg\}.
\end{aligned}
$$
Lemma \ref{lem: lem_11.2_gyorfi} yields that
\begin{equation}\label{eqn: gen_thm_11.6_gyorfi_proof_1}
\begin{aligned}
& \mathbb{P}\left(\frac 1n\sum_{i=1}^nh^*(Z_i')-\mathbb{E}[h^*(Z)|\mathcal{D}_n]>\frac{\epsilon}{4}\left\{\alpha+\frac 1n\sum_{i=1}^nh^*(Z_i')+\mathbb{E}[h^*(Z)|\mathcal{D}_n]\right\}\Bigg | \mathcal{D}_n\right) \\
<& \frac{A_n}{4(\epsilon/4)^2\alpha n}=\frac{4A_n}{\epsilon^2\alpha n}.
\end{aligned}
\end{equation}

Therefore, for $n>8A_n/(\epsilon^2\alpha)$, the probability in Eqn.~\eqref{eqn: gen_thm_11.6_gyorfi_proof_1} is no less than $1/2$, and we conclude that
$$
\begin{aligned}
& \mathbb{P}\left\{\exists h\in\mathcal{H}: \frac 1n\sum_{i=1}^nh(Z_i)-\frac 1n\sum_{i=1}^nh(Z_i')>\frac{3\epsilon}{8}\left(2\alpha+\frac 1n\sum_{i=1}^nh(Z_i)+\frac 1n\sum_{i=1}^nh(Z_i')\right)\right\} \\
\ge & \frac 12\mathbb{P}\left(\frac 1n\sum_{i=1}^nh^*(Z_i)-\mathbb{E}[h^*(Z)|\mathcal{D}_n]>\epsilon\left\{\alpha+\frac 1n\sum_{i=1}^nh^*(Z_i)+\mathbb{E}[h^*(Z)|\mathcal{D}_n]\right\}\right) \\
= & \frac 12\mathbb{P}\left(\exists h: \frac 1n\sum_{i=1}^nh(Z_i)-\mathbb{E}[h(Z)]>\epsilon\left\{\alpha+\frac 1n\sum_{i=1}^nh(Z_i)+\mathbb{E}[h(Z)]\right\}\right).
\end{aligned}
$$
This proves
$$
\begin{aligned}
& \mathbb{P}\left(\exists h: \frac{\frac 1n\sum_{i=1}^nh(Z_i)-\mathbb{E}[h(Z)]}{\alpha+\frac 1n\sum_{i=1}^nh(Z_i)+\mathbb{E}[h(Z)]}>\epsilon\right) \\
\le & 2\mathbb{P}\left\{\exists h\in\mathcal{H}: \frac 1n\sum_{i=1}^nh(Z_i)-\frac 1n\sum_{i=1}^nh(Z_i')>\frac{3\epsilon}{8}\left(2\alpha+\frac 1n\sum_{i=1}^nh(Z_i)+\frac 1n\sum_{i=1}^nh(Z_i')\right)\right\},
\end{aligned}
$$
when $n>8A_n/(\epsilon^2\alpha)$. For $n\le 8A_n/(\epsilon^2\alpha)$, on the other hand, the right-hand side of Eqn.~\eqref{eqn: gen_thm_11.6_gyorfi} exceeds one, and hence the assertion holds true trivially.

\noindent Step 2. Introduction of Rademacher random variables. Let $U_1, \dots, U_n$ be independent Rademacher random variables which are uniformly distributed over $\{-1, 1\}$, meanwhile independent of $\mathcal{D}_n\cup \mathcal{D}_n'$. Importantly, note that $\mathcal{D}_n$ and $\mathcal{D}_n'$ are interchangeable with respect to corresponding components while their joint distribution remains invariant. Therefore, we have
$$
\begin{aligned}
& \mathbb{P}\left\{\exists h\in\mathcal{H}: \frac 1n\sum_{i=1}^nh(Z_i)-\frac 1n\sum_{i=1}^nh(Z_i')>\frac{3\epsilon}{8}\left(2\alpha+\frac 1n\sum_{i=1}^nh(Z_i)+\frac 1n\sum_{i=1}^nh(Z_i')\right)\right\} \\
=& \mathbb{P}\left\{\exists h\in\mathcal{H}: \frac 1n\sum_{i=1}^nU_i[h(Z_i)-h(Z_i')]>\frac{3\epsilon}{8}\left(2\alpha+\frac 1n\sum_{i=1}^n[h(Z_i)+h(Z_i')]\right)\right\} \\
\le & \mathbb{P}\left\{\exists h\in\mathcal{H}: \frac 1n\sum_{i=1}^nU_ih(Z_i)>\frac{3\epsilon}{8}\left(\alpha+\frac 1n\sum_{i=1}^nh(Z_i)\right)\right\} \\
& \hspace{3em} +\mathbb{P}\left\{\exists h\in\mathcal{H}: \frac 1n\sum_{i=1}^nU_ih(Z_i')<-\frac{3\epsilon}{8}\left(\alpha+\frac 1n\sum_{i=1}^nh(Z_i')\right)\right\} \\
\le & 2\mathbb{P}\left\{\exists h\in\mathcal{H}: \frac 1n\sum_{i=1}^nU_ih(Z_i)>\frac{3\epsilon}{8}\left(\alpha+\frac 1n\sum_{i=1}^nh(Z_i)\right)\right\}.
\end{aligned}
$$
Here, we note that $-U_i$ is identically distributed as $U_i$.

\noindent Step 3. Conditioning and covering. Given $Z_i=z_i$ for $i=1, \dots, n$, and consider
$$
\mathbb{P}\left\{\exists h\in\mathcal{H}: \frac 1n\sum_{i=1}^nU_ih(z_i)>\frac{3\epsilon}{8}\left(\alpha+\frac 1n\sum_{i=1}^nh(z_i)\right)\right\}.
$$
Let $\delta>0$ and let $\mathcal{C}_{\delta}$ be a $\delta$-covering set of $\mathcal{H}$ constrained on $\{z_1, \dots, z_n\}$ with respect to the supremum norm. For any $h\in\mathcal{H}$, there exists a vector $h^{\sharp}=(\bar{h}(z_1), \dots, \bar{h}(z_n))^{\top}\in \mathcal{C}_{\delta}$, such that $\max_{i=1, \dots, n}|h(z_i)-\bar{h}(z_i)|<\epsilon$, thereby indicating
$$
\begin{aligned}
\frac 1n\sum_{i=1}^nU_ih(z_i) &= \frac 1n\sum_{i=1}^nU_i\bar{h}(z_i)+\frac 1n\sum_{i=1}^nU_i[h(z_i)-\bar{h}(z_i)] \\
&\le \frac 1n\sum_{i=1}^nU_i\bar{h}(z_i)+\max_{i=1, \dots, n}|h(z_i)-\bar{h}(z_i)|\le \frac 1n\sum_{i=1}^nU_i\bar{h}(z_i)+\delta,
\end{aligned}
$$
and
$$
\frac 1n\sum_{i=1}^nh(z_i)\ge \frac 1n\sum_{i=1}^n\bar{h}(z_i)-\frac 1n\sum_{i=1}^n|h(z_i)-\bar{h}(z_i)|\ge \frac 1n\sum_{i=1}^n\bar{h}(z_i)-\delta.
$$
As a result, we have
$$
\begin{aligned}
& \mathbb{P}\left\{\exists h\in\mathcal{H}: \frac 1n\sum_{i=1}^nU_ih(z_i)>\frac{3\epsilon}{8}\left(\alpha+\frac 1n\sum_{i=1}^nh(z_i)\right)\right\} \\
\le & \mathbb{P}\left\{\exists h^{\sharp}\in\mathcal{C}_{\delta}: \frac 1n\sum_{i=1}^nU_i\bar{h}(z_i)+\delta>\frac{3\epsilon}{8}\left(\alpha+\frac 1n\sum_{i=1}^n\bar{h}(z_i)-\delta\right)\right\} \\
\le & |\mathcal{C}_{\delta}|\max_{h^{\sharp}\in \mathcal{C}_{\delta}}\mathbb{P}\left\{\frac 1n\sum_{i=1}^nU_i\bar{h}(z_i)>\frac{3\epsilon\alpha}{8}-\frac{3\epsilon\delta}{8}-\delta+\frac{3\epsilon}{8n}\sum_{i=1}^n\bar{h}(z_i)\right\}.
\end{aligned}
$$
Specifying $\delta=\epsilon\alpha/5$ deduces that
$$
\frac{3\epsilon\alpha}{8}-\frac{3\epsilon\delta}{8}-\delta=\frac{3\epsilon\alpha}{8}-\frac{3\epsilon\alpha}{40}-\frac{\epsilon\alpha}{5}=\frac{\epsilon\alpha}{10}.
$$
By choosing $\mathcal{C}_{\epsilon\alpha/5}$ as an $\epsilon\alpha/5$-covering set of minimal size, we obtain
$$
\begin{aligned}
& \mathbb{P}\left\{\exists h\in\mathcal{H}: \frac 1n\sum_{i=1}^nU_ih(z_i)>\frac{3\epsilon}{8}\left(\alpha+\frac 1n\sum_{i=1}^nh(z_i)\right)\right\} \\
\le & \mathcal{N}\left(\frac{\epsilon\alpha}{5}, \|\cdot\|_{\infty}, \mathcal{H}_{|\{z_1, \dots, z_n\}}\right)\max_{h^{\sharp}\in\mathcal{C}_{\epsilon\alpha/5}}\mathbb{P}\left\{\frac 1n\sum_{i=1}^nU_i\bar{h}(z_i)>\frac{\epsilon\alpha}{10}+\frac{3\epsilon}{8n}\sum_{i=1}^n\bar{h}(z_i)\right\}.
\end{aligned}
$$

\noindent Step 4. Leveraging the Hoeffding's inequality. Note that given fixed $z_1, \dots, z_n$, $U_1\bar{h}(z_1), \dots, U_n\bar{h}(z_n)$ are independent random variables with mean zero and absolute bound $\bar{h}(z_1), \dots, \bar{h}(z_n)$ (recall that $\bar{h}\in [0, A_n]$). Therefore, Hoeffding's inequality suggests that
$$
\begin{aligned}
\mathbb{P}\left\{\frac 1n\sum_{i=1}^nU_i\bar{h}(z_i)>\frac{\epsilon\alpha}{10}+\frac{3\epsilon}{8n}\sum_{i=1}^n\bar{h}(z_i)\right\}&\le \exp\left\{-\frac{2n^2[\frac{\epsilon\alpha}{10}+\frac{3\epsilon}{8n}\sum_{i=1}^n\bar{h}(z_i)]^2}{4\sum_{i=1}^n\bar{h}(z_i)^2}\right\} \\
&\le \exp\left\{-\frac{2n^2[\frac{\epsilon\alpha}{10}+\frac{3\epsilon}{8n}\sum_{i=1}^n\bar{h}(z_i)]^2}{4A_n\sum_{i=1}^n\bar{h}(z_i)}\right\} \\
&= \exp\left\{-\frac{9\epsilon^2}{128A_n}\frac{[\frac{4\alpha }{15}+\sum_{i=1}^n\bar{h}(z_i)]^2}{\sum_{i=1}^n\bar{h}(z_i)}\right\}.
\end{aligned}
$$
Note that for any $a, y>0$, $(a+y)^2/y\ge 4a$, implying
$$
\frac{[\frac{4\alpha n}{15}+\sum_{i=1}^n\bar{h}(z_i)]^2}{\sum_{i=1}^n\bar{h}(z_i)}\ge \frac{16\alpha n}{15}.
$$
Hence, it concludes that
$$
\begin{aligned}
\mathbb{P}\left\{\frac 1n\sum_{i=1}^nU_i\bar{h}(z_i)>\frac{\epsilon\alpha}{10}+\frac{3\epsilon}{8n}\sum_{i=1}^n\bar{h}(z_i)\right\} &\le \exp\left\{-\frac{9\epsilon^2}{128A_n}\frac{[\frac{4\alpha n}{15}+\sum_{i=1}^n\bar{h}(z_i)]^2}{\sum_{i=1}^n\bar{h}(z_i)}\right\} \\
&\le  \exp\left\{-\frac{3\epsilon^2\alpha n}{40A_n}\right\},
\end{aligned}
$$
which completes the proof.
\end{proof}

\begin{proof}[Proof of Theorem \ref{thm: gen_thm_11.4_gyorfi}]

The proof is composed of six steps.

\noindent Step 1. Symmetrization. We commence by replacing $\mathbb{E}[g(f, Z)]$ through an empirical mean deduced by a pseudo-sample $\mathcal{D}_n'=\{Z_1', \dots,  Z_n'\}$ independent of $\mathcal{D}_n$. Consider a function $f^*\in\mathcal{F}_n$ depending on $\mathcal{D}_n$, such that
$$
\mathbb{E}[g(f^*, Z)|\mathcal{D}_n]-\frac 1n\sum_{i=1}^ng(f^*, Z_i)\ge \epsilon (\alpha+\beta)+\epsilon \mathbb{E}[g(f^*, Z)|\mathcal{D}_n],
$$
if such a function exists; otherwise, we let $f^*$ be an arbitrary element in $\mathcal{F}_n$. Then, Chebyshev's inequality implies
$$
\begin{aligned}
& \mathbb{P}\left\{\mathbb{E}[g(f^*, Z)|\mathcal{D}_n]-\frac 1n\sum_{i=1}^ng(f^*, Z_i')> \frac{\epsilon}{2} (\alpha+\beta)+\frac{\epsilon}{2} \mathbb{E}[g(f^*, Z)|\mathcal{D}_n]\Bigg| \mathcal{D}_n\right\} \\
\le & \frac{\mathrm{Var}[g(f^*, Z)|\mathcal{D}_n]}{n\{\frac{\epsilon}{2} (\alpha+\beta)+\frac{\epsilon}{2} \mathbb{E}[g(f^*, Z)|\mathcal{D}_n]\}^2} \\
\le & \frac{\zeta_n\mathbb{E}[g(f^*, Z)|\mathcal{D}_n]}{n\{\frac{\epsilon}{2} (\alpha+\beta)+\frac{\epsilon}{2} \mathbb{E}[g(f^*, Z)|\mathcal{D}_n]\}^2} \\
\le & \frac{\zeta_n}{\epsilon^2(\alpha+\beta)n},
\end{aligned}
$$
where the last inequality stems from $x/(a+x)^2\le 1/(4a)$ for $x\ge 0$ and $a>0$. Thus, for $n>8\zeta_n/[\epsilon^2(\alpha+\beta)]$, we have
$$
\mathbb{P}\left\{\mathbb{E}[g(f^*, Z)|\mathcal{D}_n]-\frac 1n\sum_{i=1}^ng(f^*, Z_i')\le \frac{\epsilon}{2} (\alpha+\beta)+\frac{\epsilon}{2} \mathbb{E}[g(f^*, Z)|\mathcal{D}_n]\Bigg| \mathcal{D}_n\right\}>\frac 78,
$$
yielding that
$$
\begin{aligned}
& \mathbb{P}\left\{\exists f\in\mathcal{F}_n: \frac 1n\sum_{i=1}^ng(f, Z_i')-\frac 1n\sum_{i=1}^ng(f, Z_i)\ge \frac{\epsilon}{2}(\alpha+\beta)+\frac{\epsilon}{2}\mathbb{E}[g(f, Z)]\right\} \\
\ge & \mathbb{P}\left\{\frac 1n\sum_{i=1}^ng(f^*, Z_i')-\frac 1n\sum_{i=1}^ng(f^*, Z_i)\ge \frac{\epsilon}{2}(\alpha+\beta)+\frac{\epsilon}{2}\mathbb{E}[g(f^*, Z)|\mathcal{D}_n]\right\} \\
\ge & \mathbb{P}\Bigg\{\mathbb{E}[g(f^*, Z)|\mathcal{D}_n]-\frac 1n\sum_{i=1}^ng(f^*, Z_i)\ge \epsilon(\alpha+\beta)+\epsilon\mathbb{E}[g(f^*, Z)|\mathcal{D}_n], \\
& \hspace{3em} \mathbb{E}[g(f^*, Z)|\mathcal{D}_n]-\frac 1n\sum_{i=1}^ng(f^*, Z_i')\le \frac{\epsilon}{2}(\alpha+\beta)+\frac{\epsilon}{2}\mathbb{E}[g(f^*, Z)|\mathcal{D}_n]\Bigg\} \\
=& \mathbb{E}\Bigg(\mathds{1}\left(\mathbb{E}[g(f^*, Z)|\mathcal{D}_n]-\frac 1n\sum_{i=1}^ng(f^*, Z_i)\ge \epsilon(\alpha+\beta)+\epsilon\mathbb{E}[g(f^*, Z)|\mathcal{D}_n]\right) \\
& \hspace{3em} \mathbb{P}\left\{\mathbb{E}[g(f^*, Z)|\mathcal{D}_n]-\frac 1n\sum_{i=1}^ng(f^*, Z_i')\le \frac{\epsilon}{2}(\alpha+\beta)+\frac{\epsilon}{2}\mathbb{E}[g(f^*, Z)|\mathcal{D}_n]\right\}\Bigg) \\
\ge & \frac 78\mathbb{P}\left\{\mathbb{E}[g(f^*, Z)|\mathcal{D}_n]-\frac 1n\sum_{i=1}^ng(f^*, Z_i)\ge \epsilon(\alpha+\beta)+\epsilon\mathbb{E}[g(f^*, Z)|\mathcal{D}_n]\right\} \\
=& \frac 78\mathbb{P}\left\{\exists f\in\mathcal{F}_n: \mathbb{E}[g(f, Z)]-\frac 1n\sum_{i=1}^ng(f, Z_i)\ge \epsilon(\alpha+\beta)+\epsilon\mathbb{E}[g(f, Z)]\right\}.
\end{aligned}
$$
To conclude, for $n>8\zeta_n/[\epsilon^2(\alpha+\beta)]$, we have
$$
\begin{aligned}
& \mathbb{P}\left\{\exists f\in\mathcal{F}_n: \mathbb{E}[g(f, Z)]-\frac 1n\sum_{i=1}^ng(f, Z_i)\ge \epsilon(\alpha+\beta)+\epsilon\mathbb{E}[g(f, Z)]\right\} \\
\le & \frac 87 \mathbb{P}\left\{\exists f\in\mathcal{F}_n: \frac 1n\sum_{i=1}^ng(f, Z_i')-\frac 1n\sum_{i=1}^ng(f, Z_i)\ge \frac{\epsilon}{2}(\alpha+\beta)+\frac{\epsilon}{2}\mathbb{E}[g(f, Z)]\right\}.
\end{aligned}
$$

\noindent Step 2. Randomization for $\mathbb{E}[g(f, Z)]$. By introducing additional conditions, we notice that
\begin{equation}\label{eqn: gen_thm_11.4_gyorfi_proof_1}
\begin{aligned}
& \mathbb{P}\left\{\exists f\in\mathcal{F}_n: \frac 1n\sum_{i=1}^ng(f, Z_i')-\frac 1n\sum_{i=1}^ng(f, Z_i)\ge \frac{\epsilon}{2}(\alpha+\beta)+\frac{\epsilon}{2}\mathbb{E}[g(f, Z)]\right\} \\
\le & \mathbb{P}\Bigg(\exists f\in\mathcal{F}_n: \frac 1n\sum_{i=1}^ng(f, Z_i')-\frac 1n\sum_{i=1}^ng(f, Z_i)\ge \frac{\epsilon}{2}(\alpha+\beta)+\frac{\epsilon}{2}\mathbb{E}[g(f, Z)], \\
& \hspace{3em} \frac 1n\sum_{i=1}^ng(f, Z_i)^2-\mathbb{E}[g(f, Z)^2]\le \epsilon\left\{\alpha+\beta+\frac 1n\sum_{i=1}^ng(f, Z_i)^2+\mathbb{E}[g(f, Z)^2]\right\} \\
& \hspace{3em} \frac 1n\sum_{i=1}^ng(f, Z_i')^2-\mathbb{E}[g(f, Z)^2]\le \epsilon\left\{\alpha+\beta+\frac 1n\sum_{i=1}^ng(f, Z_i')^2+\mathbb{E}[g(f, Z)^2]\right\}\Bigg) \\
& +2\mathbb{P}\left\{\exists f\in \mathcal{F}_n: \frac{\frac 1n\sum_{i=1}^ng(f, Z_i)^2-\mathbb{E}[g(f, Z)^2]}{\alpha+\beta+\frac 1n\sum_{i=1}^ng(f, Z_i)^2+\mathbb{E}[g(f, Z)^2]}>\epsilon\right\}.
\end{aligned}
\end{equation}
Then, Lemma \ref{lem: gen_thm_11.6_gyorfi} verifies that
$$
\begin{aligned}
& \mathbb{P}\left\{\exists f\in \mathcal{F}_n: \frac{\frac 1n\sum_{i=1}^ng(f, Z_i)^2-\mathbb{E}[g(f, Z)^2]}{\alpha+\beta+\frac 1n\sum_{i=1}^ng(f, Z_i)^2+\mathbb{E}[g(f, Z)^2]}>\epsilon\right\} \\
\le & 4\mathbb{E}\left[\mathcal{N}\left(\frac{(\alpha+\beta)\epsilon}{5}, \|\cdot \|_{\infty}, \{g(f, \cdot): \mathcal{Z}\to\mathbb{R}, f\in\mathcal{F}_n\}_{|\mathcal{D}_n}\right)\right]\exp\left(-\frac{3\epsilon^2(\alpha+\beta) n}{40\xi_n^2}\right).
\end{aligned}
$$
Next, we focus on the first probability on the right-hand side of Eqn.~\eqref{eqn: gen_thm_11.4_gyorfi_proof_1}. The second inequality inside the probability demonstrates that
$$
(1+\epsilon)\mathbb{E}[g(f, Z)^2]\ge (1-\epsilon)\frac 1n\sum_{i=1}^ng(f, Z_i)^2-\epsilon(\alpha+\beta),
$$
which equals to
$$
\frac{1}{2\zeta_n}\mathbb{E}[g(f, Z)^2]\ge \frac{1-\epsilon}{2\zeta_n(1+\epsilon)}\frac 1n\sum_{i=1}^ng(f, Z_i)^2-\frac{\epsilon(\alpha+\beta)}{2\zeta_n(1+\epsilon)},
$$
while the third inequality is processed in the same manner. By the assumption that $\mathbb{E}[g(f, Z)^2]\le \zeta_n\mathbb{E}[g(f, Z)]$ for all $f\in\mathcal{F}_n$, the first probability on the right-hand side of Eqn.~\eqref{eqn: gen_thm_11.4_gyorfi_proof_1} is bounded by
$$
\begin{aligned}
& \mathbb{P}\Bigg\{\exists f\in \mathcal{F}_n: \frac 1n\sum_{i=1}^ng(f, Z_i')-\frac 1n\sum_{i=1}^ng(f, Z_i)\ge \frac{\epsilon}{2}(\alpha+\beta) \\
& \hspace{2em} +\frac{\epsilon}{2}\left[\frac{1-\epsilon}{2\zeta_n(1+\epsilon)}\frac 1n\sum_{i=1}^ng(f, Z_i)^2-\frac{\epsilon(\alpha+\beta)}{2\zeta_n(1+\epsilon)}+\frac{1-\epsilon}{2\zeta_n(1+\epsilon)}\frac 1n\sum_{i=1}^ng(f, Z_i')^2-\frac{\epsilon(\alpha+\beta)}{2\zeta_n(1+\epsilon)}\right]\Bigg\}.
\end{aligned}
$$
This shows
\begin{equation}\label{eqn: gen_thm_11.4_gyorfi_proof_2}
\begin{aligned}
& \mathbb{P}\left\{\exists f\in\mathcal{F}_n: \frac 1n\sum_{i=1}^ng(f, Z_i')-\frac 1n\sum_{i=1}^ng(f, Z_i)\ge \frac{\epsilon}{2}(\alpha+\beta)+\frac{\epsilon}{2}\mathbb{E}[g(f, Z)]\right\} \\
\le & \mathbb{P}\Bigg\{\exists f\in \mathcal{F}_n: \frac 1n\sum_{i=1}^n[g(f, Z_i')-g(f, Z_i)]\ge \frac{\epsilon}{2}(\alpha+\beta)-\frac{\epsilon^2(\alpha+\beta)}{2\zeta_n(1+\epsilon)} \\
& \hspace{15em} +\frac{\epsilon(1-\epsilon)}{4\zeta_n(1+\epsilon)n}\sum_{i=1}^n\left[g(f, Z_i)^2+g(f, Z_i')^2\right]\Bigg\} \\
& +8\mathbb{E}\left[\mathcal{N}\left(\frac{(\alpha+\beta)\epsilon}{5}, \|\cdot \|_{\infty}, \{g(f, \cdot): \mathcal{Z}\to\mathbb{R}, f\in\mathcal{F}_n\}_{|\mathcal{D}_n}\right)\right]\exp\left(-\frac{3\epsilon^2(\alpha+\beta) n}{40\xi_n^2}\right).
\end{aligned}
\end{equation}

\noindent Step 3. Introduction of Rademacher random variables. Let $U_1, \dots, U_n$ be independent Rademacher random variables which are uniformly distributed over $\{-1, 1\}$, meanwhile independent of $\mathcal{D}_n\cup \mathcal{D}_n'$. We note that $\mathcal{D}_n$ and $\mathcal{D}_n'$ are interchangeable with respect to corresponding components while their joint distribution remains invariant. As a consequence, the first probability on the right-hand side of Eqn.~\eqref{eqn: gen_thm_11.4_gyorfi_proof_2} is equivalent to
$$
\begin{aligned}
& \mathbb{P}\Bigg\{\exists f\in \mathcal{F}_n: \frac 1n\sum_{i=1}^nU_i[g(f, Z_i')-g(f, Z_i)]\ge \frac{\epsilon}{2}(\alpha+\beta)-\frac{\epsilon^2(\alpha+\beta)}{2\zeta_n(1+\epsilon)} \\
& \hspace{15em} +\frac{\epsilon(1-\epsilon)}{4\zeta_n(1+\epsilon)n}\sum_{i=1}^n\left[g(f, Z_i)^2+g(f, Z_i')^2\right]\Bigg\},
\end{aligned}
$$
which is further bounded by
$$
\begin{aligned}
& \mathbb{P}\Bigg\{\exists f\in \mathcal{F}_n: \left|\frac 1n\sum_{i=1}^nU_ig(f, Z_i')\right|\ge \frac{\epsilon}{4}(\alpha+\beta)-\frac{\epsilon^2(\alpha+\beta)}{4\zeta_n(1+\epsilon)}+\frac{\epsilon(1-\epsilon)}{4\zeta_n(1+\epsilon)n}\sum_{i=1}^ng(f, Z_i')^2\Bigg\} \\
&+ \mathbb{P}\Bigg\{\exists f\in \mathcal{F}_n: \left|\frac 1n\sum_{i=1}^nU_ig(f, Z_i)\right|\ge \frac{\epsilon}{4}(\alpha+\beta)-\frac{\epsilon^2(\alpha+\beta)}{4\zeta_n(1+\epsilon)}+\frac{\epsilon(1-\epsilon)}{4\zeta_n(1+\epsilon)n}\sum_{i=1}^ng(f, Z_i)^2\Bigg\} \\
=& 2\mathbb{P}\Bigg\{\exists f\in \mathcal{F}_n: \left|\frac 1n\sum_{i=1}^nU_ig(f, Z_i)\right|\ge \frac{\epsilon}{4}(\alpha+\beta)-\frac{\epsilon^2(\alpha+\beta)}{4\zeta_n(1+\epsilon)}+\frac{\epsilon(1-\epsilon)}{4\zeta_n(1+\epsilon)n}\sum_{i=1}^ng(f, Z_i)^2\Bigg\}.
\end{aligned}
$$

\noindent Step 4. Conditioning and Covering. Given $Z_i=z_i$ for $i=1, \dots, n$, consider
$$
\mathbb{P}\Bigg\{\exists f\in \mathcal{F}_n: \left|\frac 1n\sum_{i=1}^nU_ig(f, z_i)\right|\ge \frac{\epsilon}{4}(\alpha+\beta)-\frac{\epsilon^2(\alpha+\beta)}{4\zeta_n(1+\epsilon)}+\frac{\epsilon(1-\epsilon)}{4\zeta_n(1+\epsilon)n}\sum_{i=1}^ng(f, z_i)^2\Bigg\}.
$$
Let $\delta>0$ and let $\mathcal{C}_{\delta}$ be a $\delta$-covering set of $\{g(f, \cdot): \mathcal{Z}\to\mathbb{R}, f\in\mathcal{F}_n\}$ constrained on $\{z_1, \dots, z_n\}$ with respect to the supremum norm. For any $f\in\mathcal{F}$, there exists a vector $h^{\sharp}=(h(z_1), \dots, h(z_n))^{\top}\in\mathcal{C}_{\delta}$, such that $\max_{i=1, \dots, n}|g(f, z_i)-h(z_i)|<\epsilon$, thereby indicating
$$
\begin{aligned}
\left|\frac 1n\sum_{i=1}^nU_ig(f, z_i)\right| &= \left|\frac 1n\sum_{i=1}^nU_ih(z_i)+\frac 1n\sum_{i=1}^nU_i\left[g(f, z_i)-h(z_i)\right]\right| \\
&\le \left|\frac 1n\sum_{i=1}^nU_ih(z_i)\right|+\frac 1n\sum_{i=1}^n|g(f, z_i)-h(z_i)| \\
&< \left|\frac 1n\sum_{i=1}^nU_ih(z_i)\right|+\delta,
\end{aligned}
$$
and
$$
\begin{aligned}
\frac 1n\sum_{i=1}^ng(f, z_i)^2 &= \frac 1n\sum_{i=1}^nh(z_i)^2+\frac 1n\sum_{i=1}^n\left[g(f, z_i)^2-h(z_i)^2\right] \\
&= \frac 1n\sum_{i=1}^nh(z_i)^2+\frac 1n\sum_{i=1}^n\left[g(f, z_i)-h(z_i)\right]\left[g(f, z_i)+h(z_i)\right] \\
&\ge \frac 1n\sum_{i=1}^nh(z_i)^2-\frac{2\xi_n}{n}\sum_{i=1}^n|g(f, z_i)-h(z_i)| \\
&> \frac 1n\sum_{i=1}^nh(z_i)^2-2\delta\xi_n.
\end{aligned}
$$
Hence, it follows that
$$
\begin{aligned}
& \mathbb{P}\Bigg\{\exists f\in \mathcal{F}_n: \left|\frac 1n\sum_{i=1}^nU_ig(f, z_i)\right|\ge \frac{\epsilon}{4}(\alpha+\beta)-\frac{\epsilon^2(\alpha+\beta)}{4\zeta_n(1+\epsilon)}+\frac{\epsilon(1-\epsilon)}{4\zeta_n(1+\epsilon)n}\sum_{i=1}^ng(f, z_i)^2\Bigg\} \\
\le & \mathbb{P}\Bigg\{\exists h^{\sharp}\in \mathcal{C}_{\delta}: \left|\frac 1n\sum_{i=1}^nU_ih(z_i)\right|+\delta\ge \frac{\epsilon}{4}(\alpha+\beta)-\frac{\epsilon^2(\alpha+\beta)}{4\zeta_n(1+\epsilon)} \\
& \hspace{22em} +\frac{\epsilon(1-\epsilon)}{4\zeta_n(1+\epsilon)}\left[\frac 1n\sum_{i=1}^nh(z_i)^2-2\delta\xi_n\right]\Bigg\} \\
\le & |\mathcal{C}_{\delta}|\max_{h^{\sharp}\in\mathcal{C}_{\delta}}\mathbb{P}\Bigg\{\left|\frac 1n\sum_{i=1}^nU_ih(z_i)\right|\ge \frac{\epsilon}{4}(\alpha+\beta)-\frac{\epsilon^2(\alpha+\beta)}{4\zeta_n(1+\epsilon)}-\delta-\frac{\epsilon(1-\epsilon)\delta\xi_n}{2\zeta_n(1+\epsilon)} \\
& \hspace{22em} +\frac{\epsilon(1-\epsilon)}{4\zeta_n(1+\epsilon)n}\sum_{i=1}^nh(z_i)^2\Bigg\}.
\end{aligned}
$$
Next we set $\delta=(6\eta-2)\epsilon\beta/(30\eta+3\gamma\eta)=\varpi\epsilon\beta$. Then, when $n\ge N$, we have $\zeta_n\ge \eta$, $\xi_n\le \gamma\zeta_n$, and for $0<\epsilon\le 1/2$,
$$
\begin{aligned}
\frac{\epsilon\beta}{4}-\frac{\epsilon^2\beta}{4\zeta_n(1+\epsilon)}-\delta-\frac{\epsilon(1-\epsilon)\delta\xi_n}{2\zeta_n(1+\epsilon)} &= \epsilon\beta\left(\frac 14-\frac{\epsilon}{4\zeta_n(1+\epsilon)}-\varpi-\frac{\epsilon(1-\epsilon)\varpi\xi_n}{2\zeta_n(1+\epsilon)}\right) \\
&\ge \epsilon\beta\left(\frac 14-\frac{1}{12\zeta_n}-\varpi-\frac{\varpi\xi_n}{10\zeta_n}\right) \\
&\ge \epsilon\beta\left(\frac 14-\frac{1}{12\eta}-\varpi-\frac{\varpi\gamma}{10}\right)\ge 0.
\end{aligned}
$$
Therefore, it holds that
$$
\begin{aligned}
& \mathbb{P}\Bigg\{\exists f\in \mathcal{F}_n: \left|\frac 1n\sum_{i=1}^nU_ig(f, z_i)\right|\ge \frac{\epsilon}{4}(\alpha+\beta)-\frac{\epsilon^2(\alpha+\beta)}{4\zeta_n(1+\epsilon)}+\frac{\epsilon(1-\epsilon)}{4\zeta_n(1+\epsilon)n}\sum_{i=1}^ng(f, z_i)^2\Bigg\} \\
\le & |\mathcal{C}_{\varpi\epsilon\beta}|\max_{h^{\sharp}\in\mathcal{C}_{\varpi\epsilon\beta}}\mathbb{P}\Bigg\{\left|\frac 1n\sum_{i=1}^nU_ih(z_i)\right|\ge \frac{\epsilon\alpha}{4}-\frac{\epsilon^2\alpha}{4\zeta_n(1+\epsilon)}+\frac{\epsilon(1-\epsilon)}{4\zeta_n(1+\epsilon)n}\sum_{i=1}^nh(z_i)^2\Bigg\}.
\end{aligned}
$$

\noindent Step 5. Leveraging the Bernstein's inequality. Firstly, we note that
$$
\frac 1n\sum_{i=1}^n\mathrm{Var}[U_ih(z_i)]=\frac 1n\sum_{i=1}^nh(z_i)^2\mathrm{Var}(U_i)=\frac 1n\sum_{i=1}^nh(z_i)^2.
$$
Hence, we have
$$
\begin{aligned}
& \mathbb{P}\Bigg\{\left|\frac 1n\sum_{i=1}^nU_ih(z_i)\right|\ge \frac{\epsilon\alpha}{4}-\frac{\epsilon^2\alpha}{4\zeta_n(1+\epsilon)}+\frac{\epsilon(1-\epsilon)}{4\zeta_n(1+\epsilon)n}\sum_{i=1}^nh(z_i)^2\Bigg\} \\
=& \mathbb{P}\left(\left|\frac 1n\sum_{i=1}^nV_i\right|\ge A_1+A_2\sigma^2\right),
\end{aligned}
$$
where
$$
\begin{aligned}
V_i &= U_ih(z_i), \quad \sigma^2=\frac 1n\sum_{i=1}^n\mathrm{Var}[U_ih(z_i)] \\
A_1 &= \frac{\epsilon\alpha}{4}-\frac{\epsilon^2\alpha}{4\zeta_n(1+\epsilon)}, \quad A_2=\frac{\epsilon(1-\epsilon)}{4\zeta_n(1+\epsilon)}.
\end{aligned}
$$
Observe that $V_1, \dots, V_n$ are independent random variables satisfying $|V_i|\le |h(z_i)|\le \xi_n (i=1, \dots, n)$, and that $A_1, A_2>0$ for $n\ge N$. By Bernstein's inequality, we have
$$
\begin{aligned}
\mathbb{P}\left(\left|\frac 1n\sum_{i=1}^nV_i\right|\ge A_1+A_2\sigma^2\right) &\le 2\exp\left(-\frac{n(A_1+A_2\sigma^2)^2}{2\sigma^2+2(A_1+A_2\sigma^2)\frac{\xi_n}{3}}\right) \\
&= 2\exp\left(-\frac{nA_2^2}{\frac 23\xi_nA_2}\cdot\frac{(\frac{A_1}{A_2}+\sigma^2)^2}{\frac{A_1}{A_2}+(1+\frac{3}{\xi_nA_2})\sigma^2}\right) \\
&= 2\exp\left(-\frac{3nA_2}{2\xi_n}\cdot\frac{(\frac{A_1}{A_2}+\sigma^2)^2}{\frac{A_1}{A_2}+(1+\frac{3}{\xi_nA_2})\sigma^2}\right).
\end{aligned}
$$
It is easy to verify that for arbitrary $a, b, u>0$, it follows that
$$
\frac{(a+u)^2}{a+bu}\ge \frac{4a}{b^2}[(b-1)\vee 0].
$$
Then, by letting $a=A_1/A_2$, $b=1+3/(\xi_nA_2), u=\sigma^2$, we obtain
$$
\frac{3nA_2}{2\xi_n}\cdot\frac{(\frac{A_1}{A_2}+\sigma^2)^2}{\frac{A_1}{A_2}+(1+\frac{3}{\xi_nA_2})\sigma^2}\ge \frac{3nA_2}{2\xi_n}\cdot\frac{4\frac{A_1}{A_2}}{(1+\frac{3}{\xi_nA_2})^2}\cdot \frac{3}{\xi_nA_2}=\frac{18nA_1A_2}{(\xi_nA_2+3)^2}.
$$
In addition, notice that for $n\ge N$ and $0<\epsilon\le 1/2$,
$$
A_1= \frac{\epsilon\alpha}{4}-\frac{\epsilon^2\alpha}{4\zeta_n(1+\epsilon)}>\frac{\epsilon\alpha}{4}-\frac{\epsilon^2\alpha}{4(1+\epsilon)}=\frac{\epsilon\alpha}{4}\left(1-\frac{\epsilon}{1+\epsilon}\right)\ge \frac{\epsilon\alpha}{6},
$$
which results in
$$
\begin{aligned}
\frac{18nA_1A_2}{(\xi_nA_2+3)^2} &= 18n\cdot \frac{\epsilon\alpha}{6}\cdot \frac{\epsilon(1-\epsilon)}{4\zeta_n(1+\epsilon)}\cdot \frac{1}{\left[\frac{\xi_n\epsilon(1-\epsilon)}{4\zeta_n(1+\epsilon)}+3\right]^2} \\
&\ge 18n\cdot \frac{\epsilon\alpha}{6}\cdot \frac{\epsilon(1-\epsilon)}{4\zeta_n(1+\epsilon)}\cdot \frac{1}{\left(\frac{\gamma}{20}+3\right)^2}\ge \frac{3\epsilon^2(1-\epsilon)\alpha\omega n}{4\zeta_n(1+\epsilon)},
\end{aligned}
$$
where $\omega=400/(\gamma+60)^2$. To conclude, it follows that
$$
\begin{aligned}
& \mathbb{P}\Bigg\{\left|\frac 1n\sum_{i=1}^nU_ih(z_i)\right|\ge \frac{\epsilon\alpha}{4}-\frac{\epsilon^2\alpha}{4\zeta_n(1+\epsilon)}+\frac{\epsilon(1-\epsilon)}{4\zeta_n(1+\epsilon)n}\sum_{i=1}^nh(z_i)^2\Bigg\} \\
=& \mathbb{P}\left(\left|\frac 1n\sum_{i=1}^nV_i\right|\ge A_1+A_2\sigma^2\right)\le 2\exp\left(-\frac{3\epsilon^2(1-\epsilon)\alpha\omega n}{4\zeta_n(1+\epsilon)}\right).
\end{aligned}
$$

\noindent Step 6. Conclusion. We have shown that, for $n>(8\zeta_n/[\epsilon^2(\alpha+\beta)])\vee N$, it follows that
$$
\begin{aligned}
& \mathbb{P}\left(\exists f\in\mathcal{F}_n: \mathbb{E}[g(f, Z)]-\frac 1n\sum_{i=1}^ng(f, Z_i)\ge \epsilon\left\{\alpha+\beta+\mathbb{E}[g(f, Z)]\right\}\right) \\
\le & \frac 87 \mathbb{P}\left\{\exists f\in\mathcal{F}_n: \frac 1n\sum_{i=1}^ng(f, Z_i')-\frac 1n\sum_{i=1}^ng(f, Z_i)\ge \frac{\epsilon}{2}(\alpha+\beta)+\frac{\epsilon}{2}\mathbb{E}[g(f, Z)]\right\} \\
\le & \frac 87\Bigg(\mathbb{P}\Bigg\{\exists f\in \mathcal{F}_n: \frac 1n\sum_{i=1}^n[g(f, Z_i')-g(f, Z_i)]\ge \frac{\epsilon}{2}(\alpha+\beta)-\frac{\epsilon^2(\alpha+\beta)}{2\zeta_n(1+\epsilon)} \\
& \hspace{15em} +\frac{\epsilon(1-\epsilon)}{4\zeta_n(1+\epsilon)n}\sum_{i=1}^n\left[g(f, Z_i)^2+g(f, Z_i')^2\right]\Bigg\} \\
& +8\mathbb{E}\left[\mathcal{N}\left(\frac{(\alpha+\beta)\epsilon}{5}, \|\cdot \|_{\infty}, \{g(f, \cdot): \mathcal{Z}\to\mathbb{R}, f\in\mathcal{F}_n\}_{|\mathcal{D}_n}\right)\right]\exp\left(-\frac{3\epsilon^2(\alpha+\beta) n}{40\xi_n^2}\right)\Bigg) \\
\le & \frac 87\Bigg(2\mathbb{P}\Bigg\{\exists f\in \mathcal{F}_n: \left|\frac 1n\sum_{i=1}^nU_ig(f, Z_i)\right|\ge \frac{\epsilon}{4}(\alpha+\beta)-\frac{\epsilon^2(\alpha+\beta)}{4\zeta_n(1+\epsilon)}+\frac{\epsilon(1-\epsilon)}{4\zeta_n(1+\epsilon)n}\sum_{i=1}^ng(f, Z_i)^2\Bigg\} \\
& +8\mathbb{E}\left[\mathcal{N}\left(\frac{(\alpha+\beta)\epsilon}{5}, \|\cdot \|_{\infty}, \{g(f, \cdot): \mathcal{Z}\to\mathbb{R}, f\in\mathcal{F}_n\}_{|\mathcal{D}_n}\right)\right]\exp\left(-\frac{3\epsilon^2(\alpha+\beta) n}{40\xi_n^2}\right)\Bigg) \\
\le & \frac 87\Bigg\{2\mathbb{E}\left(|\mathcal{C}_{\varpi\epsilon\beta}|\max_{h^{\sharp}\in\mathcal{C}_{\varpi\epsilon\beta}}\mathbb{P}\Bigg\{\left|\frac 1n\sum_{i=1}^nU_ih(Z_i)\right|\ge \frac{\epsilon\alpha}{4}-\frac{\epsilon^2\alpha}{4\zeta_n(1+\epsilon)}+\frac{\epsilon(1-\epsilon)}{4\zeta_n(1+\epsilon)n}\sum_{i=1}^nh(Z_i)^2\Bigg\}\right) \\
& +8\mathbb{E}\left[\mathcal{N}\left(\frac{(\alpha+\beta)\epsilon}{5}, \|\cdot \|_{\infty}, \{g(f, \cdot): \mathcal{Z}\to\mathbb{R}, f\in\mathcal{F}_n\}_{|\mathcal{D}_n}\right)\right]\exp\left(-\frac{3\epsilon^2(\alpha+\beta) n}{40\xi_n^2}\right)\Bigg\} \\
\le & \frac 87\Bigg\{2\mathbb{E}\left[2|\mathcal{C}_{\varpi\epsilon\beta}|\exp\left(-\frac{3\epsilon^2(1-\epsilon)\alpha\omega n}{4\zeta_n(1+\epsilon)}\right)\right] \\
& +8\mathbb{E}\left[\mathcal{N}\left(\frac{(\alpha+\beta)\epsilon}{5}, \|\cdot \|_{\infty}, \{g(f, \cdot): \mathcal{Z}\to\mathbb{R}, f\in\mathcal{F}_n\}_{|\mathcal{D}_n}\right)\right]\exp\left(-\frac{3\epsilon^2(\alpha+\beta) n}{40\xi_n^2}\right)\Bigg\}.
\end{aligned}
$$
While conditioning on $Z_i=z_i$ for $i=1, \dots, n$, we choose the $\varpi\epsilon\beta$-covering set of minimal size, which yields
$$
\begin{aligned}
& \mathbb{P}\left(\exists f\in\mathcal{F}_n: \mathbb{E}[g(f, Z)]-\frac 1n\sum_{i=1}^ng(f, Z_i)\ge \epsilon\left\{\alpha+\beta+\mathbb{E}[g(f, Z)]\right\}\right) \\
\le & \frac{32}{7}\mathbb{E}\left[\mathcal{N}\left(\varpi\epsilon\beta, \|\cdot \|_{\infty}, \{g(f, \cdot): \mathcal{Z}\to\mathbb{R}, f\in\mathcal{F}_n\}_{|\mathcal{D}_n}\right)\right]\exp\left(-\frac{3\epsilon^2(1-\epsilon)\alpha\omega n}{4\zeta_n(1+\epsilon)}\right) \\
& +\frac{64}{7}\mathbb{E}\left[\mathcal{N}\left(\frac{(\alpha+\beta)\epsilon}{5}, \|\cdot \|_{\infty}, \{g(f, \cdot): \mathcal{Z}\to\mathbb{R}, f\in\mathcal{F}_n\}_{|\mathcal{D}_n}\right)\right]\exp\left(-\frac{3\epsilon^2(\alpha+\beta) n}{40\xi_n^2}\right) \\
\le & \frac{32}{7}\mathcal{N}_n\left(\varpi\epsilon\beta, \|\cdot \|_{\infty}, \{g(f, \cdot): \mathcal{Z}\to\mathbb{R}, f\in\mathcal{F}_n\}\right)\exp\left(-\frac{3\epsilon^2(1-\epsilon)\alpha\omega n}{4\zeta_n(1+\epsilon)}\right) \\
& +\frac{64}{7}\mathcal{N}_n\left(\frac{(\alpha+\beta)\epsilon}{5}, \|\cdot \|_{\infty}, \{g(f, \cdot): \mathcal{Z}\to\mathbb{R}, f\in\mathcal{F}_n\}\right)\exp\left(-\frac{3\epsilon^2(\alpha+\beta) n}{40\xi_n^2}\right).
\end{aligned}
$$
Observe that for $n\ge N$, we have $\eta\ge 1$, $\gamma>0$,
$$
\varpi=\frac{6\eta-2}{30\eta+3\gamma\eta}\le \frac{6\eta}{30\eta}=\frac{1}{5}, \quad \omega=\frac{400}{(\gamma+60)^2}\le \frac 19,
$$
and
$$
\begin{aligned}
\frac{3\epsilon^2(1-\epsilon)\alpha\omega n}{4\zeta_n(1+\epsilon)} &\ge \frac{27\epsilon^2(1-\epsilon)\alpha\omega n}{40\zeta_n(1+\epsilon)}\ge \frac{27\epsilon^2(1-\epsilon)\alpha\omega n}{40(\xi_n^2\vee \zeta_n)(1+\epsilon)}, \\
\frac{3\epsilon^2(\alpha+\beta) n}{40\xi_n^2} &\ge \frac{3\epsilon^2\alpha n}{40(\xi_n^2\vee \zeta_n)}\ge \frac{27\epsilon^2(1-\epsilon)\alpha\omega n}{40(\xi_n^2\vee \zeta_n)(1+\epsilon)}.
\end{aligned}
$$
Consequently, it follows that
\begin{equation}\label{eqn: gen_thm_11.4_gyorfi_proof_3}
\begin{aligned}
& \mathbb{P}\left(\exists f\in\mathcal{F}_n: \mathbb{E}[g(f, Z)]-\frac 1n\sum_{i=1}^ng(f, Z_i)\ge \epsilon\left\{\alpha+\beta+\mathbb{E}[g(f, Z)]\right\}\right) \\
\le & \frac{32}{7}\mathcal{N}_n\left(\varpi\epsilon\beta, \|\cdot \|_{\infty}, \{g(f, \cdot): \mathcal{Z}\to\mathbb{R}, f\in\mathcal{F}_n\}\right)\exp\left(-\frac{3\epsilon^2(1-\epsilon)\alpha\omega n}{4\zeta_n(1+\epsilon)}\right) \\
& +\frac{64}{7}\mathcal{N}_n\left(\frac{(\alpha+\beta)\epsilon}{5}, \|\cdot \|_{\infty}, \{g(f, \cdot): \mathcal{Z}\to\mathbb{R}, f\in\mathcal{F}_n\}\right)\exp\left(-\frac{3\epsilon^2(\alpha+\beta) n}{40\xi_n^2}\right) \\
\le & 14\mathcal{N}_n\left(\varpi\epsilon\beta, \|\cdot \|_{\infty}, \{g(f, \cdot): \mathcal{Z}\to\mathbb{R}, f\in\mathcal{F}_n\}\right)\exp\left(-\frac{27\epsilon^2(1-\epsilon)\alpha\omega n}{40(\xi_n^2\vee \zeta_n)(1+\epsilon)}\right),
\end{aligned}
\end{equation}
for $n>(8\zeta_n/[\epsilon^2(\alpha+\beta)])\vee N$. When $N\le n\le 8\zeta_n/[\epsilon^2(\alpha+\beta)])$, on the other hand, we note
$$
\exp\left(-\frac{27\epsilon^2(1-\epsilon)\alpha\omega n}{40(\xi_n^2\vee \zeta_n)(1+\epsilon)}\right)\ge \exp\left(-\frac 35\right)\ge \frac{1}{14},
$$
demonstrating that the last right-hand side of Eqn.~\eqref{eqn: gen_thm_11.4_gyorfi_proof_3} exceeds one, and hence the inequality holds trivially, which completes the proof.
\end{proof}

\end{appendix}

\bibliographystyle{imsart-number} 
\bibliography{drc0317.bib}       

\end{document}